\documentclass[final,5p,times,twocolumn]{elsarticle}

\usepackage[bookmarks=true]{hyperref}

\usepackage{balance}

\usepackage{graphicx}
  \graphicspath{{fig/}}
  \DeclareGraphicsExtensions{.pdf,.jpeg,.png}

\usepackage{caption}
\usepackage{stfloats}
\usepackage{float}
\usepackage{url}

\usepackage{tablefootnote}
\usepackage{footnote}
\makesavenoteenv{tabular}
\usepackage[subrefformat=parens,labelformat=parens]{subfig}

\usepackage{amsmath}
\usepackage{bm}

\usepackage{mathrsfs}

\usepackage{amsthm}

\newtheorem{lemma}{Lemma}
\newtheorem{definition}{Definition}
\newtheorem{assumption}{Assumption}
\newtheorem{proposition}{Proposition}

\usepackage{algorithm}
\usepackage{algorithmic}

\usepackage{color}
\usepackage{enumitem}

\newcommand{\bentseq}{\texttt{Bent/Sequential}}
\newcommand{\bentconc}{\texttt{Bent/Concurrent}}
\newcommand{\strseq}{\texttt{Straight/Sequential}}
\newcommand{\strconc}{\texttt{Straight/Concurrent}}

\newcommand{\calP}{\mathcal P}
\newcommand{\calC}{\mathcal C}
\newcommand{\calK}{\mathcal K}
\newcommand{\calT}{\mathcal T}
\newcommand{\calR}{\mathcal R}

\journal{Robotics and Autonomous Systems}

\begin{document}

\begin{frontmatter}
%
\title{Planning coordinated motions for tethered planar mobile robots}

\author[1]{Xu Zhang}
\author[2]{Quang-Cuong Pham\corref{cuong}}


\cortext[cuong]{Corresponding author. Email: \texttt{cuong@ntu.edu.sg}.}

\address{Singapore Centre for 3D Printing, School of Mechanical \& Aerospace Engineering,\\
Nanyang Technological University, 50 Nanyang Avenue, 639798,
Singapore}

\begin{abstract}
  This paper considers the motion planning problem for multiple
  tethered planar mobile robots. Each robot is attached to a fixed
  base by a flexible cable. Since the robots share a common workspace,
  the interactions amongst the robots, cables, and obstacles pose
  significant difficulties for planning. Previous works have studied
  the problem of detecting whether a target cable configuration is
  intersecting (or entangled). Here, we are interested in the motion
  planning problem: how to plan and coordinate the robot motions to
  realize a given non-intersecting target cable configuration. We
  identify four possible modes of motion, depending on whether (i) the
  robots move in straight lines or following their cable lines; (ii)
  the robots move sequentially or concurrently. We present an in-depth
  analysis of \texttt{Straight \& Concurrent}, which is the most
  practically-interesting mode of motion. In particular, we propose
  algorithms that (a) detect whether a given target cable
  configuration is realizable by a \texttt{Straight \& Concurrent}
  motion, and (b) return a valid coordinated motion plan. The
  algorithms are analyzed in detail and validated in simulations and
  in a hardware experiment.
\end{abstract}


\end{frontmatter}


\section{Introduction} \label{sec:Introduction}

In many applications, tethers are used to provide power and resources
to robots or used as a communication
link~\cite{nassiraei2007concept,hong1997tethered,zhang2018large},
especially in extreme environments, such as in underwater, disaster
recovery and rescue~\cite{pratt2008use,
  michael2012collaborative,StephenCass}.  While tethering helps in
communication and resource delivery, it also poses challenges to robot
control and planning.  The limited-length cable restricts the mobile
robot to the workspace around its home station.  In addition, due to
the presence of obstacles and cables, certain positions can be reached
by robot only under specific robot cable configurations.  Under many
circumstances, multiple tethered robots have to cooperate and work
together in a shared workspace~\cite{pratt2008use,
  michael2012collaborative, mcgarey2016system,zhang2018large}, which
introduces more constraints to the planning and control problems.  
More specifically, in the large-scale, multi-robot 3D Printing 
system~\cite{zhang2018large}, the tethers are
indispensable to deliver the fresh concrete from the mixer to the
print nozzle. 
Consider also the cable-cable interaction and robot-cable interaction 
when multiple tethered robots are deployed for terrain and harsh 
environment exploration~\cite{pratt2008use, michael2012collaborative, mcgarey2016system}.
It is important to make sure that the cables are not entangled while 
the robots move to their target positions, considering the above mentioned interactions and obstacle avoidance.



\begin{figure}[htp]
\centering
\subfloat[]{%
    \includegraphics[width=3.5cm]{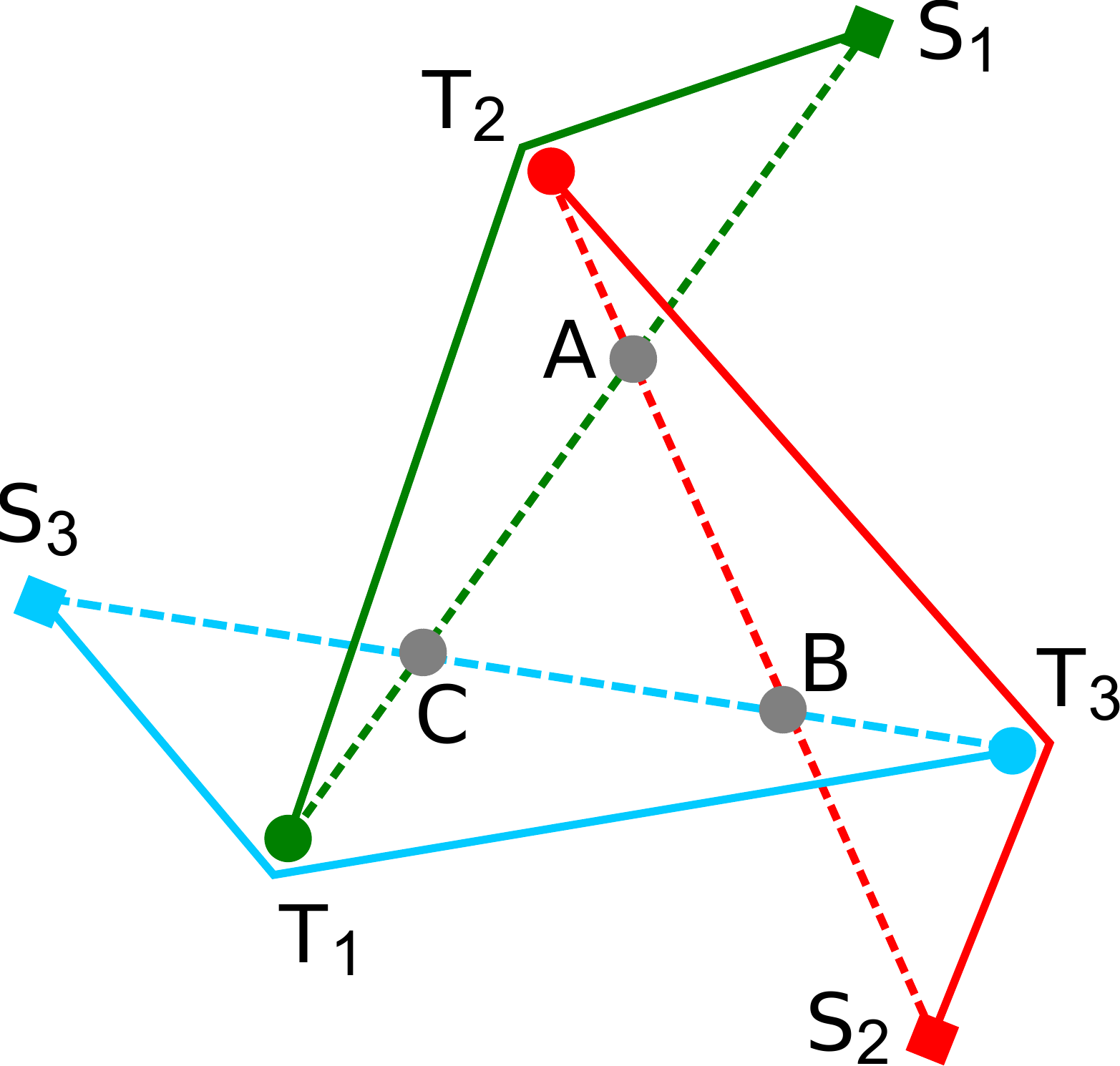} }

\subfloat[]{%
    \includegraphics[width=7cm]{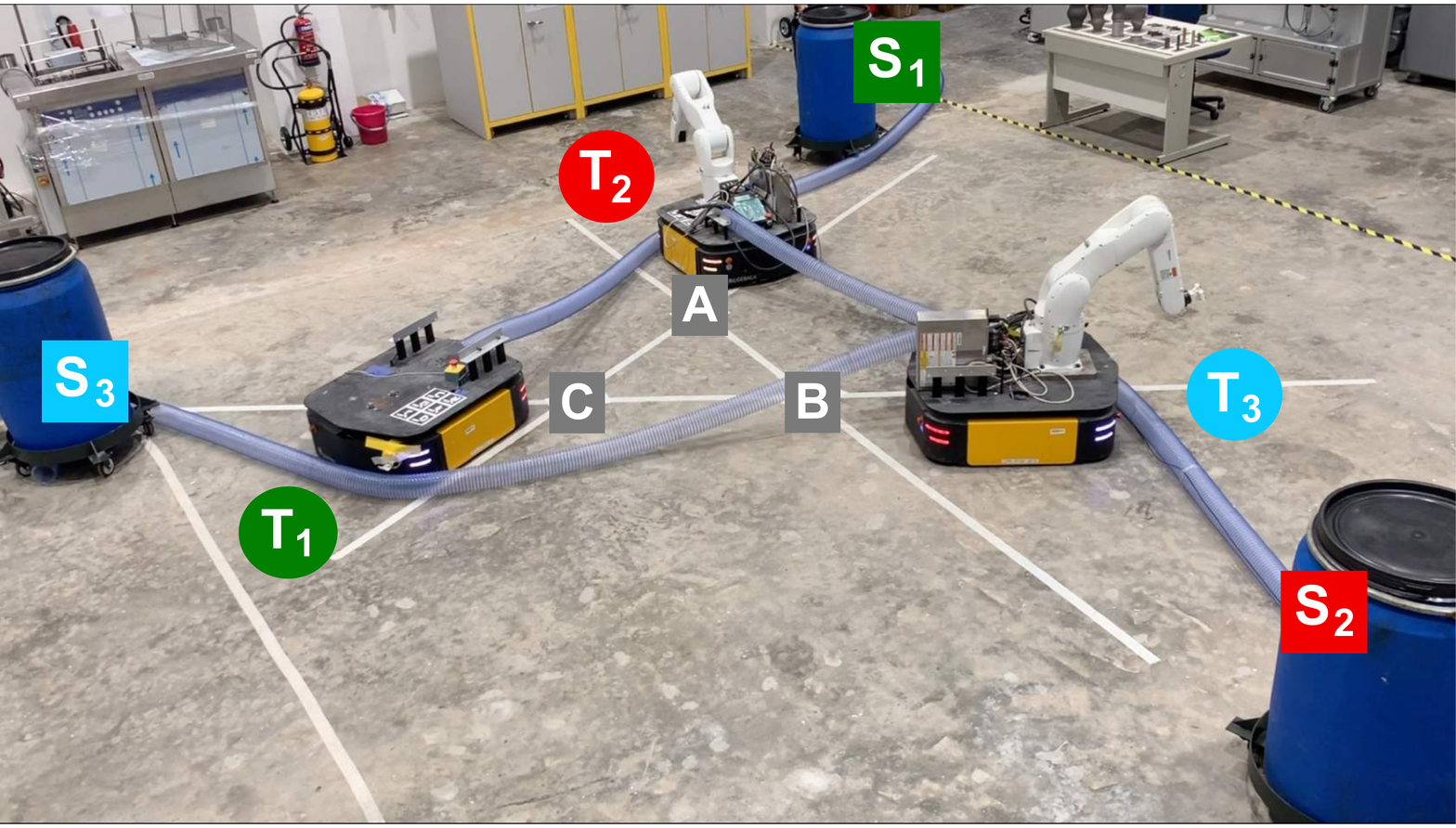} }
  \caption{An example of coordinated motion planning problem for three
    tethered mobile robots. (a) Each robot $r_i$ ($i=1,2,3$) must move
    from its starting position $S_i$ towards its target position
    $T_i$, while being tethered to $S_i$. The desired configuration of
    the cables at the end of the motions are depicted in solid
    lines. This paper proposes an algorithm to plan simultaneous,
    straight-line, motions of the robots to realize the desired cable
    configuration. We show in particular that the motion priority at
    the intersection points $A,B,C$ is critical. (b) Hardware
    implementation of the coordinated motion. See full video of the
    experiment at \url{https://youtu.be/Wdk9E0bB4yA}.}
\label{fig:tether-intro}
\end{figure}

Here we consider multiple tethered planar (point) mobile robots sharing 
a common workspace. Each robot is attached to a fixed base by a flexible
cable, which is kept taut at all times.  The cable can be pushed by
other robots, and can bend around other robots or obstacles. The 
assumptions are consistent with the requirements of the tethers in many 
practical applications, such as the material-delivery tethers in the 
concrete 3D Printing system~\cite{zhang2018large}, or in terrain explorer 
system~\cite{mcgarey2016system}. 
Previous work has studied the problem of detecting whether a target 
cable configuration is intersecting (or
entangled)~\cite{hert1997planar}. Here, we
investigate the \emph{motion planning} problem to \emph{realize} a
given non-intersecting target cable configuration.  More precisely, we
study four possible modes of motion -- \texttt{Straight/Sequential}
and \texttt{Straight/Concurrent}, \texttt{Bent/Sequential},
\texttt{Bent/Concurrent}, depending on whether (i) the robots move in
straight lines or following their cable lines; (ii) the robots move
sequentially or concurrently.

\begin{figure}[htp]
\centering
\subfloat[]{%
    \fbox{\includegraphics[height=2.5cm]{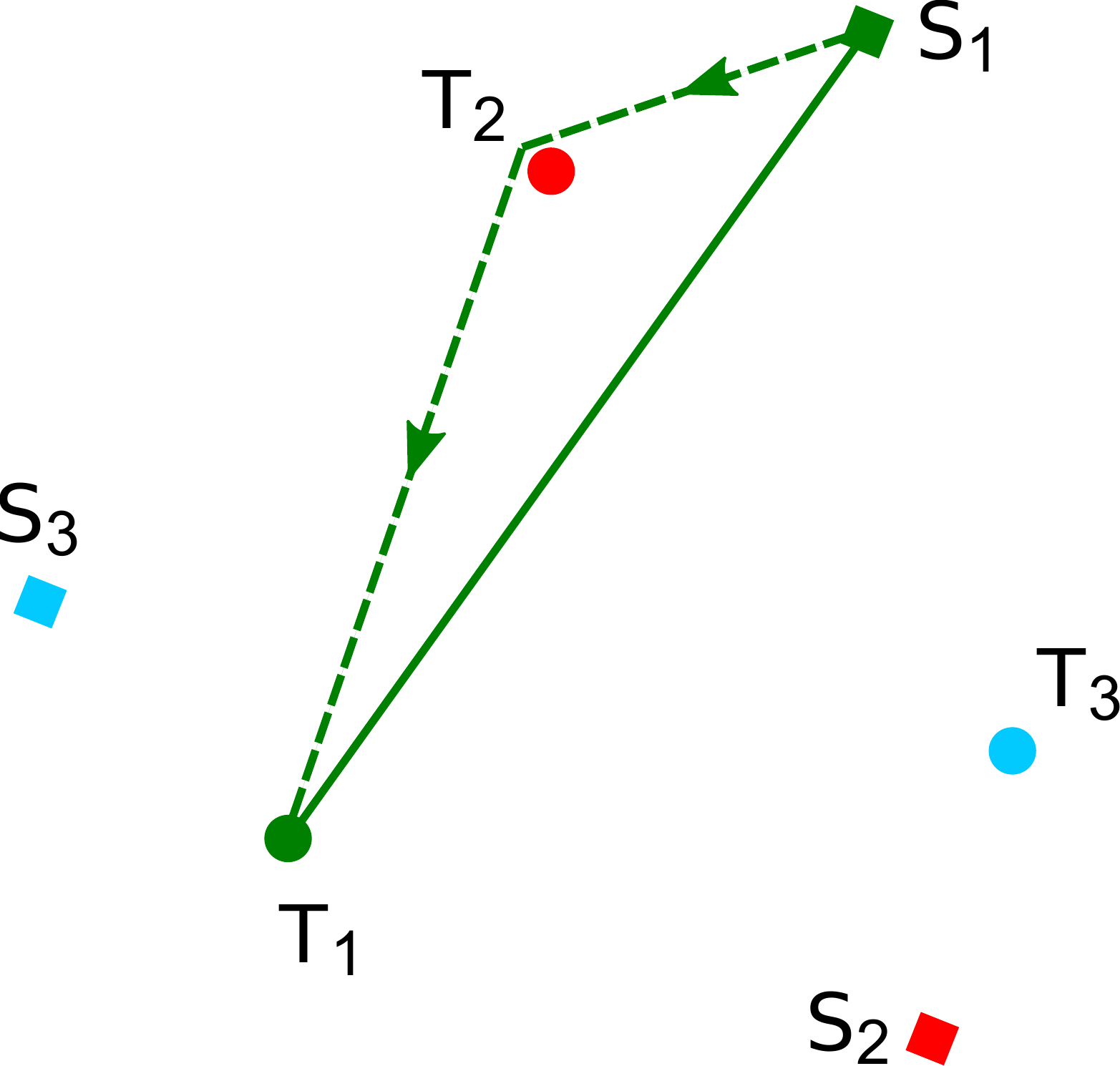}}   
    \fbox{\includegraphics[height=2.5cm]{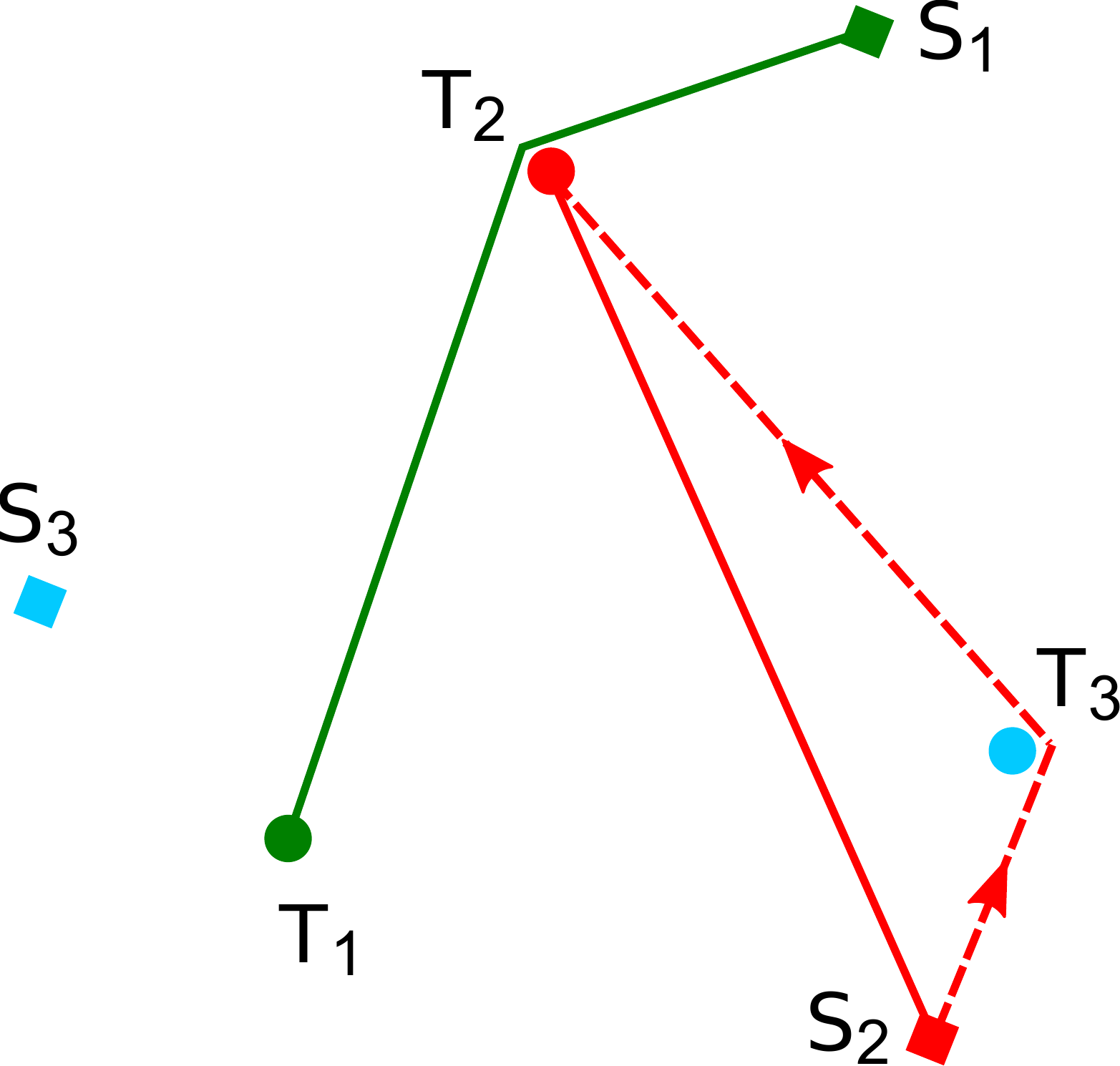}}
    \fbox{\includegraphics[height=2.5cm]{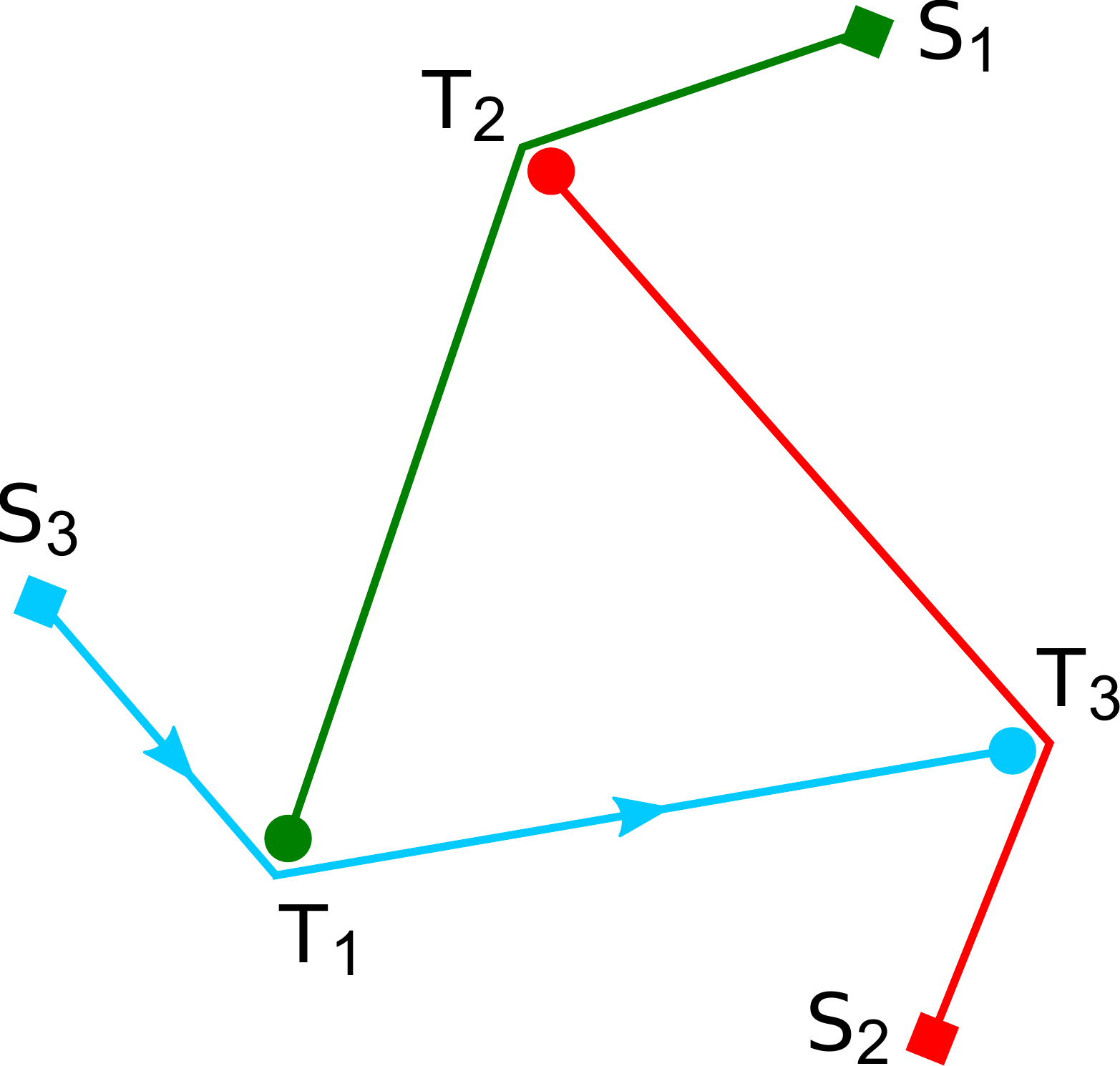}}}

\subfloat[]{%
    \fbox{\includegraphics[height=2.5cm]{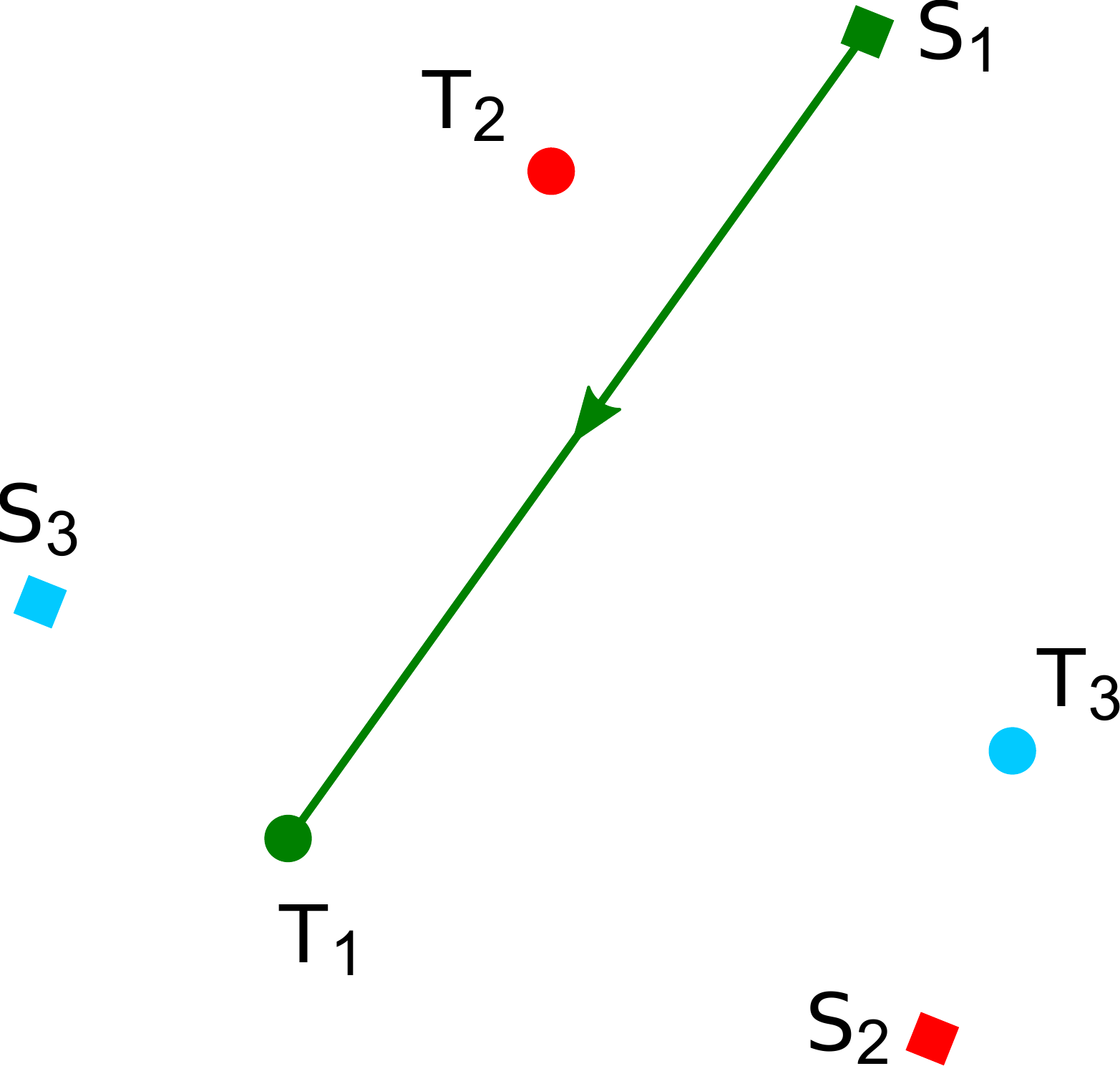}}   
    \fbox{\includegraphics[height=2.5cm]{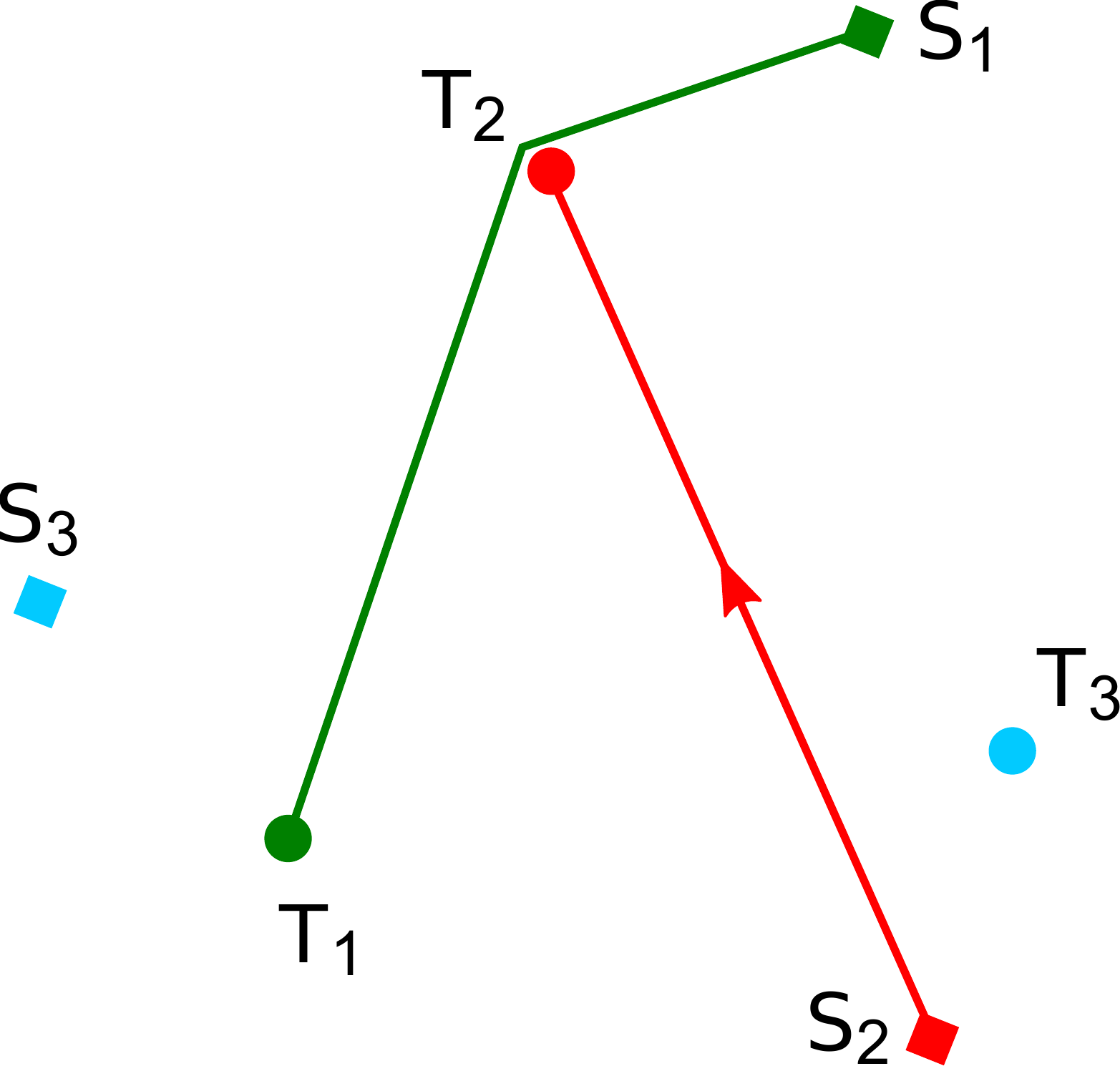}}
    \fbox{\includegraphics[height=2.5cm]{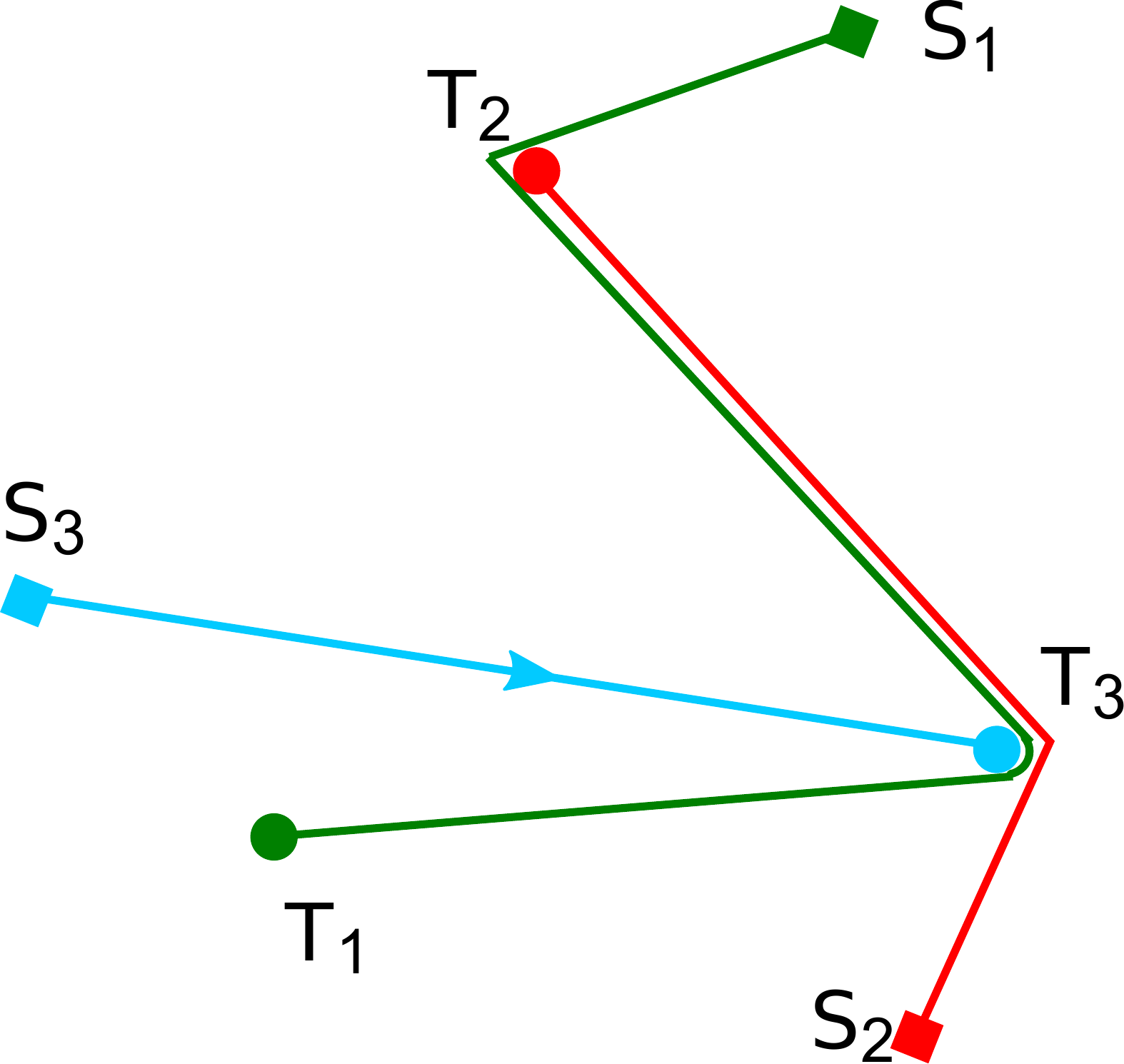}}}

\subfloat[]{%
    \fbox{\includegraphics[height=2.5cm]{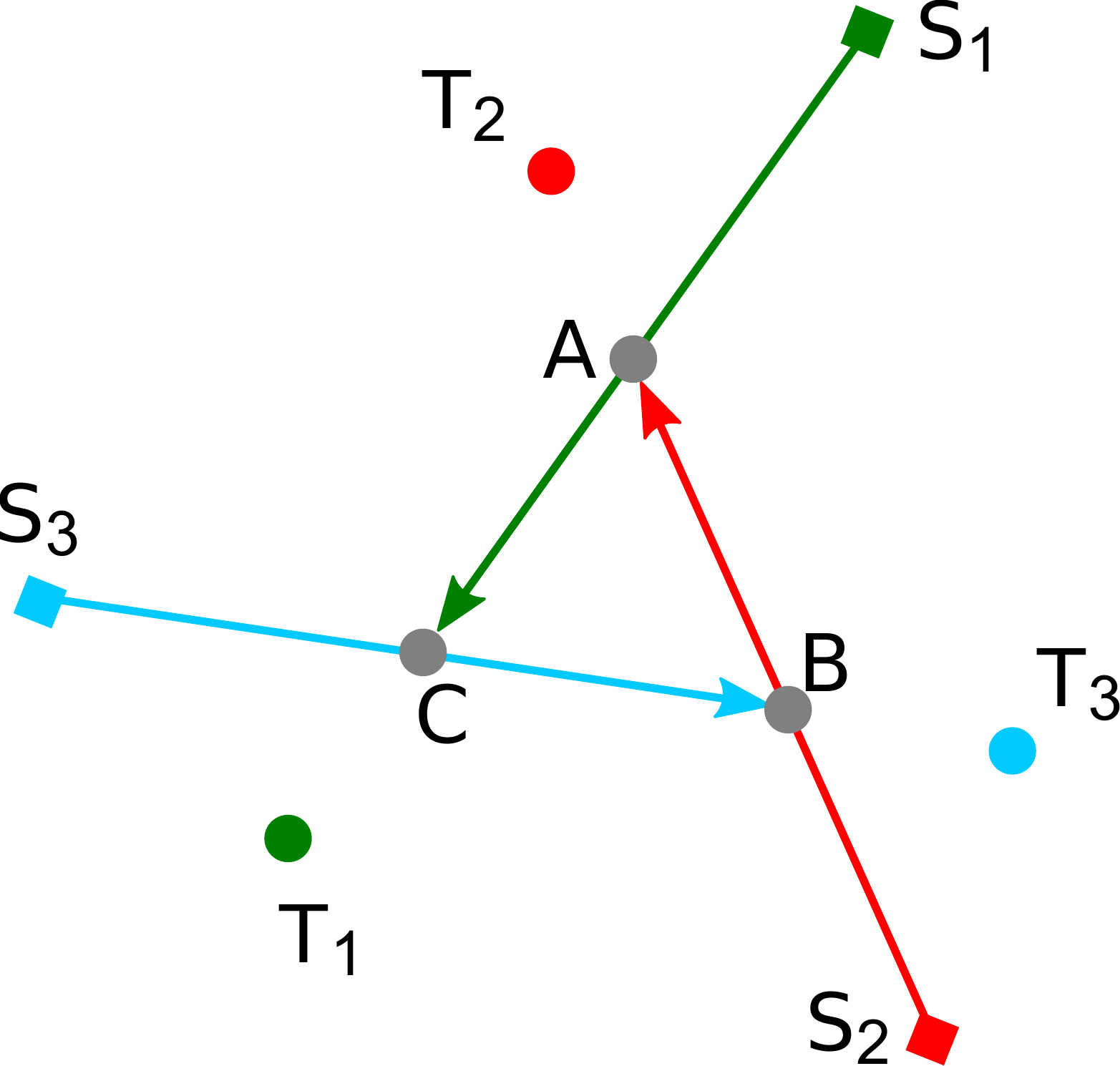}}   
    \fbox{\includegraphics[height=2.5cm]{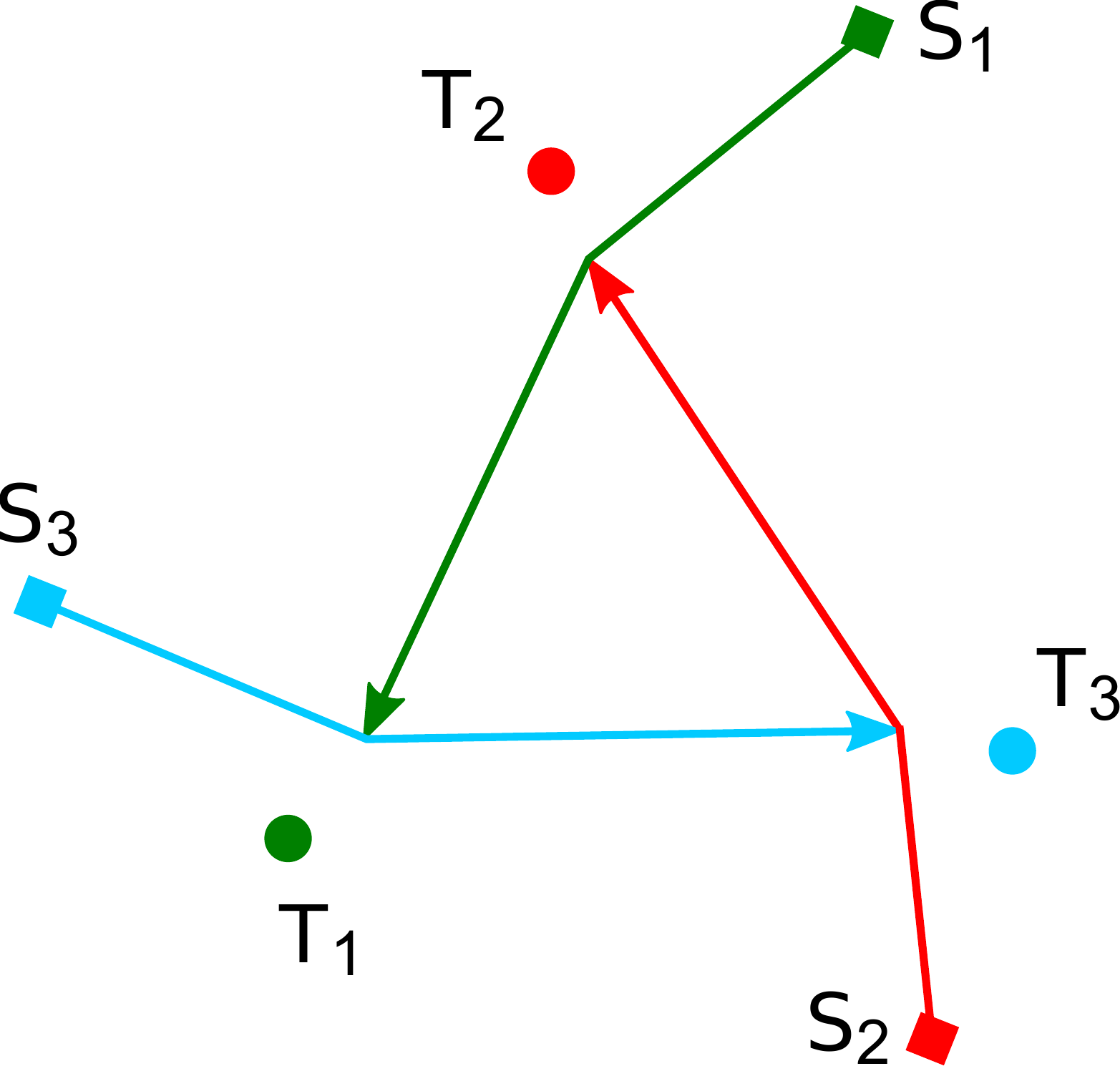}}
    \fbox{\includegraphics[height=2.5cm]{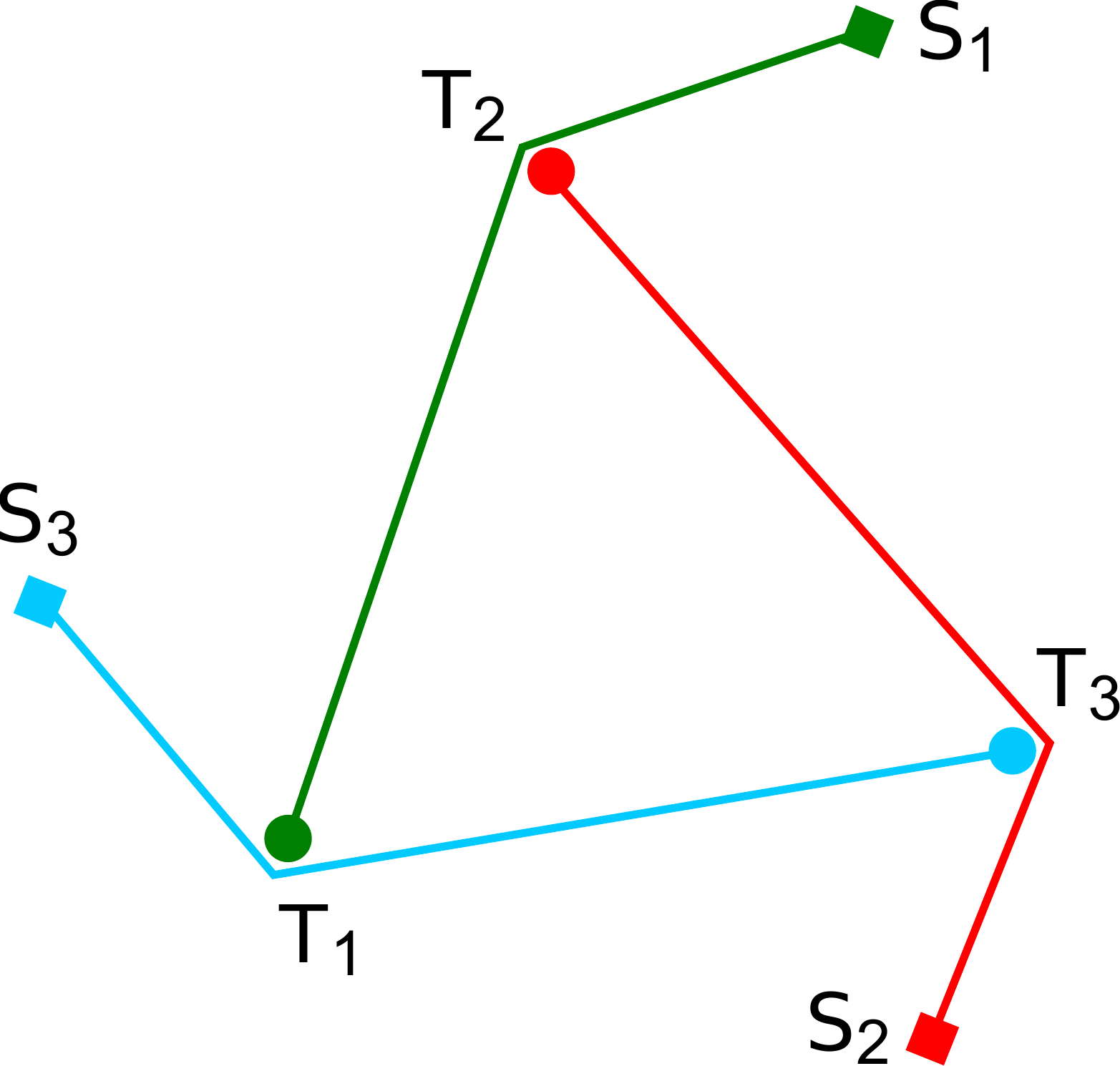}}}

  \caption{An example of planning for a target cable configuration by different modes of robot motion. 
  (a) \bentseq{} mode. The target cable configuration can be achieved when robots move in the order: $r_1$, $r_2$ and $r_3$. Dotted cable lines retract to solid ones since the cables are kept taut at all times.
  (b) \strseq{} mode. The target cable configuration \emph{CANNOT} be achieved by this mode of motion. The example shows how the final result deviates from the target cable configuration.
  (c) \strconc{} mode. The target cable configuration can be achieved when robots move along the $(S_i, T_i)$ straight-line segments at the same velocity.}
\label{fig:example-intro}
\end{figure}

\subsection{Illustrative example}

To illustrate some of the difficulties, consider a scenario with three
tethered robots as in Fig.~\ref{fig:tether-intro}: each robot $r_i$
($i=1,2,3$) tethered to its starting base $S_i$ must navigate from
$S_i$ to its target position $T_i$.  The target cable configuration as
shown in solid lines has to be achieved upon arrival of all robots. 

In \bentseq{} mode, the target cable configuration can be achieved,
for example, in the order: $r_1$, $r_2$ and $r_3$.  Since the cables
are kept taut at all times, when $r_1$ first moves along
$S_1 - T_2 - T_1$, its cable $C_1$ retracts to the straight-line
segment $S_1T_1$ when it reaches $T_1$.  Then $r_2$ moves along
$S_2 - T_3 - T_2$, pushing cable $C_1$ to the layout $S_1T_2T_1$ as it
reaches $T_2$, and $C_2$ retracts to $S_2T_2$.  Finally $r_3$ moves
along $S_3 - T_1 - T_3$ and pushes cable $C_2$ to the layout
$S_2T_3T_2$.  Since $r_1$ has already arrived at $T_1$, cable $C_3$
will bend around it and remain at its final layout as $S_3T_1T_3$.
The target cable configuration is thus achieved (see Fig.~\ref{fig:example-intro}(a)).  Note that the
configuration can also by achieved by \bentconc{} mode where the
robots move simultaneously following their cable lines.
 
However, simply following robots’ cable lines is not the most desirable 
motion scheme.
Compared to \texttt{Bent} motion, \texttt{Straight}-line robot motion offers the following benefits:
\begin{enumerate}
  \item[1)] Straight-line robot motion reduce traveling distance, which in turn improves time and energy efficiency. Take the two target cable configurations provided in this paper for example.
  To achieve the target cable configuration in Fig.~\ref{fig:tether-intro} and Fig.~\ref{fig:straight_sequential}(a), straight-line robot motion saves respectively 30\% and 20\% of the traveling distance compared to robot motion that follows cable lines.
  \item[2)] 
  It is well-known that in robot localization, rotational and curved translational motions induce larger estimation errors than straight motions, in relation to both Inertial Measurement Unit (IMU) and odometry~\cite{siegwart2011introduction}. Thus, by using the \texttt{Straight} mode, one can significantly reduce localization errors;
  \item[3)] Working environment safety could be enhanced as straight-line robot motion is more predictable for the workers who need to collaborate or work in the same environment with the robots;
\end{enumerate}
The first benefit motivates the work by Hert and
Lumelsky~\cite{hert1994ties}, where an algorithm is proposed for
robots to achieve a given target cable configuration by \strseq{}
motion.  \strseq{} mode improves the efficiency at the sacrifice of
realizability, which is illustrated in Fig.~\ref{fig:example-intro}(b).
Suppose the robots move in the order of $r_1$, $r_2$ and $r_3$ along
$S_1-T_1$, $S_2-T_2$ and $S_3-T_3$ respectively.  Upon arrival of
$r_1$ at $T_1$ and $r_2$ at $T_2$, we can obtain the same cable
configuration as under \bentseq{} motion, but after that, $r_3$ must
push $C_2$ in order to navigate to $T_3$, which violates the target
cable configuration, resulting in a deadlock.
In their later work~\cite{HERT1996187},
\strconc{} mode is considered, but the concurrent motions are not 
leveraged to solve deadlocks.  Yet, considering again the example 
of Fig.~\ref{fig:tether-intro}, concurrency can be leveraged to 
solve the deadlock as follows: $r_1$, $r_2$ and $r_3$ move along 
their $(S_i, T_i)$ straight-line segments at the same velocity 
(see Fig.~\ref{fig:example-intro}(c)).
By doing so, they will indeed pass respectively points $A$, $B$ 
and $C$ at the same time, which in turn satisfies the requirement 
that $r_1$ passes $A$ before $r_2$ does, $r_2$ passes $B$ before 
$r_3$ does, and $r_3$ passes $C$ before $r_1$ does.
The general algorithm to produce such a schedule is one of the 
main contributions of the present paper.

\subsection{Contributions and organization of the paper}

Our contribution in this work is three-fold.  First, we systematically
examine the realizability of a target cable configuration by the above
mentioned four modes of motion.  We then present a further analysis of
\strconc{} motion, which is the most practically-relevant mode.
Second, we propose algorithms to detect whether a given target cable
configuration is realizable by \strconc{} motion.  Third, our
algorithms return a valid coordinated motion plan if there exists a
solution. In addition, we provide some corrections to the algorithms
presented in~\cite{hert1997planar} to detect whether a cable
configuration is intersecting.

The remainder of this article is organized as follows.  In
Section~\ref{sec:literature}, we review related works in motion
planning for single and multiple tethered robots.  In
Section~\ref{sec:background}, we introduce the background on
non-intersecting cable configurations, formulate the planning problem
for multiple tethered robots, and survey the realizability of cable
configurations by the four modes of motion. In
Section~\ref{sec:proposed-algorithm}, we present our planning
algorithms for \strconc{} robot motion in detail. In
Section~\ref{sec:experiments}, we report experimental results in
simulation and in hardware.
Finally, we conclude our work in Section~\ref{sec:conclusion}.

\section{Related work} \label{sec:literature}

\subsection{Single tethered robot} \label{single}

Motion planning for a single tethered robot mainly falls into two
topics: 1) shortest path planning; 2) coverage and terrain
exploration.

\subsubsection{Shortest path planning}    \hspace*{\fill}

There has been active research on planning the shortest path for a 
tethered robot to navigate from a starting position to a target position.  
Methods developed by Xavier~\cite{xavier1999shortest} and Xu et
al.~\cite{Xu2012AnIA} are similar: they construct a visibility graph
based on the triangulation of the environment, and perform an
almost-exhaustive enumeration of graph paths in the selected homotopy
class.  In~\cite{igarashi2010homotopic,kim2014path,kim2015path}, the
authors approached the problem by using a conception of homotopy
classes of the cables.  The difference between these papers lies in
the way they identify homotopy classes.  Igarashi and
Stilman~\cite{igarashi2010homotopic} use distance from the initial
vertex that entirely based on a metric, which under certain circumstances
would fail to identify that two vertices in the graph represent two
different homotopy classes.  To avoid this failure, Kim et
al.~\cite{kim2014path,kim2015path} use a true homotopy invariant
(h-signature) to construct a h-augmented graph that explicitly bears
the topological information. The method is capable of producing an optimal
result based on a discretized workspace. Similar concepts are also applied 
to plan for manipulation and transportation of objects on the plane using 
tethered mobile robots~\cite{bhattacharya2015topological}.
In another recent
study~\cite{salzman2015optimal} that is built
on~\cite{kim2014path,kim2015path}, the authors 
use a visibility-graph based approach instead of the grid-based approach, 
and extend the point robot assumption to polygonal (translating) robot.
Unlike prior work that generally adopts off-line global discretization of 
the configuration space, Teshnizi and Shell~\cite{teshnizi2014computing} 
decompose the environment into cells using a subset of the visibility graph. 
The approach enables dynamic generation of necessary parts of the space, 
which in turn, improves efficiency of the proposed method.
In their later work~\cite{teshnizi2016planning}, they generalize the 
decomposition method to solve the planning problem for a robot attached 
to a stiff tether.

\subsubsection{Coverage and terrain exploration}    \hspace*{\fill}

Coverage and extreme terrain exploration have also been studied regarding 
planning for tethered robots. Shnaps and
Rimon~\cite{shnaps2014online} present an online algorithm considering
the case where a tethered robot has to cover an unknown planar
environment that contains obstacles, and returns back to the base.
Abad-Manterola et al.~\cite{abad2011motion} give a motion planner for an 
\emph{Axel} tethered robotic rover on steep terrains whose geometry is 
known a priori with high precision.  Their work is later extended to the 
case where the terrain information is not fully known beforehand~\cite{tanner2013online}.
In~\cite{mcgarey2016line}, a nonvisual tether-based localization and 
mapping technique based on FastSLAM is developed and tested on the 
Tethered Robotic eXplorer (\emph{Axel}). The online particle-filter 
approach is then expanded to TSLAM~\cite{mcgarey2017tslam}.
Using the same \emph{Axel} platform with customized tether controller, 
the authors demonstrate visual route following for extreme terrain 
exploration~\cite{mcgarey2017falling}. 
Visual route following for a tethered mobile robot is also discussed 
in~\cite{tsai2013autonomous}, without addressing tether management.

\subsection{Multiple tethered robots} \label{multiple}

There are two major streams pertaining to multiple tethered robots planning:
1) cooperative control and manipulation; 2) navigation from respective
starting to target positions.

\subsubsection{Cooperative control and manipulation}    \hspace*{\fill}

In ~\cite{yamashita1998cooperative}, the authors propose a method that
make multiple mobile robots accomplish tasks with the
cables act as tools connected between the robots.  
Similarly, in ~\cite{bhattacharya2011cooperative}, the authors address 
the cooperative control of two autonomous surface vehicles for oil
skimming and cleaning.

\subsubsection{Robot navigation}    \hspace*{\fill}

Khuller et al.~\cite{khuller1998graphbots} plan for a team of robots 
to move from any starting location
to any goal destination in a graph-structured space simultaneously 
while maintaining a particular configuration.

In some other applications, instead of keeping a particular formation,
each of the robots needs to navigate to a different location
for a specific task, where cable entanglement avoidance emerges as a
crucial challenge.
There are some remarkable works by Hert and Lumelsky in this field of
study.  In~\cite{hert1994ties,hert1999motion}, to achieve tangle-free
planning, robots are prioritized and scheduled to move sequentially in
a common planar and spatial environment that is free of obstacles.
However, the sequential motion of a group of robots is not efficient and
the planning may fall in deadlock under certain circumstances.
Although simultaneous robot motion is mentioned in their later work on 
planar workspaces~\cite{HERT1996187}, it is not leveraged to 
solve deadlocks. 
Another work by Sinden provides an algorithm for scheduling multiple
robots moving simultaneously while avoiding tether tangling, but only 
under some special cases~\cite{sinden1990tethered}.

\subsection{Summary}
In this work, we consider the robot navigation problem for multiple tethered planar mobile robots. 
We aim at examining the realizability of a target cable configuration 
by different modes of robot motion, and
developing path planning algorithms for the robots to achieve the
given target cable configuration by \strconc{} motion.

\section{Background: non-intersecting target cable configurations and
  realizability} \label{sec:background}


\subsection{Problem formulation}

Consider $n$ point robots on a planar workspace $\boldsymbol{W}$.
Each robot $r_i$, $i = 1, \dots, n$, must navigate from its starting
position $S_i$ to its target position $T_i$, while being attached to
$S_i$ by a flexible cable.  The cables are assumed to remain on the
ground and to be taut at all times; they can bend around other robots, 
and be pushed by other robots.

\begin{definition}[Cable line, cable configuration]
  Because of the assumption that the cables are taut and can only bend
  around point robots, the cable of any robot at any time instant is a
  polygonal chain, which we call the robot's \emph{cable line}. The
  \emph{final (target) cable line} of a robot is the (desired) state
  of the robot's cable line after all robots have reached their
  target positions.

  A \emph{cable configuration} consists of the cable lines of all
  robots at a given time instant. A \emph{final (target) cable
    configuration} is the (desired) cable configuration after all
  robots have reached their targets.
\end{definition}

\emph{Cable polygon} in the following definition was originally defined 
as \emph{retraction polygon} in~\cite{hert1994ties}.
\begin{definition}[Cable polygon]
  The cable polygon $\Pi_i$ of a given robot $r_i$ is formed by
  closing the robot's target cable line with the straight-line segment
  $(S_i,T_i)$ joining the robot's starting and target positions.
\end{definition}

We now introduce two important assumptions.

\begin{assumption}[Starting positions and cable polygons]
  The starting position $S_i$ of a robot $r_i$ does not belong to any
  cable polygon $\Pi_j$, $j\neq i$.
  \label{assumption:starting-positions}
\end{assumption}

Assumption~\ref{assumption:starting-positions} fully covers the
hypotheses of previous works~\cite{hert1994ties,hert1999motion}, which
assume that $\boldsymbol{W}$ is a convex polygon and that all starting
positions lay on the boundary of $\boldsymbol{W}$. The extension to
more general cases with obstacles present in the environment is
discussed in~\ref{Stage1-Obstacle}.

\begin{assumption}[No self loop]
  When a target cable line contains repeated waypoints, we say that it
  contains \emph{self loops} (see Fig.~\ref{fig:self loop}(a)) We
  assume that there is no self loop in any target cable line.
\end{assumption}

self loops are obviously redundant and there is always a better
solution.  An example is presented in Fig.~\ref{fig:self loop}: The
target cable configuration in Fig.~\ref{fig:self loop}(a) is valid
in the sense that no cable tangle happens, but it is highly
inefficient and can be replaced with an alternative simple cable
layout as shown in Fig.~\ref{fig:self loop}(b).

\begin{figure} [htp] \subfloat[An example of cable self loop. \label{subfig:self loop-no}]{%
    \includegraphics[width=1.6in]{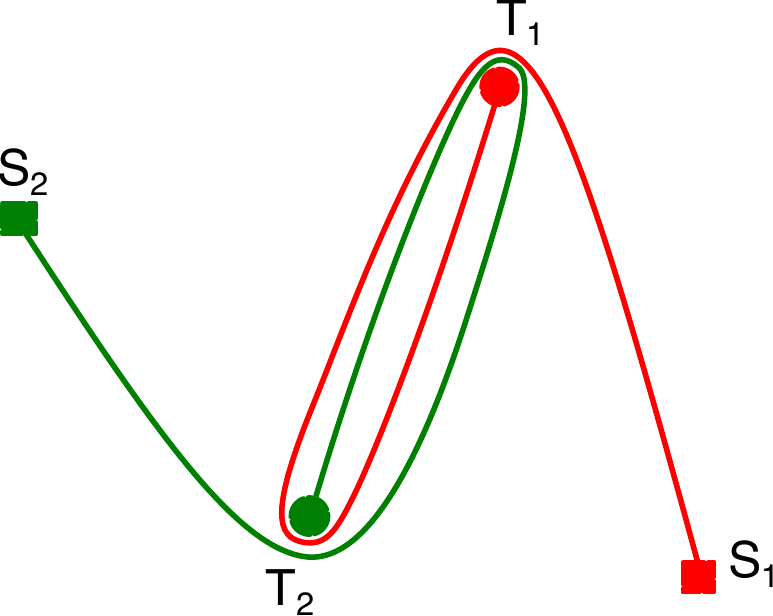} }
\hspace{3mm} \subfloat[An alternative simple cable
layout. \label{subfig:self loop-has}]{%
    \includegraphics[width=1.6in]{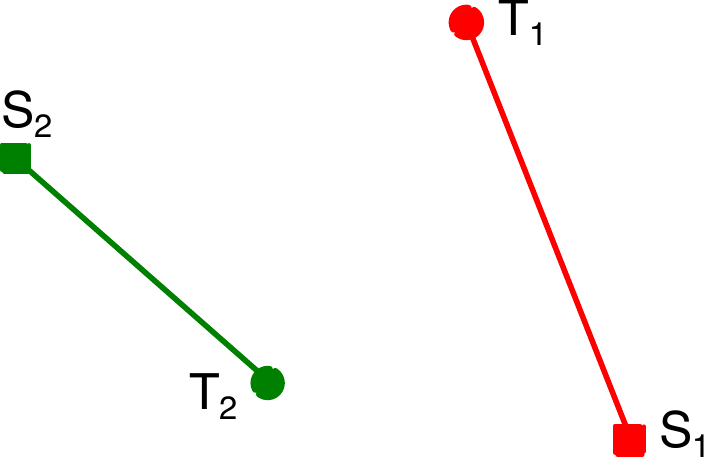} }
 \caption{An example of cable self loop and its alternative simple
layout.}
 \label{fig:self loop}
\end{figure}

\subsection{Non-intersecting target cable configurations}

We say that a cable configuration is \emph{non intersecting} if no
pair of cable lines in that configuration intersect each other.
In~\cite{hert1997planar}, the problem of finding non-intersecting
target cable configurations is formulated as a graph problem.
$G = \{V, E\}$ contains all starting points \{$S_1, \dots, S_n$\} and
target points \{$T_1, \dots, T_n$\} in its node set $V$, and all edges
in the form of ($S_i, T_j$) in its edge set $E$.  Finding
non-intersecting target cable configuration is now equivalent to
finding a set of consecutive edges from $S_i$ to $T_i$ for each robot
$r_i$ while observing cable property.  An algorithm using exhaustive
search method with pruning is suggested to identify and eliminate
every route intersection case, producing all possible non-intersecting
cable configurations.  Regarding the cable intersection detection
algorithm, readers may refer to~\cite{hert1997planar} for more
details.  Note however that the algorithm would fail under certain
circumstances. We discuss the failure cases and provide corrections in~\ref{Hert-Lumelsky}.

\subsection{Realizability of non-intersecting target cable
  configurations} \label{sec:realizability}

We now study the realizability of a given non-intersecting target
cable configuration (in the following text we shall omit the term
``non-intersecting'' when there is no possible confusion).

There are four reasonable modes of motion: \bentseq{}, \bentconc{},
\strseq{} and \strconc{}.  \texttt{Straight} and \texttt{Bent} refer to the
route that the robots have to follow during the navigation:
\texttt{Straight} denotes that the robots moves in straight lines from
their starting positions directly towards target positions, while
\texttt{Bent} indicates that the robots follow their target cable
lines. Meanwhile, \texttt{Sequential} and \texttt{Concurrent} refer to the schedule
of the robot motion: \texttt{Sequential} indicates that the robots are
prioritized and they have to move towards their target positions one
at a time, while \texttt{Concurrent} denotes that the robots move to target
positions simultaneously.  The solution spaces of these different
modes of motion is shown in Fig.~\ref{fig:solvability}, and a thorough
analysis is presented in the following sections.

\begin{figure}[htp]
\centering
\includegraphics[width=6cm]{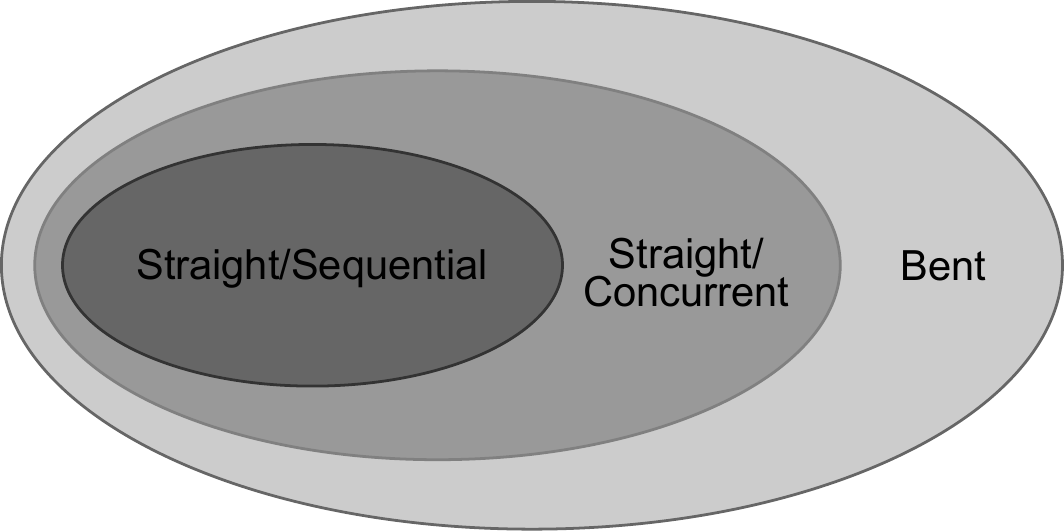}
\caption{Solution spaces of the different modes of motion
  (\texttt{Bent}, \texttt{Straight/Concurrent},
  \texttt{Straight/Sequential}).}
\label{fig:solvability}
\end{figure}

\subsubsection{Realizability by \bentseq{}} \label{subsub:bent_sequential} \hspace*{\fill}

It is proved in~\cite{hert1994ties} that: ``\textit{If each robot
  moves along its target cable line, the final cable configuration
  will be the same as the target cable configuration.}''  Note that
the proof is non-trivial because of the assumption that the cables
remain taut at all time.  A robot's cable line will thus deform and
auto-retract to maintain tautness if the robots that cause bends in
its target cable line are not yet in position.

\subsubsection{Realizability by \bentconc{}} \label{subsub:bent_concurrent} \hspace*{\fill}

It is straightforward to infer that \bentseq{} could be regarded as an
extreme case of \bentconc{}.  Since \bentseq{} mode of motion is
proved to be able to achieve any feasible target cable configuration,
it follows that \bentconc{} can produce a solution for any feasible
target cable configuration as well.

The most efficient motion for this mode would be ``just move'', which
means that the robots move along their target cable lines at the same
time.  When robot $r_i$ passes point $T_j$ on its target cable line,
if $T_j$ has already been occupied by robot $r_j$, cable $C_i$ wraps
 $r_j$ from the desired direction; if $r_j$ has not reached $T_j$
yet, $C_i$ will be retracted and will be pushed by $r_j$ later from
the correct side (see Fig.~\ref{fig:Ben-Con}). 

\begin{figure} [htp] 
\subfloat[]{%
    \includegraphics[width=1.5in]{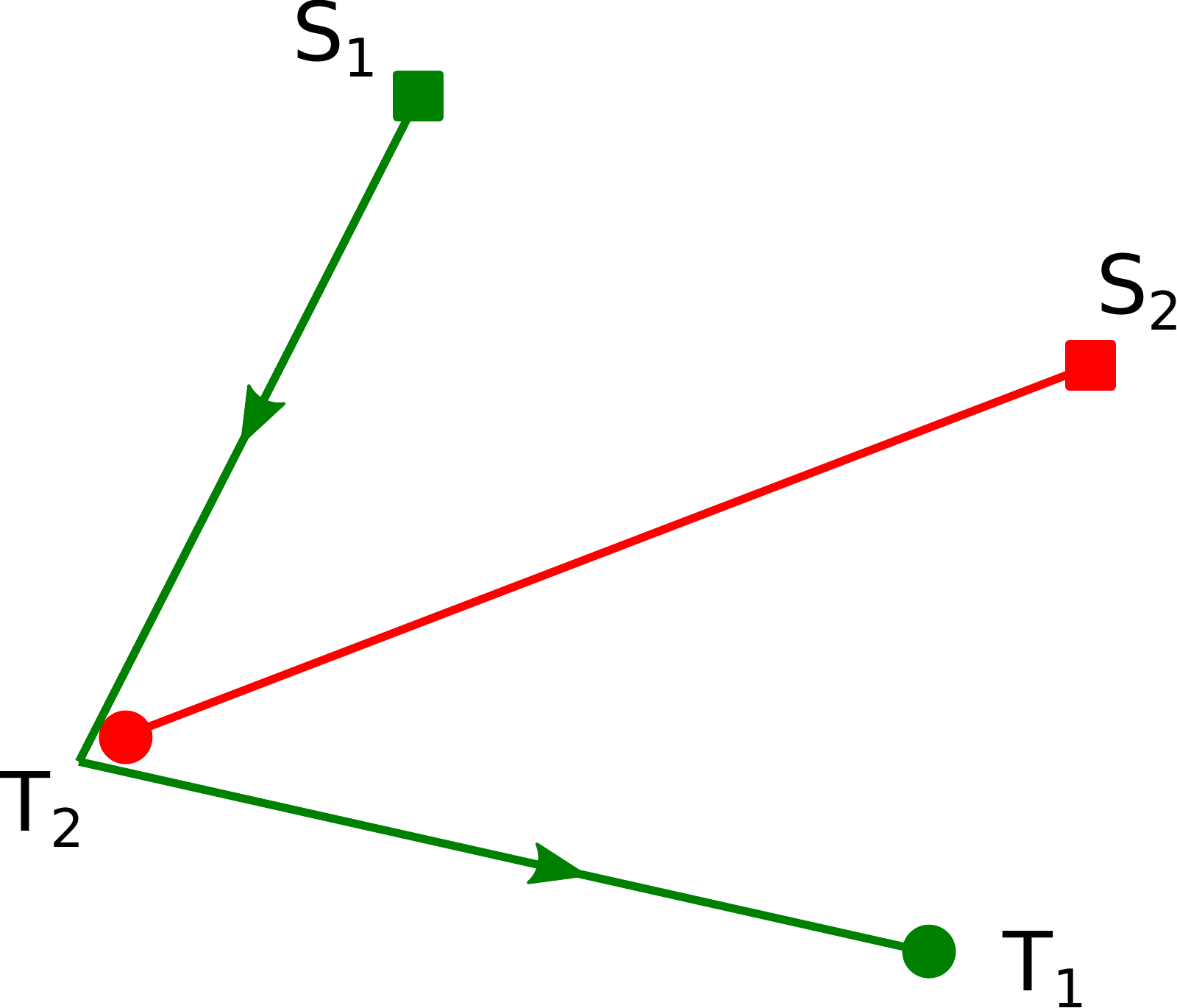} }
\hspace{3mm} 
\subfloat[]{%
    \includegraphics[width=1.5in]{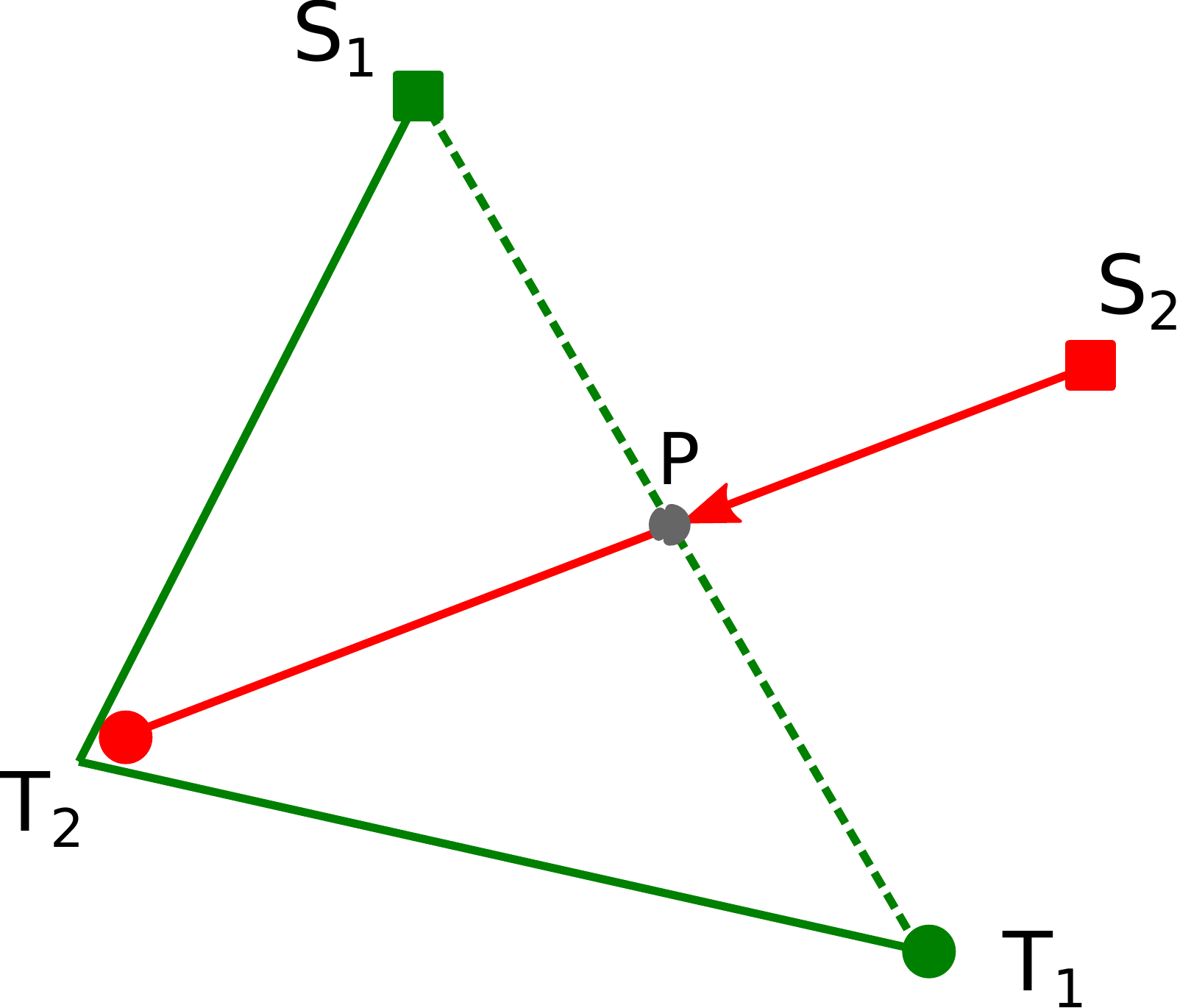} }
 \caption{An illustration of realizability by \bentconc{}. When $r_1$ passes $T_2$: (a) $r_2$ has reached $T_2$. $C_1$ wraps around $r_2$ at $T_2$; (b) $r_2$ has not reached $T_2$ yet. $C_1$ retracts to the dotted line and is pushed by $r_2$ at $P$ during $r_2$'s navigation later.}
 \label{fig:Ben-Con}
\end{figure}

Compared to \bentseq{} motion, \bentconc{} requires less execution
time as robots do not need to wait and move one by one.  However,
restricting the robots to move along their target cable lines is still
a waste of time and energy. The most efficient motion scheme would be
straight-line motion for robots to navigate from starting to target
positions.

\subsubsection{Realizability by \strseq{}} \label{subsub:straight_sequential} \hspace*{\fill}

In~\cite{hert1994ties}, Hert and Lumelsky provide an algorithm where
robots are prioritized and scheduled to move sequentially in a
straight line as much as possible.  In this case, ``\emph{each robot 
must move before any robot its cable
  line bends around has moved}''~\cite{hert1994ties}. With this
observation, a directed graph can be constructed: each vertex in the
graph represents a robot, and a directed edge from $v_i$ to $v_j$
implies that robot $r_i$ must move before $r_j$.  The order of robot
motions can then be determined by extracting at each step the root
node of the graph, which corresponds to the robot that does not have
to wait for any other robots.  In the event of loop occurrence in the
graph, the robot motion falls in deadlock situation. As an illustration, 
Fig.~\ref{fig:deadlock_str_seq} is the directed graph constructed from
Fig.~\ref{fig:tether-intro}. Since cable $C_1$, $C_2$ and $C_3$ have to
wrap around robot $r_2$, $r_3$ and $r_1$ at their target positions
respectively, a loop is formed in the directed graph, and no root node
can be extracted in this case.  Therefore, there is no feasible
solution for the target cable configuration by \strseq{} motion,
leading to a deadlock situation.

\begin{figure} [htp]
\centering
  \includegraphics[width=1.3in]{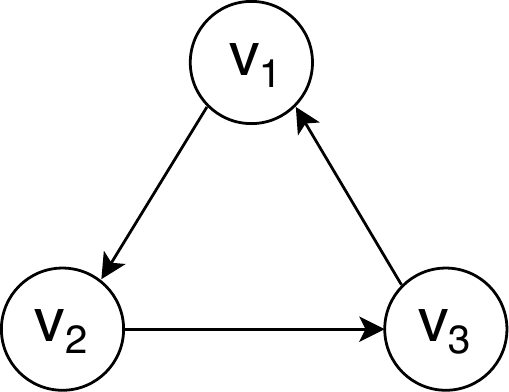}
   \caption{The corresponding directed graph~\cite{hert1994ties} for 
   the target cable configuration in Fig.~\ref{fig:tether-intro}.
   A directed loop implies that it is a deadlock situation for 
   \strseq{} motion.}

 \label{fig:deadlock_str_seq}
\end{figure}

This approach is intuitive, but it fails to detect all possible deadlocks.  
Consider the scenario depicted in Fig.~\ref{fig:straight_sequential}(a), which is a deadlock situation for \strseq{} motion.
However, with its corresponding directed graph as shown in 
Fig.~\ref{fig:straight_sequential}(c),
the robots' motion order is extracted, and eventually leads to
a wrong final cable configuration as shown in
Fig.~\ref{fig:straight_sequential}(b)~\cite{hert1994ties}.

\begin{figure} [htp]
  \subfloat[The given target cable configuration.\label{subfig:target}]{%
    \includegraphics[width=1.5in]{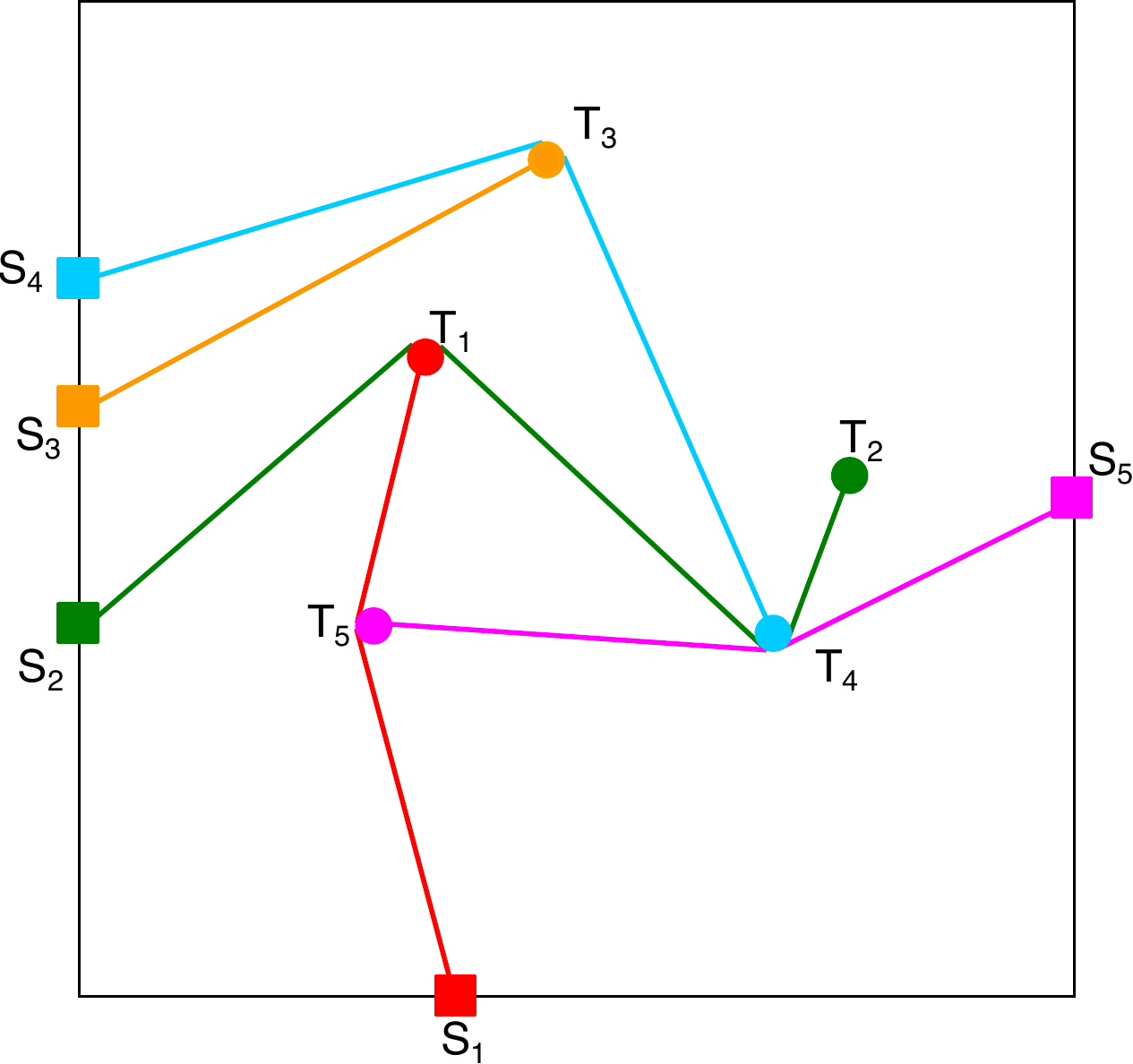}
 }
 \hspace{3mm}
  \subfloat[Actual final cable configuration following motion order extracted from the directed graph.\label{subfig:actual}]{%
    \includegraphics[width=1.5in]{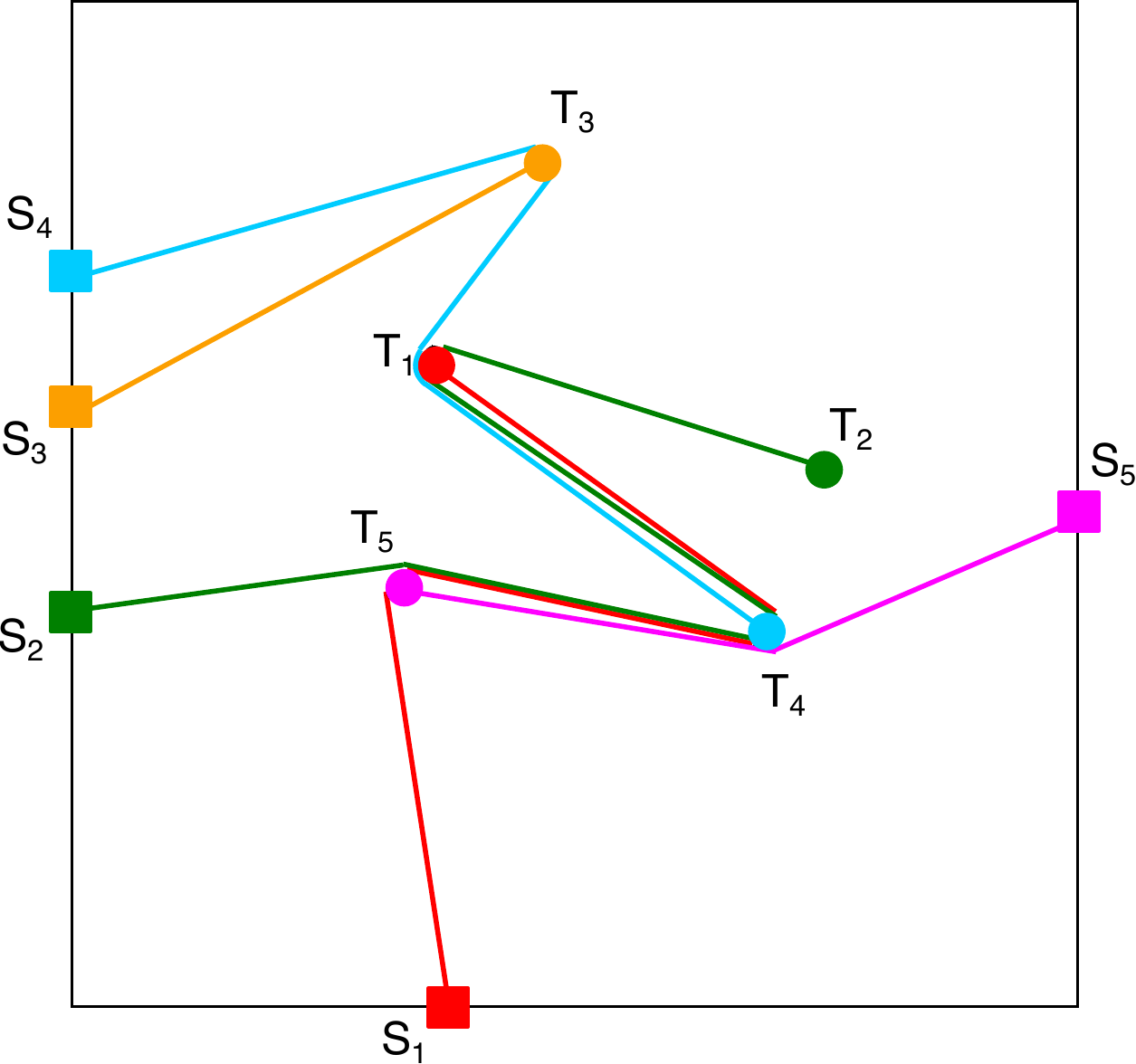}
 }
  \vspace{1mm}
  \centering
  \subfloat[Directed graph for order extraction.\label{subfig:directed_graph}]{%
    \includegraphics[width=2.2in]{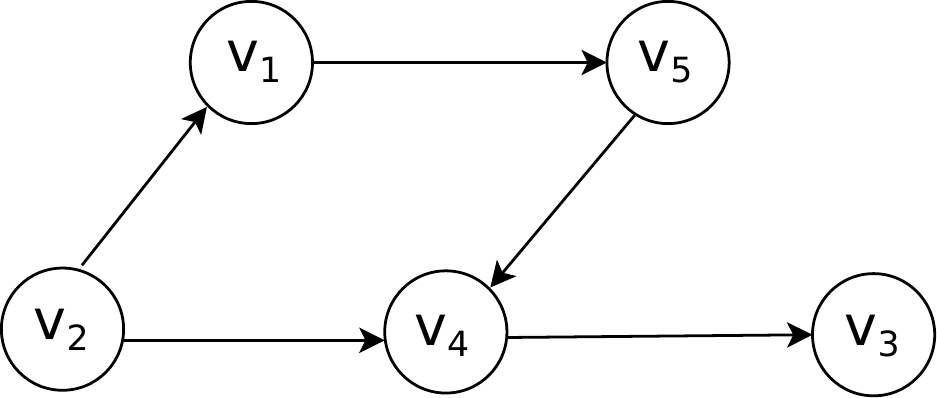}
 }
 
 \caption{An example of deadlock detection failure for \strseq{} motion with the algorithm proposed in~\cite{hert1994ties}.}
 \label{fig:straight_sequential}
\end{figure}

A later work by Hert and Lumelsky~\cite{HERT1996187} supplements the
previous algorithm. In addition to the previous rule,
they consider each pair of robots $r_i$ and $r_j$
that satisfies the following constraints:
\begin{enumerate}
\item there is no directed edge between $v_i$ and $v_j$ in the
  directed graph;
\item the straight-line paths of $r_i$ and $r_j$ cross each other.
\end{enumerate}
Then an edge from $v_i$ to $v_j$ is added to the graph if $T_j$ is in
the cable polygon of $r_i$, and vice versa.  This supplemented
constraint plays an important role in deadlock detection.
Recall the target cable configuration in 
Fig.~\ref{fig:straight_sequential}(a). With this supplemented constraint,
an edge from $v_4$ to $v_1$ should be added to 
Fig.~\ref{fig:straight_sequential}(c), forming a directed
loop $v_1-v_5-v_4-v_1$, which reveals the deadlock nature of the
target cable configuration by \strseq{} motion.  However, this
algorithm is still not complete in detecting all possible deadlock
situations. It fails to include the case when straight-line paths
of robots do not cross each other, but their target points are in each
other's cable polygon.
This case is a deadlock situation for \texttt{Straight} mode, and 
is illustrated in Fig.~\ref{fig:pair-interaction}(e, f).
We define it as \textit{pair deadlock}, and detail the conception in 
Section~\ref{sec:PIG}.

\subsubsection{Realizability by \strconc{}} \label{subsub:straight_concurrent} \hspace*{\fill}

\strseq{} motion improves the efficiency compared to the \texttt{Bent} modes due to its shorter paths, yet it is a waste of time if the target cable configuration can be achieved by concurrent robot motion. 

It is not difficult to infer that the solution space of \strconc{} mode
covers that of \strseq{} mode. In other words, target cable configurations 
that can be achieved by \strseq{} motion can also be achieved by 
\strconc{} motion, with a more efficient solution; deadlock conditions for 
\strconc{} motion must also be deadlock conditions for \strseq{} motion.

Although in~\cite{HERT1996187}, the simultaneous motion of robots is
discussed, it is not leveraged to solve deadlocks.  
Take Fig.~\ref{fig:straight_sequential}(a) as an example. 
As an output from their algorithm, the target cable configuration is a 
deadlock situation, and the path for $r_4$ is $S_4-T_1-T_4$.
However, this target cable configuration is shown to be realizable by 
\strconc{} motion with the algorithms we present in the next section.

\subsubsection{Summary} \label{subsub:summary} \hspace*{\fill}

In this paper, we propose algorithms that (a) detect whether a given
target cable configuration can be achieved by \strconc{} motion, and
(b) return a valid coordinated motion plan, if applies.  When there is 
no feasible solution, some robots will be selected to move along their
target cable lines to maintain the correctness of the final cable
configuration.

\section{Algorithm for \strconc{} motion} \label{sec:proposed-algorithm}

\subsection{Overview of the algorithm} \label{sub:summary}

The flow of the algorithm is as follows (see
Fig.~\ref{fig:algo-flow}): given an input target cable configuration,
\begin{figure}[htp]
\centering
\includegraphics[width=2.55in]{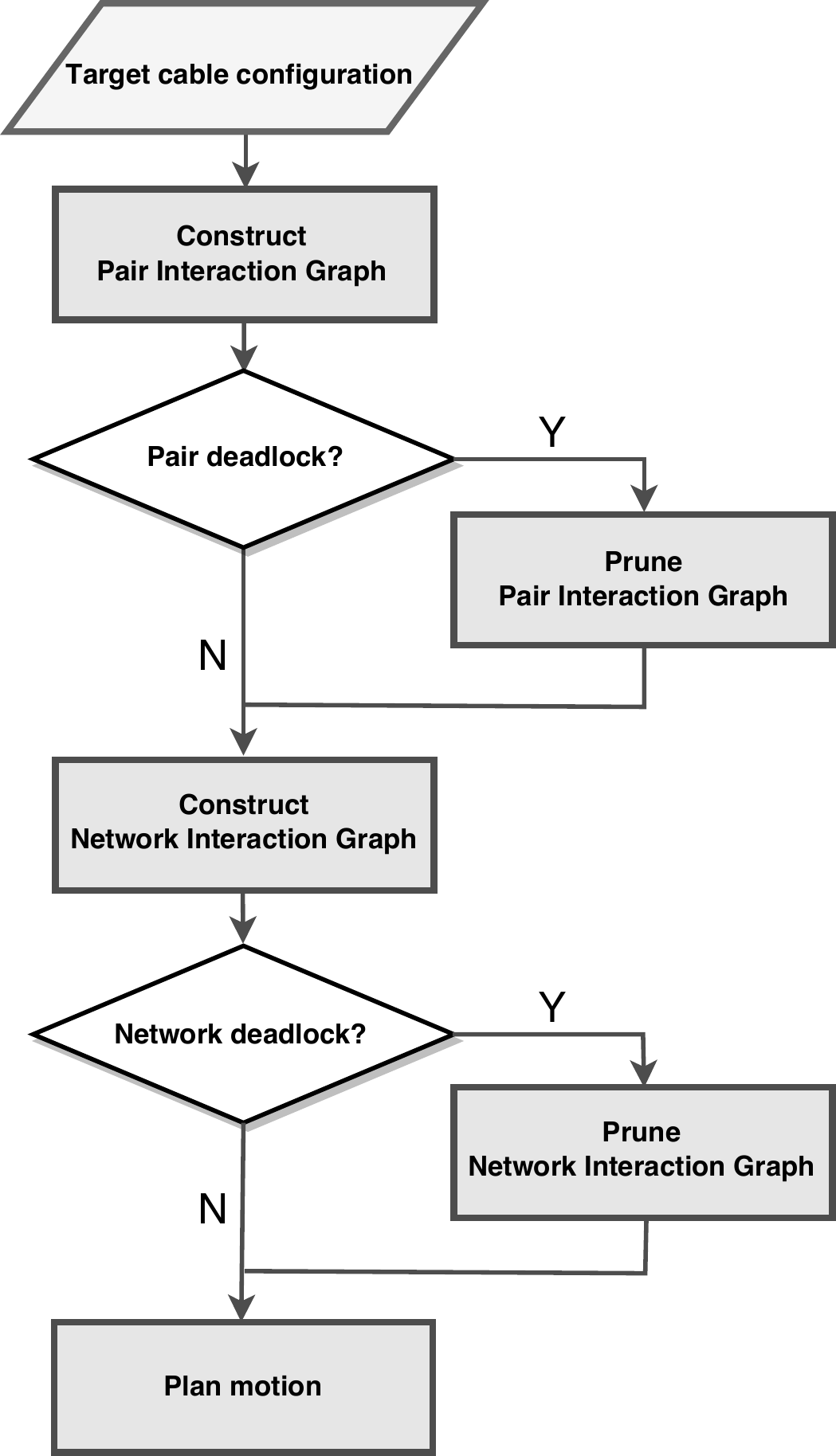}
\caption{Flow of the algorithm. \label{fig:algo-flow}}
\end{figure}

\begin{enumerate}
\item construct the Pair Interaction Graph (PIG), which encodes the
  interactions between pair of robots based on their cable polygons;
\item detect \emph{pair deadlocks} in the PIG, and solve them (prune
  the PIG) by assigning one robot of each pair deadlock to follow its
  target cable line;
\item from the pruned PIG, construct the Network Interaction Graph
  (NIG), which encodes the network interactions induced by the pair
  interactions, taking into account the intersections of robot paths;
\item detect \emph{network deadlocks} in the NIG, and solve them
  (prune the NIG) by assigning one robot involved in
  each network deadlock to follow its cable line;
\item from the pruned NIG, compute the final motion schedule.
\end{enumerate}

The different steps of the algorithm are detailed in the following
sections.

\subsection{Constructing the Pair Interaction Graph (PIG)} 
\label{sec:PIG}

We consider first the elementary interactions that occur between a
pair of robots. Those interactions are determined by the position of
one robot's target with respect to the other robot's cable polygon
(recall from Assumption~\ref{assumption:starting-positions} that a
robot's starting position is always outside other robots' cable
polygons).  Given a pair of robots $(r_i,r_j)$, there are thus four
types of pair interactions: (1) $T_i \notin \Pi_j$, $T_j\notin\Pi_i$
[Fig.~\ref{fig:pair-interaction}(a)]; (2) $T_i\in\Pi_j$,
$T_j\notin\Pi_i$ [Fig.~\ref{fig:pair-interaction}(c)]; (3)
$T_i\notin\Pi_j$, $T_j\in\Pi_i$; (4) $T_i\in\Pi_j$, $T_j\in\Pi_i$
[Fig.~\ref{fig:pair-interaction}(e)]. Note that types  2 and 3 are
symmetrical and will be treated together.

\begin{figure}[htp]
  \centering
  \includegraphics[width=3.5in]{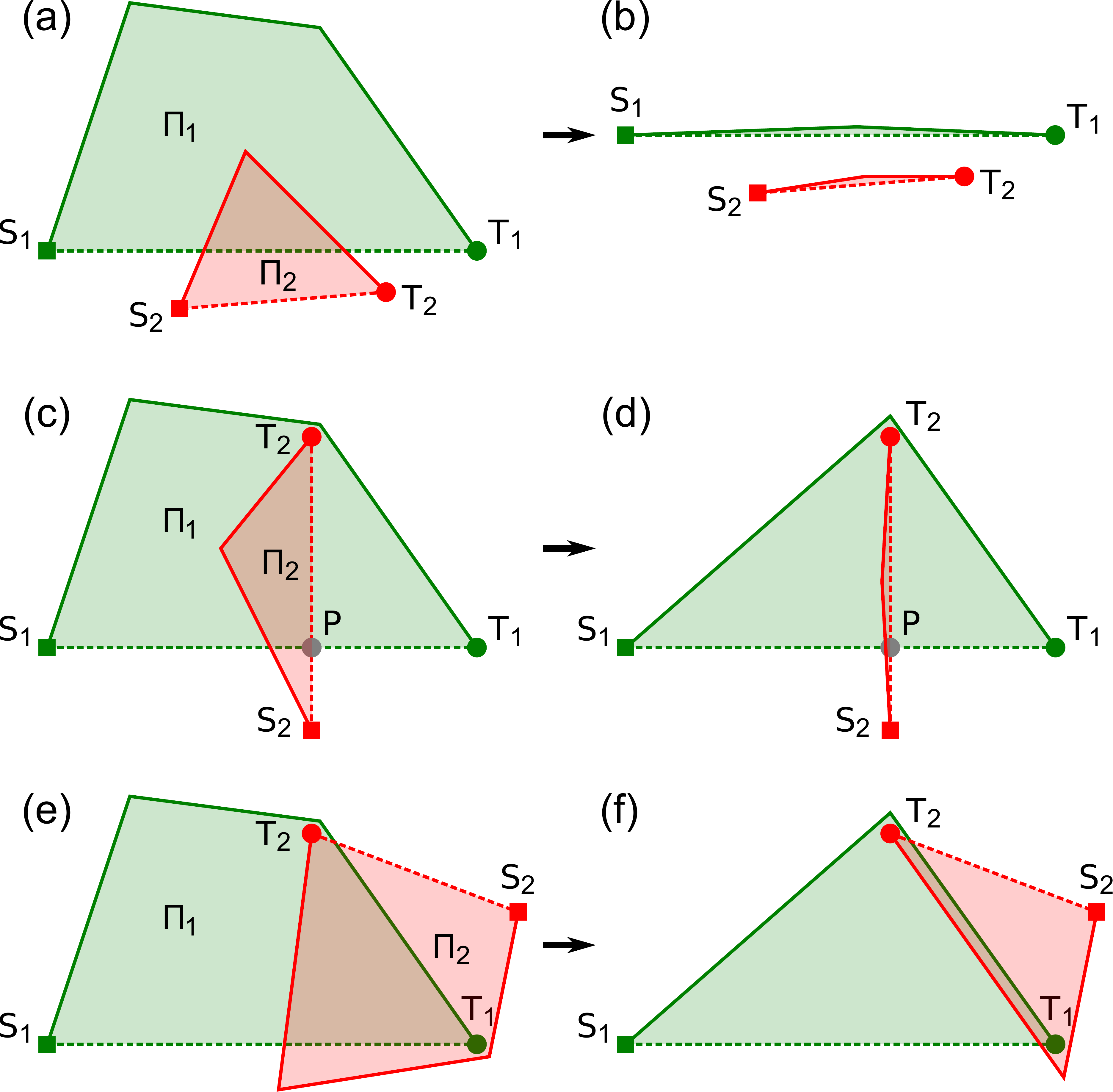} 
  \caption{Pair robot interactions are determined by the position of
    targets with respect to cable polygons. Left (a, c, e): cable
    polygons for type (1), (2-3), and (4) respectively. Right (b, d,
    f): retracted cable polygons after removing robots that are not
    involved in the pair interaction.}
  \label{fig:pair-interaction}
\end{figure}

Since we are interested in pair interactions, we suppose that the
other robots are temporarily removed from the problem (multi-robot
interactions will be addressed in Section~\ref{sec:NIG}). As a result,
the cable polygons are \emph{retracted} as in
Fig.~\ref{fig:pair-interaction}(b, d, f). Note that the retraction
process does not affect the interaction types just described.

The following propositions establish the motion priorities induced by
each type of pair interactions.

\begin{proposition}If $T_i \notin \Pi_j$ and  $T_j\notin\Pi_i$, then there
  are no priority induced by the pair interaction.
\end{proposition}

\begin{proof}
  We prove in~\ref{appendix:cable-polygons} that, in this
  case, the two segments $(S_i, T_i)$ and $(S_j, T_j)$ do not
  intersect. Thus, the retracted cable polygons are reduced to two
  independent segments [Fig.~\ref{fig:pair-interaction}(b)], inducing no
  motion priority between $r_i$ and $r_j$.
\end{proof}

\begin{proposition}If $T_i\in\Pi_j$ and $T_j\notin\Pi_i$, then, with $P$
  denoting the intersection point between $(S_i, T_i)$ and
  $(S_j, T_j)$, $r_j$ must pass $P$ before $r_i$ does.
\end{proposition}

\begin{proof}
  We prove in~\ref{appendix:cable-polygons} that, in this
  case, the two segments $(S_i, T_i)$ and $(S_j, T_j)$ must intersect,
  at a point $P$. One can then
  see that $r_i$ must pass $P$ before $r_j$
  does. Fig.~\ref{fig:pair-interaction}(d) illustrates this statement:
  the motion paths $(S_1,T_1)$ and $(S_2,T_2)$ intersect at $P$ and
  $T_2\in\Pi_1$. Clearly, $r_1$ must pass $P$ before $r_2$ does, so
  that $r_2$ can ``push'' the cable line of $r_1$.
\end{proof}

\begin{proposition}If $T_i\in\Pi_j$ and $T_j\in\Pi_i$, then there is a
  deadlock, which we call a \emph{pair deadlock}. One of the two
  robots must diverge from the straight line motion.
\end{proposition}

\begin{proof}
  We prove in~\ref{appendix:cable-polygons} that, in this
  case, the two segments $(S_i, T_i)$ and $(S_j, T_j)$ do not
  intersect. This case implies a deadlock situation, which we call
  \emph{pair deadlock}, as illustrated in
  Fig.~\ref{fig:pair-interaction}(f).  The motion paths $(S_1, T_1)$
  and $(S_2, T_2)$ do not intersect, but $T_1\in\Pi_2$ and
  $T_2\in\Pi_1$.  When $r_1$ and $r_2$ move along their motion paths,
  the final cable configuration will be two independent segments
  $(S_1, T_1)$ and $(S_2, T_2)$, which are different from their
  retracted polygons as shown in Fig.~\ref{fig:pair-interaction}(f).
  In this case, $r_1$ ($r_2$) has to diverge from the straight line
  motion and move around $T_2$ ($T_1$) in order to push $\calC_2$
  ($\calC_1$) to the desired target cable configuration.
\end{proof}

Based on the above propositions, we can formulate
Algorithm~\ref{algo:directed_layout_graph} to encode pair
interactions. The algorithm constructs a directed graph, termed the
Pair Interaction Graph (PIG), whose vertices represent the robots, and
whose edges encode motion priority relationships: (1) if there is a
unidirectional edge pointing from vertex $v_i$ to vertex $v_j$, then
robot $r_i$ must pass the intersection point before $r_j$; (2) if
there is a bidirectional edge between vertex $v_i$ and vertex $v_j$,
then there is a pair deadlock.

\begin{algorithm}
\caption{Constructing the Pair Interaction Graph}
\label{algo:directed_layout_graph}
\begin{algorithmic}[1]
  \renewcommand{\algorithmicrequire}{\textbf{Input:}}
  \renewcommand{\algorithmicensure}{\textbf{Output:}}
  \REQUIRE Target cable configuration
  \ENSURE Pair Interaction Graph (PIG)

  \FOR{$i=1:n$}
     \STATE Add a vertex $v_i$ for each robot $r_i$\;
  \ENDFOR

  \FOR{$i=1:n-1$}
    \FOR{$j=i+1:n$}
      \IF{$(S_i, T_i)$\ \AND $(S_j, T_j)$\
                          intersect at $P$}
        \IF{$T_i \in \Pi_j$}
          \STATE Add an edge from $v_j$
                                        to $v_i$, labeled
                                        $\{v_j\stackrel{P}{\to}
                                        v_i\}$\; 
        \ELSE
          \STATE Add an edge from $v_i$
                                        to $v_j$, labeled 
                                        $\{v_i\stackrel{P}{\to}
                                        v_j\}$\; 
        \ENDIF
      \ELSIF{$T_i\in\Pi_j$ \AND $T_j\in\Pi_i$}
          \STATE Add a bi-directed edge between $v_i$ and $v_j$\;
      \ENDIF
    \ENDFOR
  \ENDFOR
\end{algorithmic}
\end{algorithm}

\subsection{Detecting and solving pair deadlocks}

When a pair deadlock occurs, one of the two robots must diverge from
the straight line motion. Consider the deadlock case of
Fig.~\ref{fig:pair-interaction}(f): the deadlock can be solved by (i)
one robot, say $r_1$, moves in a straight line from $S_1$ to $T_1$;
(ii) the other robot, $r_2$, follows its target cable line (start from
$S_2$, wrap around $T_1$, then go to $T_2$). Although motion paths
other than the robot target cable line are possible in stage (ii), we
shall not consider them to avoid too much complexity.

Accordingly, we propose the following steps to address pair
deadlocks:
\begin{enumerate}
\item Detect all pair deadlocks (by finding all bi-directional edges
  in the PIG), and put the robots involved in those deadlocks into a
  list $\widetilde{L}$;
\item Iteratively extract one robot, say $r_i$, from $\widetilde{L}$;
  put it in a list $L^*$; remove $r_i$ completely from the problem;
  update accordingly the PIG and $\widetilde{L}$; repeat until
  $\widetilde{L}$ is empty.
\end{enumerate}

At the end of this algorithm, one has a list $L^*$ of robots that will
be forced to follow their target cable lines after all other robots
have moved, and a PIG that no longer contains pair deadlocks (pruned
PIG).

In step~2, there are diverse methods to choose which robot to remove
first from the problem. One can for instance choose the robot that is
involved in the largest number of pair deadlocks, or the one whose
target cable line generates the largest detour as compared to the
straight line.

\subsection{Constructing the Network Interaction Graph (NIG)}
\label{sec:NIG}

We have seen in Section~\ref{sec:PIG} that pair interactions of type~1
do not induce priority, while pair interactions of type 4 induce
deadlock situations, which are subsequently solved by forcing one of
the robots of the pairs to follow its target cable line. As also
discussed, pair interactions of types 2 and 3 give rise to priority
relationships between the two robots of the pair at their intersection
point. These relationships can in turn lead to \emph{network}
interactions if some robots are simultaneously involved in more than
one pair interactions.

To illustrate, consider the scenario depicted in
Fig.~\ref{fig:network-interaction}(a). The pair interactions give rise
to the PIG of Fig.~\ref{fig:network-interaction}(c), which includes a
circular priority chain. Yet, this situation is not deadlocked, since
the following coordinated motion leads to the correct target cable
configuration: $r_1$, $r_2$ and $r_3$ move along their $(S_i, T_i)$
motion paths at the same velocity. By doing so, they will indeed pass
respectively points $A$, $B$ and $C$ at the same time, which in turn
satisfies the requirement that $r_1$ passes $A$ before $r_2$ does,
$r_2$ passes $B$ before $r_3$ does, and $r_3$ passes $C$ before $r_1$
does.
 
Consider now the scenario depicted in
Fig.~\ref{fig:network-interaction}(b). The pair interactions give rise
to the same PIG of Fig.~\ref{fig:network-interaction}(c). However, one
can see that the situation is deadlocked.  The intersection points
between each pair of robots have different relative positions along
robots' $(S_i, T_i)$ motion paths now: $C$ is closer to $S_1$ than
$A$, $A$ is closer to $S_2$ than $B$, and $B$ is closer to $S_3$ than
$C$.  Indeed, the constraint imposed by the three directed edges in
PIG (see Fig.~\ref{fig:network-interaction}(c)) implies that $r_1$
passes $C$ before $r_2$ passes $A$, $r_2$ passes $A$ before $r_3$
passes $B$, and $r_3$ passes $B$ before $r_1$ passes $C$.  It is thus
clear that the situation is deadlocked.

\begin{figure}[htp]
\centering
  \subfloat[A non-deadlock configuration for \strconc{}.\label{subfig:non_deadlock}]{%
    \includegraphics[width=1.55in]{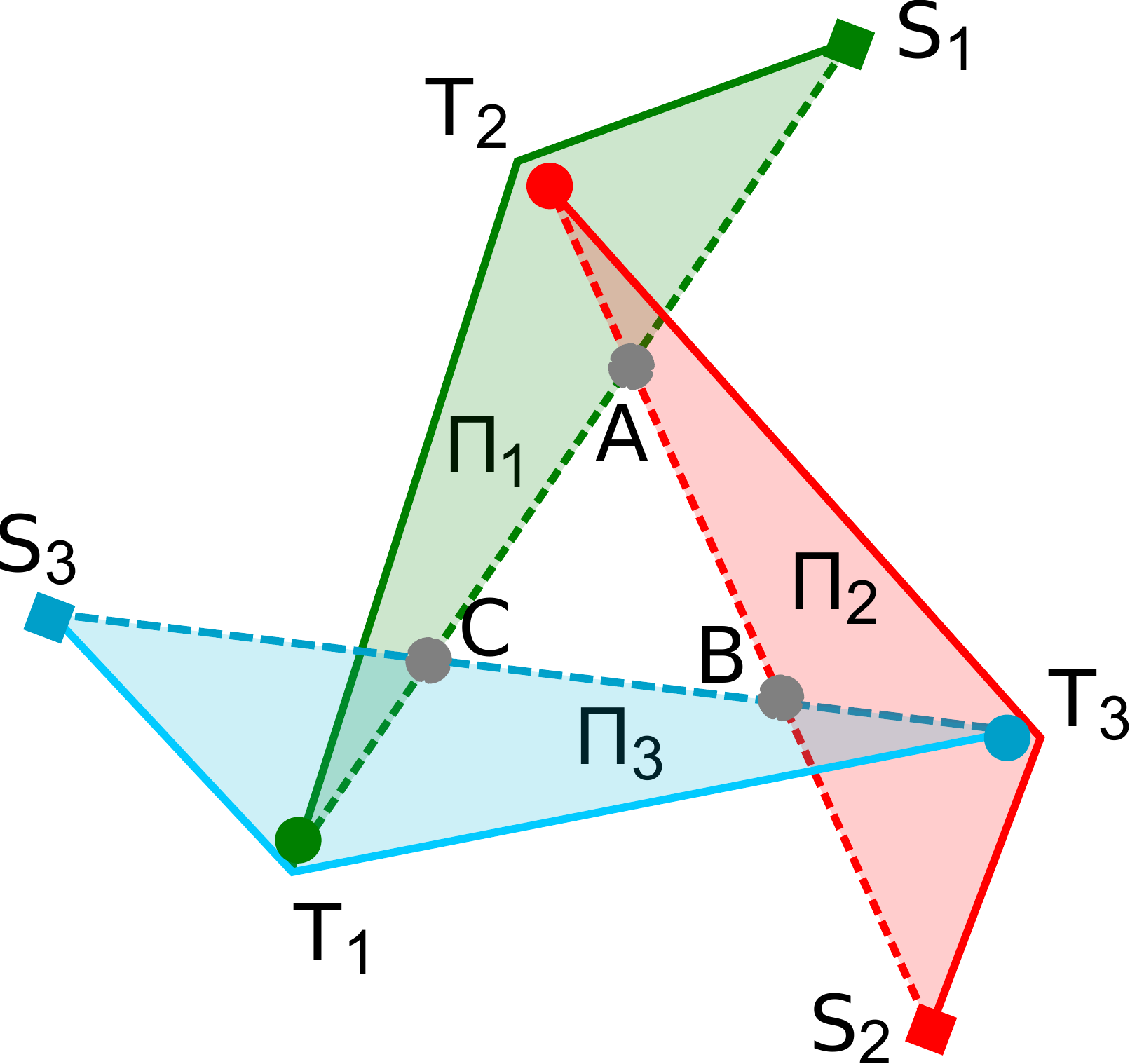}
 }
  \hspace{4mm}
    \subfloat[A deadlock configuration for \strconc{}. \label{subfig:has_deadlock}]{%
    \includegraphics[width=1.4in]{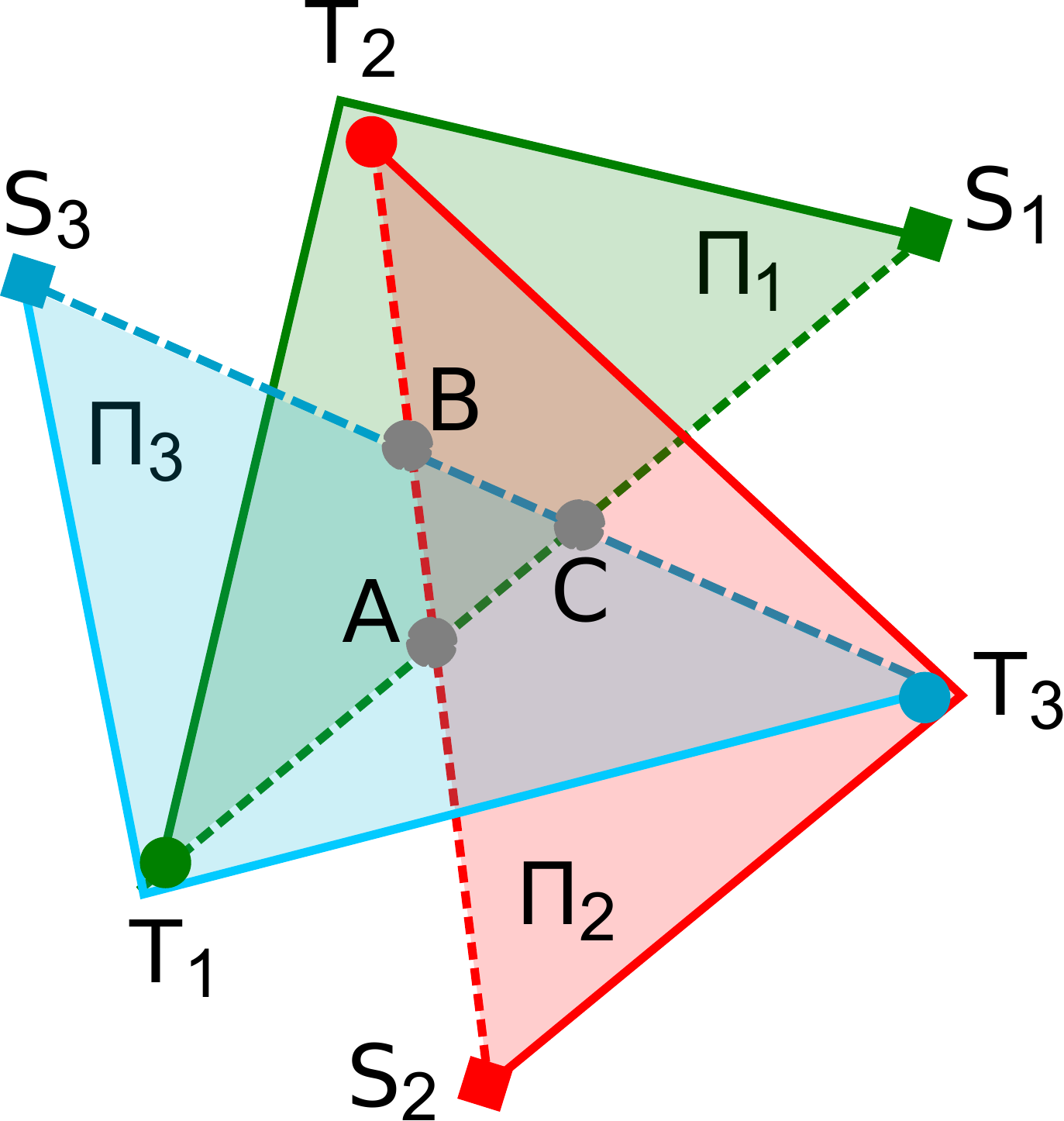}
 }

  \vspace{1mm}
  \subfloat[Pair interaction graph (same for both configurations). \label{subfig:graph_solve}]{%
    \includegraphics[width=1.4in]{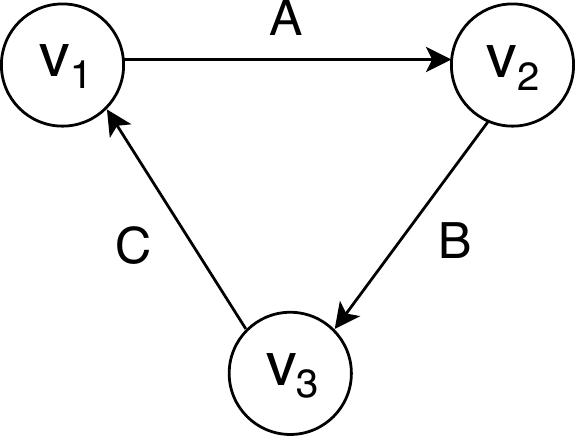}
 }

  \vspace{1mm}
  \subfloat[Network interaction graph for non-motion deadlock case.\label{subfig:dl-loop-detect-non}]{%
    \includegraphics[width=1.6in]{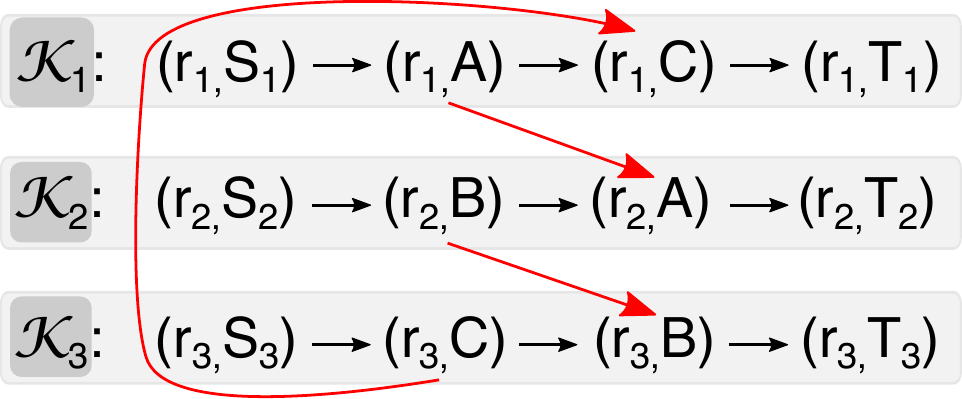}
 }
 \hspace{1mm}
  \subfloat[Network interaction graph for motion deadlock case.\label{subfig:dl-loop-detect-has}]{%
    \includegraphics[width=1.6in]{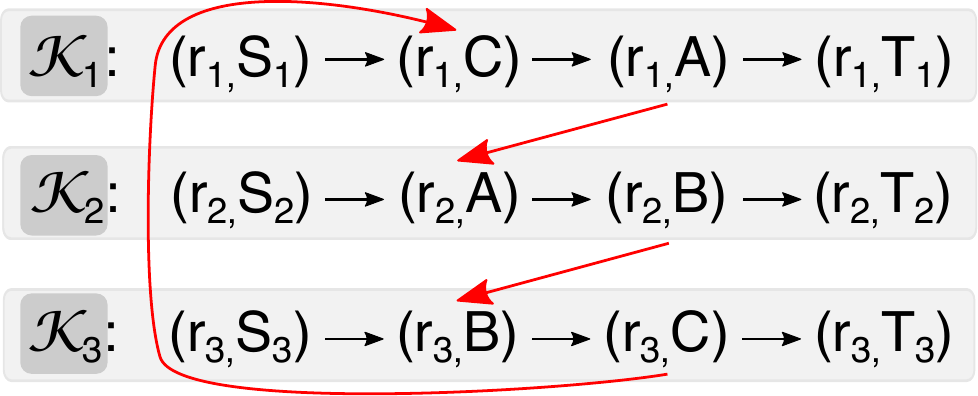}
 }
 \caption{Network interactions are determined, not only by the priority
   relationship between the robots at their intersection point, but
   also by the relative arrangement of the intersection points.}
 \label{fig:network-interaction}
\end{figure}

It appears from the two scenarios discussed above that correctly
handling network interactions requires taking into account, not only
the priority relationship between the robots at their intersection
points, but also the \emph{relative arrangement of those intersection
  points}. To encode this information, we propose to construct the
Network Interaction Graph (NIG) whose vertices are the \emph{events}
``robot $r_i$ passes intersection point $P$'', denoted $(r_i,P)$, and
whose directed edges encode the temporal relationships among the
events, as follows [see Fig.~\ref{fig:network-interaction}(d, e) for
illustrations]:

\begin{enumerate}
\item For each robot $r_i$ of the PIG, add to the NIG a chain
  $\calK_i$:
  $(r_i,S_i) \to (r_i,P_1) \to \dots \to (r_i,P_m) \to (r_i,T_i)$,
  where $P_1, \dots, P_m$ are the ordered intersection points along
  the segment $(S_i, T_i)$;
\item For each edge $\{v_i\stackrel{P}{\to} v_j\}$ of the PIG, add to
  the NIG a directed edge from $(r_i,P)$ (in $\calK_i$) to
  $(r_j,P)$ (in $\calK_j$).
\end{enumerate}

\subsection{Detecting and solving network deadlocks}

We can now state and prove the following proposition.

\begin{proposition}
  If the NIG contains a directed cycle, then there exists a network
  deadlock.
\end{proposition}
\begin{proof}
  Note that an edge $(r_i,P)\to (r_i,P')$ added in step 1 of the
  algorithm to construct the NIG (intra-chain edge) encodes the
  requirement that robot $r_i$ must pass point $P$ before point
  $P'$. We note this relationship $(r_i,P) < (r_i,P')$.

  An edge $(r_i,P)\to (r_j,P)$ added in step 2 (inter-chain edge)
  encodes the requirement that robot $r_i$ must pass point $P$ before
  robot $r_j$ passes $P$. We note this relationship
  $(r_i,P) < (r_j,P)$.

  Observe that the relationship ``$<$'' just defined, by its temporal
  nature, is a strict order relationship in the set of events, i.e. it
  verifies the irreflexivity, antisymmetry, and transitivity
  properties.

  Consider now a directed cycle in the NIG. Such a cycle induces a
  cycle for the strict order relationship ``$<$'', which is clearly
  not satisfiable. Therefore, there exists at least one priority
  requirement that cannot be satisfied in the cycle, i.e. a deadlock
  situation.
\end{proof}

When there is a network deadlock, one of the involved robots has to
diverge from the straight-line motion, similarly to the pair deadlock
case. Accordingly, we propose the following steps to address network
deadlocks:
\begin{enumerate}
\item Detect all network deadlocks (by finding all directed cycles in
  the NIG), and put the robots involved in those deadlocks into a list
  $\widetilde{L}$;
\item Iteratively extract one robot, say $r_i$, from $\widetilde{L}$;
  put it in the list $L^*$ (the same as used previously for solving
  pair deadlocks); remove $r_i$ completely from the problem; update
  accordingly the NIG and $\widetilde{L}$; repeat until
  $\widetilde{L}$ is empty.
\end{enumerate}

At the end of this algorithm, one has a list $L^*$ of robots that will
be forced to follow their target cable lines after all other robots
have moved, and a NIG that no longer contains network deadlocks (pruned
NIG).

Similarly to the case of pair deadlocks, there are diverse methods, in
step~2, to choose which robot to remove first from the problem. One
can for instance choose the robot that is involved in the largest
number of network deadlocks, or the one whose target cable line
generates the largest detour as compared to the straight line.

\subsection{Computing the final motion schedule} \label{sec:schedule}

When we coordinate the motion of a pair of robots at the intersection
point of their $(S, T)$ motion paths, we could either make the 
lower priority robot wait until the higher priority robot 
passes the point, or adjust the speed of the robots to produce 
continuous motions.  In this work, the waiting scheme is adopted for
illustration. A coordinated motion plan with shortest traveling time 
is generated if the target cable configuration can be achieved by 
\strconc{} robot motion.

In the following procedures, $(r_i, P^-)$ and $(r_i, P^+)$ denote the
predecessor and successor of $(r_i, P)$ on the chain $\calK_i$
respectively.  $\calT_i$ represents the time-line for robot $r_i$, and
$T_{i,P}$ is the time-stamp on $\calT_i$, upon which robot $r_i$ starts
to navigate from $(r_i, P)$ to $(r_i, P^+)$.  Weight $w_{i,P}$ of node
$(r_i, P)$ indicates the amount of time that $r_i$ takes to travel
from $(r_i, P^-)$ to $(r_i, P)$.

From the NIG (which we suppose does not contain any cycle), we can
extract a coordinated motion schedule for the robots to achieve the
target cable configuration as follows:
\begin{enumerate} 
\item For each chain $\calK_i$ of the NIG, $i = 1, \dots, n$, calculate
  and assign the weight to every node (weight is $0$ for $(r_i, S_i)$);
\item Initialize time-line $\calT_i$ for each robot $r_i$,
  $i = 1, \dots, n$;
\item For the root node $(r_i, R)$ in $\calK_i, i = 1, \dots, n$, 
  \begin{enumerate}
    \item[(a).] If there is no incoming inter-chain edge from any 
      $\calK_j, i \neq j$, append traveling time $w_{i,R}$ to $\calT_i$, 
      and remove $(r_i, R)$ from the NIG; 
    \item[(b).] If there is an incoming inter-chain edge from an already
      removed node $(r_j, P)$ in $\calK_j$, $w_{i,R}$ must be
      appended to $\calT_i$ at a time-stamp that is \emph{no earlier than}
      $T_{j,P}$\footnote{Here we use a relaxed chronological order
      \emph{no earlier than} for a directed inter-chain edge from
      $(r_j, P)$ in $\calK_j$ to $(r_i, R)$ in $\calK_i$ in order to
      exempt $r_i$ from the unnecessary waiting time for the whole
      duration $w_{j,P}$.}, after which $(r_i, R)$ and this incoming 
      inter-chain edge are removed from the NIG;
    \end{enumerate}
\item Repeat step 3 until there is only the node $T_i$ left in each chain $\calK_i, i = 1, \dots, n$ of the NIG;
\item Make robots in $L^*$ move along their target cable lines after
  all other robots have reached their target positions.
\end{enumerate}

A detailed example is presented in Section~\ref{sec:experiments}.

\subsection{Complexity of the algorithms} \label{sec:Complexity}

We now analyze the complexity of the algorithms just presented as a
function of the number $n$ of robots. Note first that, for each robot
$r_i$, the number of vertices (and edges) of its cable polygon $\Pi_i$
is $O(n)$ since it is bounded by the total number of robots (the
cables can only wrap around other robots).

\subsubsection{Pair Interaction Graph} 

To construct the PIG, the step with the highest complexity is to check
whether $T_i\in\Pi_j$ for all pairs $(i,j)$. The test $T_i\in\Pi_j$
has a linear complexity in the number of edges of $\Pi_j$, which is
$O(n)$. Thus, the construction of the PIG has a complexity of
$O(n^3)$.

Detecting pair deadlocks involves finding all bi-directional edges in
the PIG and can therefore be done in $O(n^2)$.

There are potentially $O(n)$ deadlocked robots. For a deadlocked
robot, say robot $r_i$, one needs to remove it and update the PIG
accordingly, which can be done in $O(n^2)$. Thus, pruning the PIG can
be done in $O(n^3)$.

\subsubsection{Network Interaction Graph} 

To construct the NIG, the main step is to compute, for each robot $r_i$,
the intersections of $(S_i,T_i)$ with all the $(S_j,T_j)$ and order
those intersections along $(S_i,T_i)$. This can be done in
$O(n^2\log(n))$. 

Detecting network deadlocks involves checking whether the NIG has a
cycle, which can be done in $O(n^2)$ (topological sort). Pruning the
NIG can be done in $O(n^3)$.

\subsubsection{Motion schedule}

The algorithm for scheduling the robot motions is linear in the number
of nodes of the NIG and is therefore $O(n^2)$.

From the above analysis, the worst time-complexity of the full
pipeline is $O(n^3)$.

\section{Examples} \label{sec:experiments}

\subsection{An example from Hert \& Lumelsky 1996}

In this section, we validate our proposed algorithms with the example
previously shown in Fig.~\ref{fig:straight_sequential}(a), which is a
deadlock situation for \strseq{} motion.  We show that the target
cable configuration can be achieved by \strconc{} motion and $r_4$
does not need to take a detour as in the output of the algorithm
in~\cite{HERT1996187}.

The intermediate stages of achieving the target cable configuration
with our proposed algorithms are shown in Fig.~\ref{fig:example}.  In
addition, a step-wise illustration of motion scheduling is presented
in Fig.~\ref{fig:motion-schedule}.

\begin{figure} [htp]
\centering
  \subfloat[The given target cable configuration.\label{subfig:example-target}]{%
    \includegraphics[width=1.55in]{target}
 }
  \hspace{1mm}
    \subfloat[(S, T) straight-line segments of the robots with labels for intersection points.\label{subfig:solve_str_seq}]{%
    \includegraphics[width=1.63in]{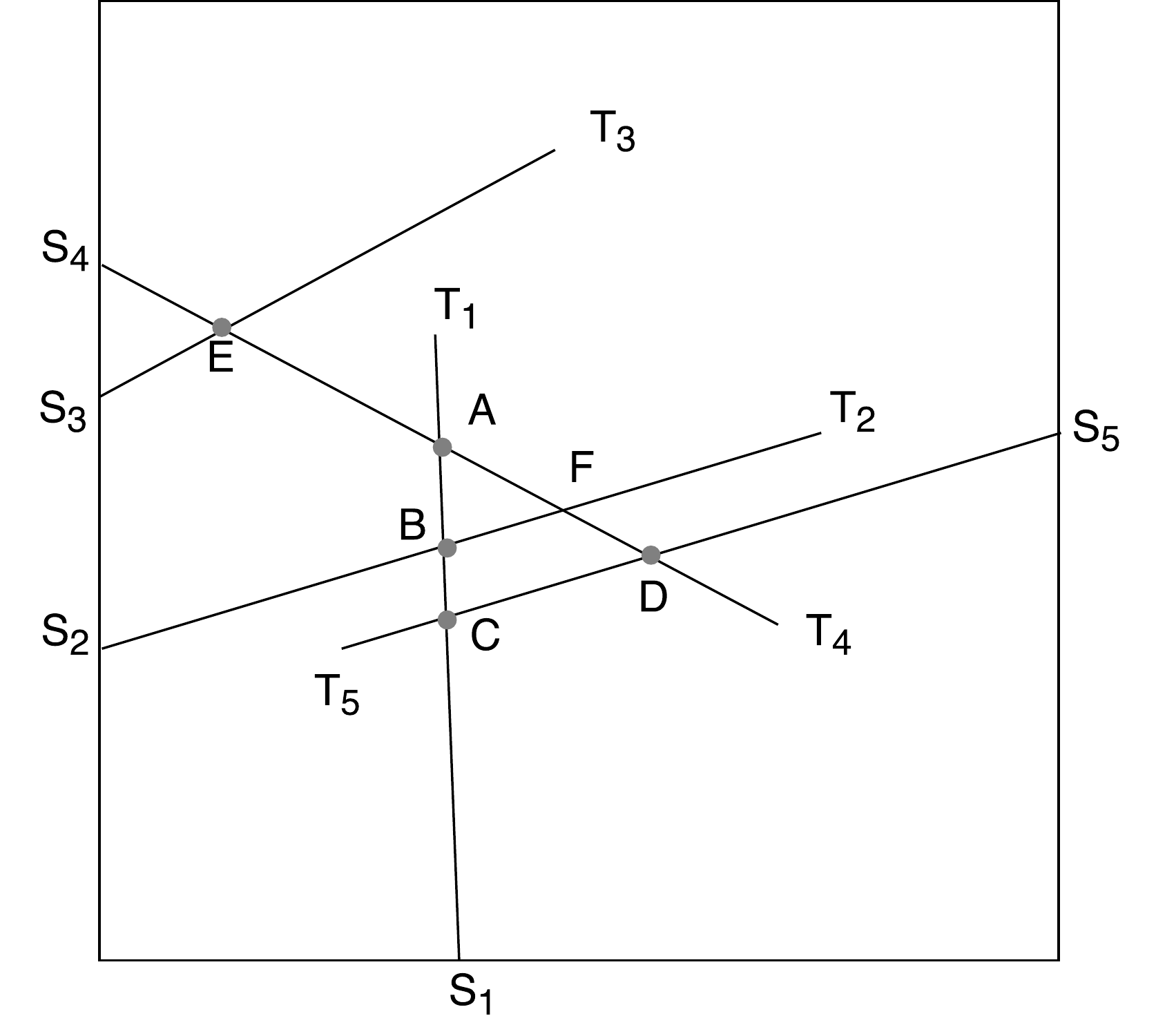}
 }

 \vspace{1mm}
  \subfloat[Pair interaction graph.\label{subfig:graph_solve_str_seq}]{%
    \includegraphics[width=1.4in]{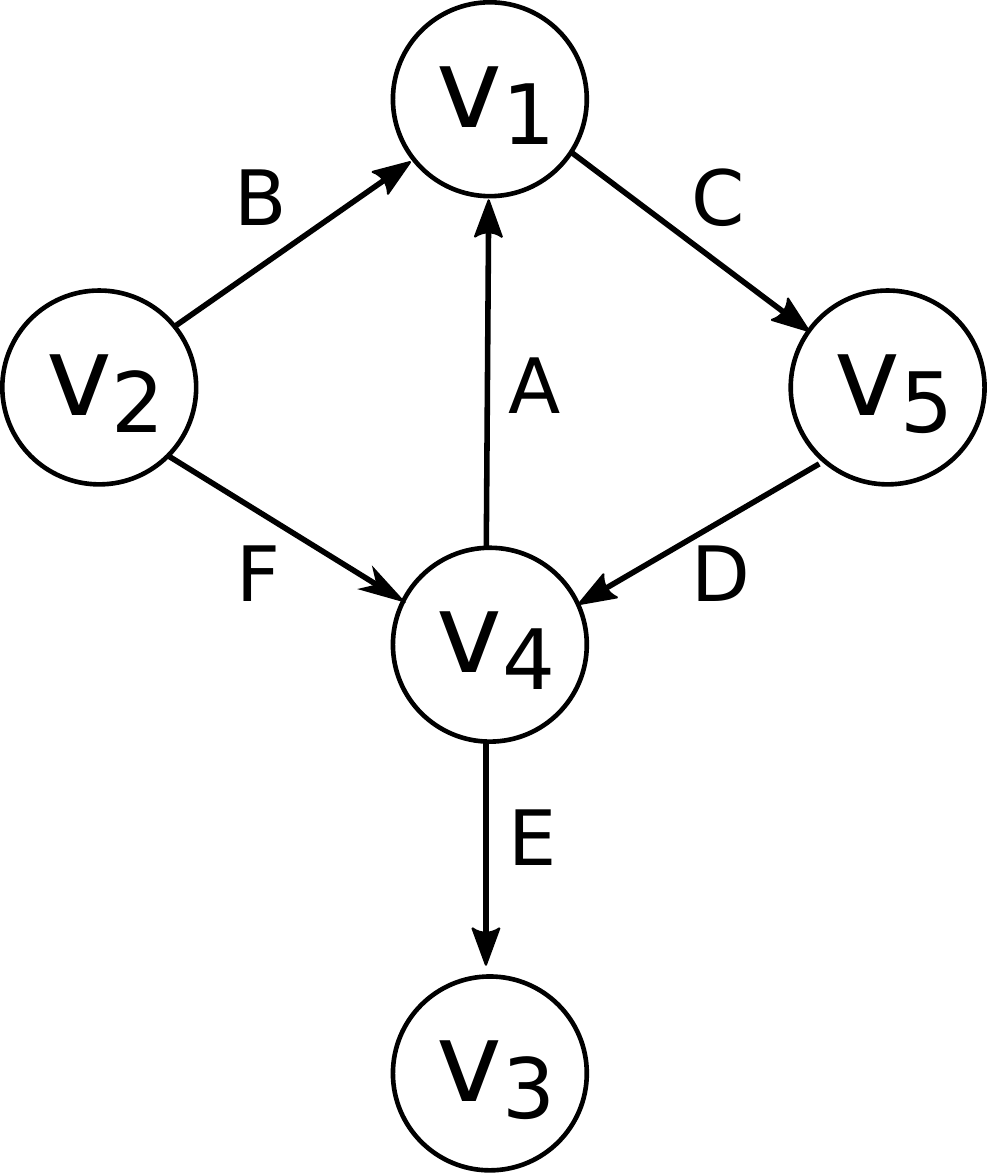}
 }

 \vspace{1mm}
  \subfloat[Network interaction graph.  \label{subfig:example-NIG}]{%
  \includegraphics[width=2.8in]{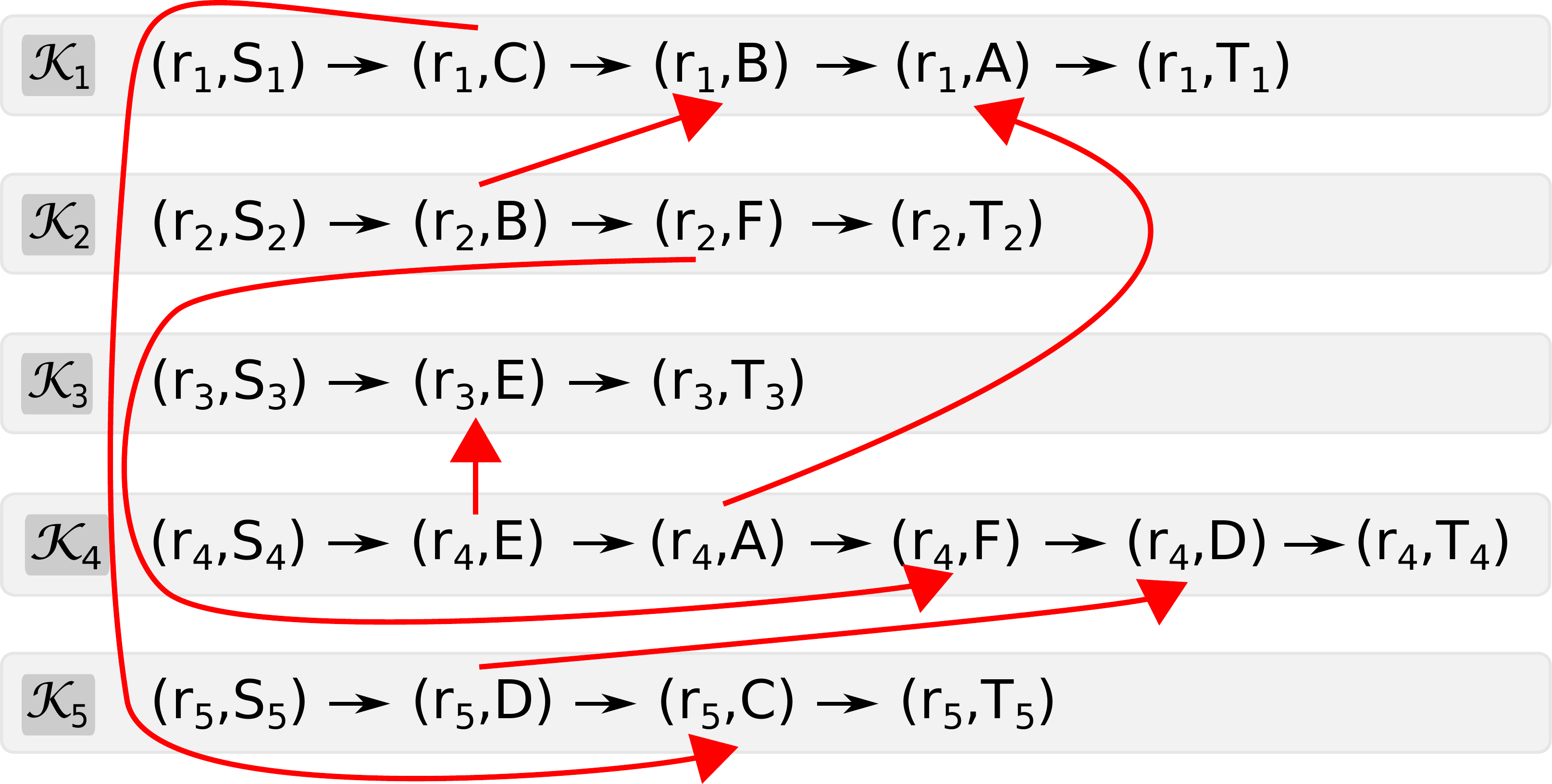}
}
\caption{The proposed algorithms are validated with a deadlocked target cable configuration for \strseq{} motion.}
\label{fig:example}
\end{figure}

To avoid the wrong final cable configuration as shown in
Fig.~\ref{fig:straight_sequential}(b), $r_4$ must pass $A$ before
$r_1$ does. It is not possible in \strseq{} motion, since $r_1$ has to
reach its target position before $r_4$ does.  However, moving in
\strconc{} motion following the local navigation priority as indicated
in Fig.~\ref{fig:example}(d), which contains no directed cycle, the
target cable configuration can be achieved successfully.

\begin{figure*} [htp]
\centering
  \subfloat[Initialized time-lines for the robots.]{%
  \includegraphics[width=1.6in]{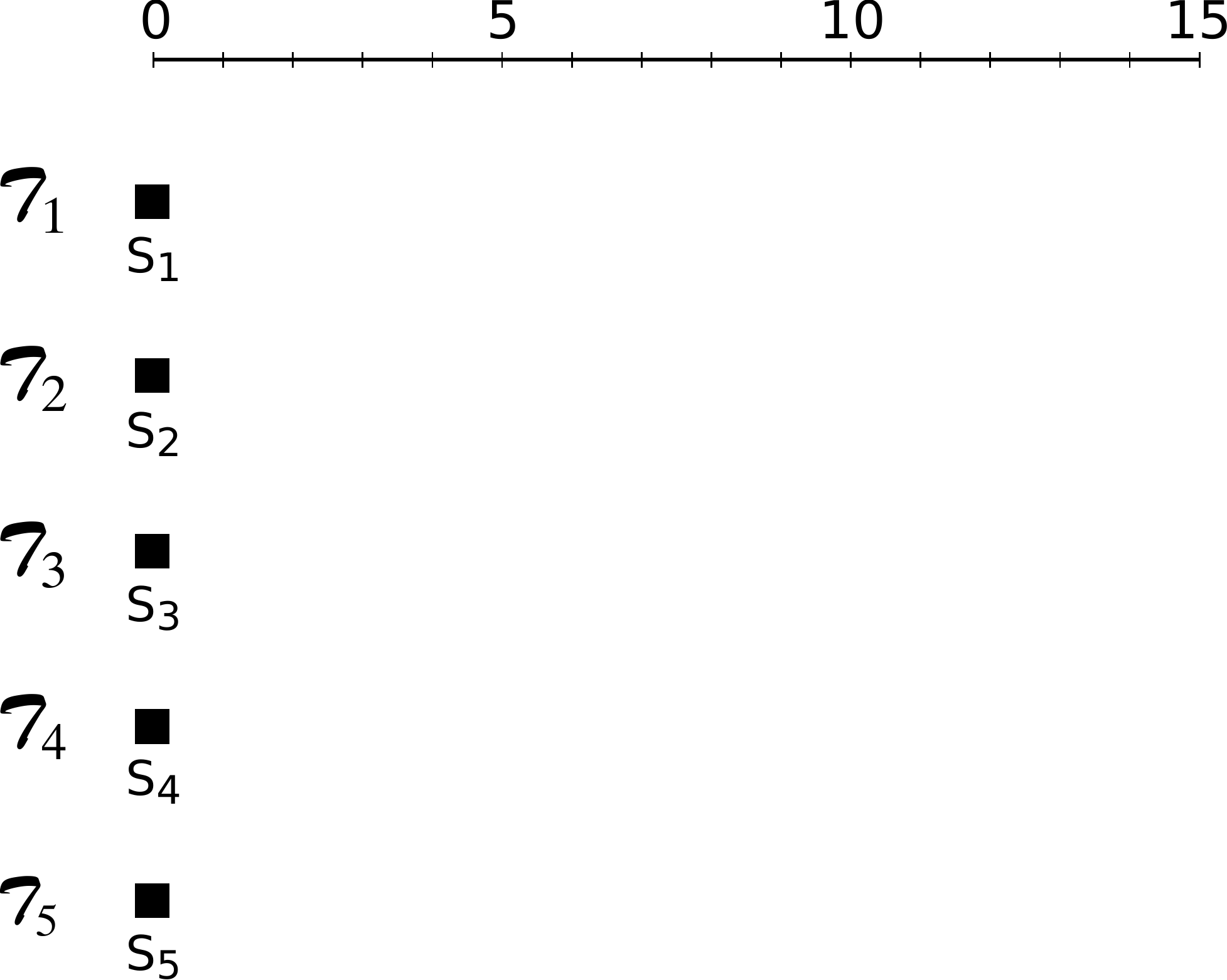}
}
\hspace{5mm}
  \subfloat[The weighted NIG.\label{subfig:schedule-0}]{%
  \includegraphics[width=2.9in]{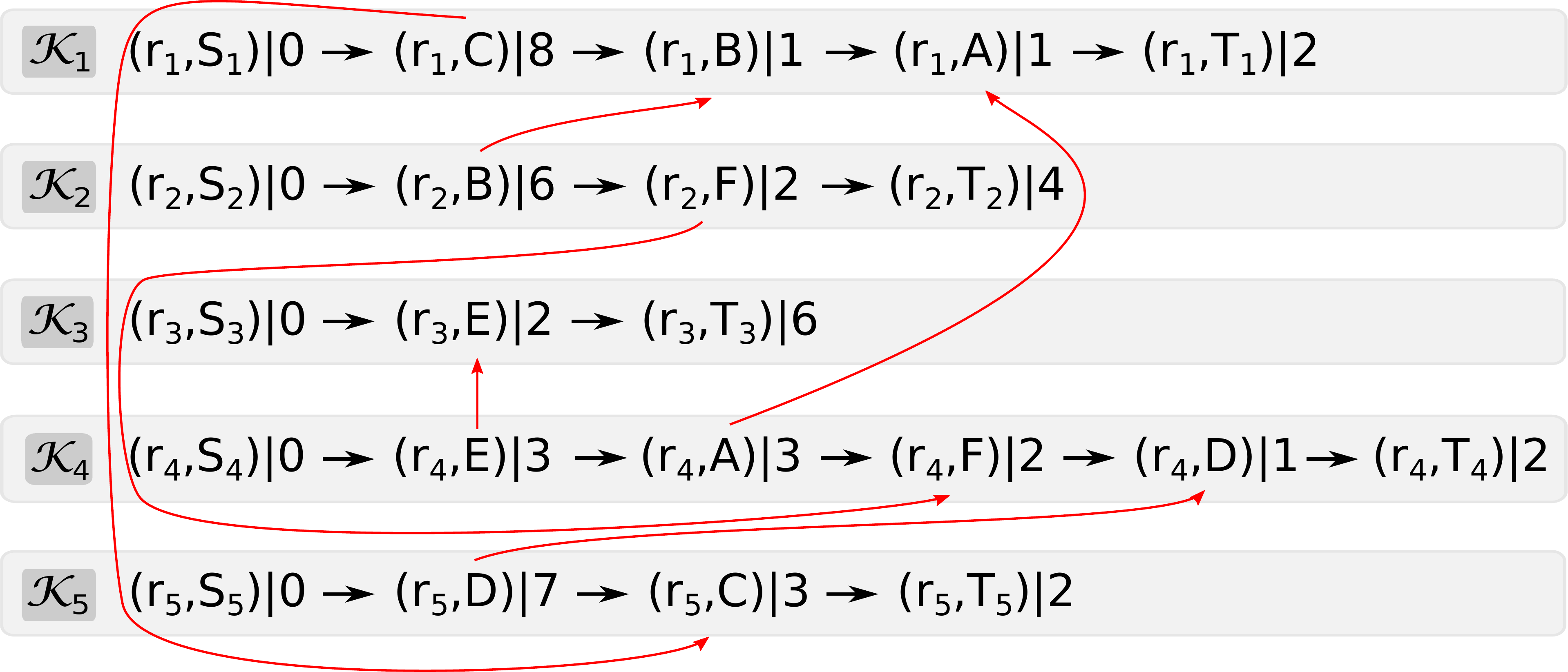}
}

  \subfloat[Schedule step 1.\label{subfig:schedule-1}]{%
  \includegraphics[width=1.6in]{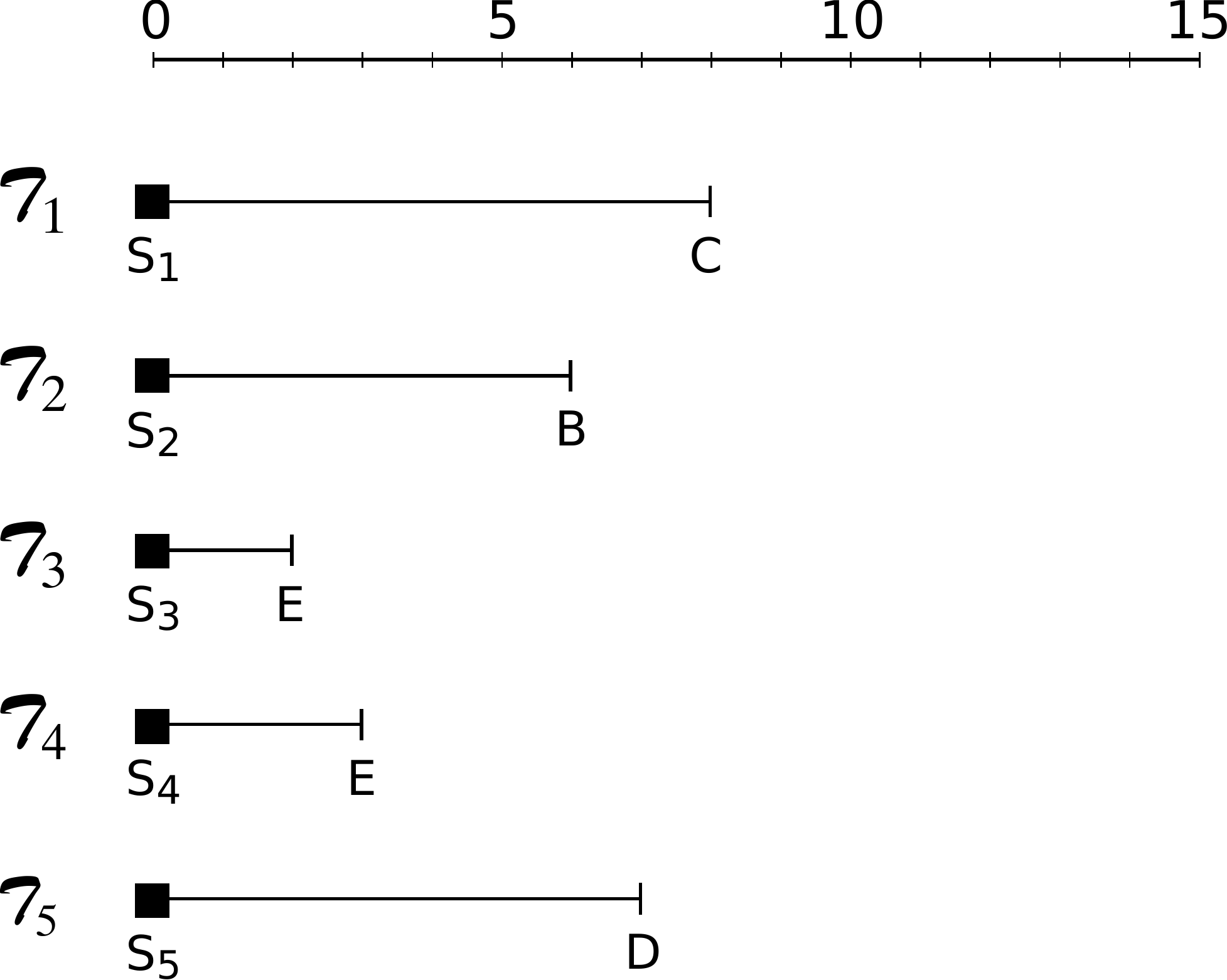}
}
\hspace{5mm}
  \subfloat[Updated NIG after step 1.  \label{subfig:DAG-1}]{%
  \includegraphics[width=2.9in]{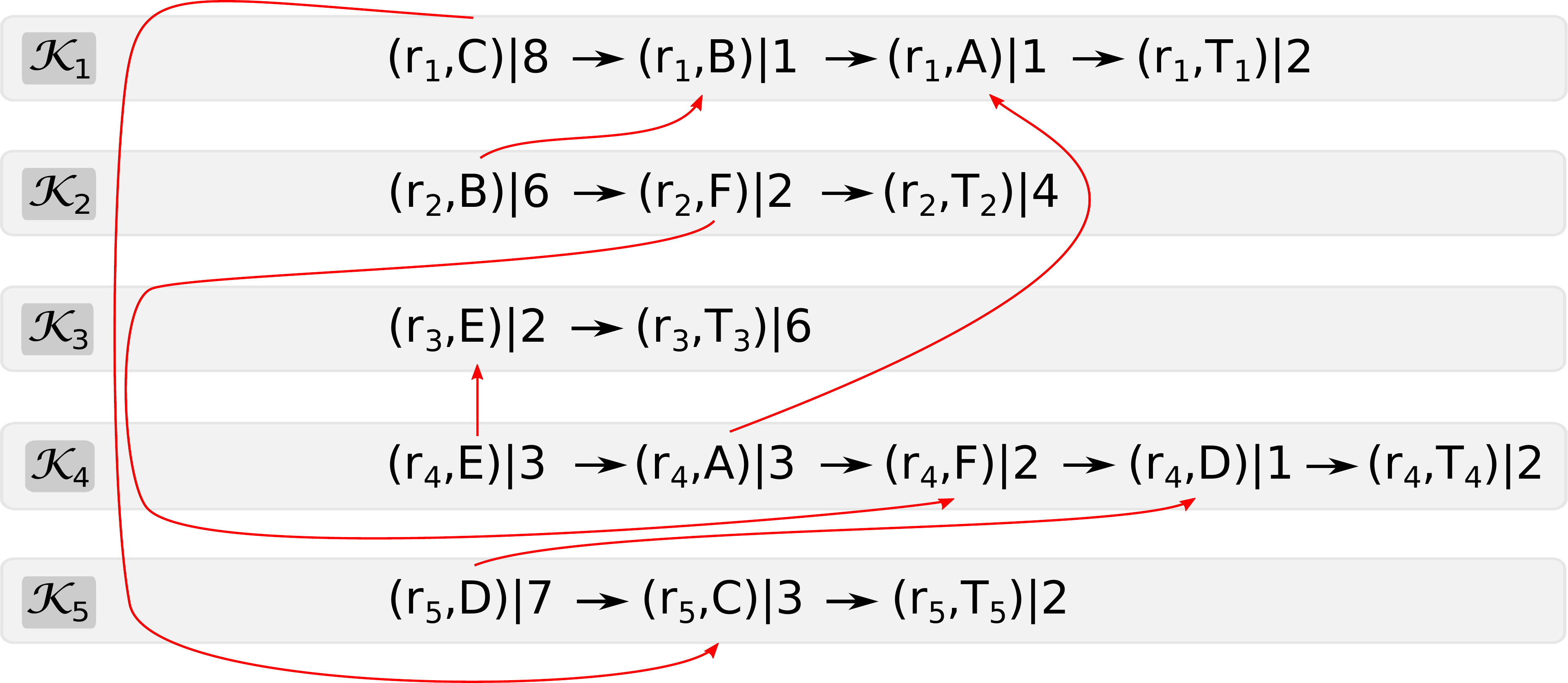}
}

  \subfloat[Schedule step 2.\label{subfig:schedule-2}]{%
  \includegraphics[width=1.6in]{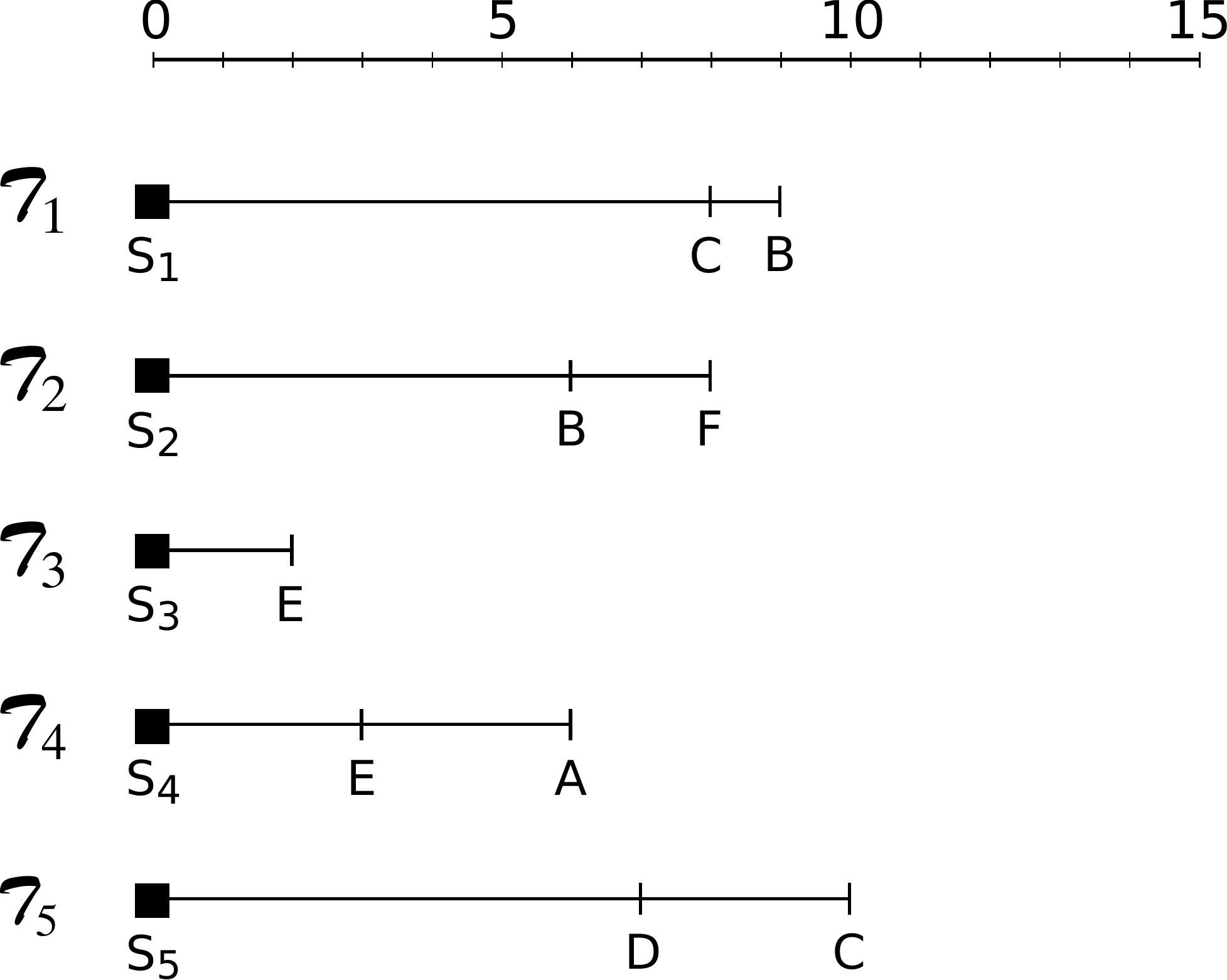}
}
\hspace{5mm}
  \subfloat[Updated NIG after step 2.  \label{subfig:DAG-2}]{%
  \includegraphics[width=2.9in]{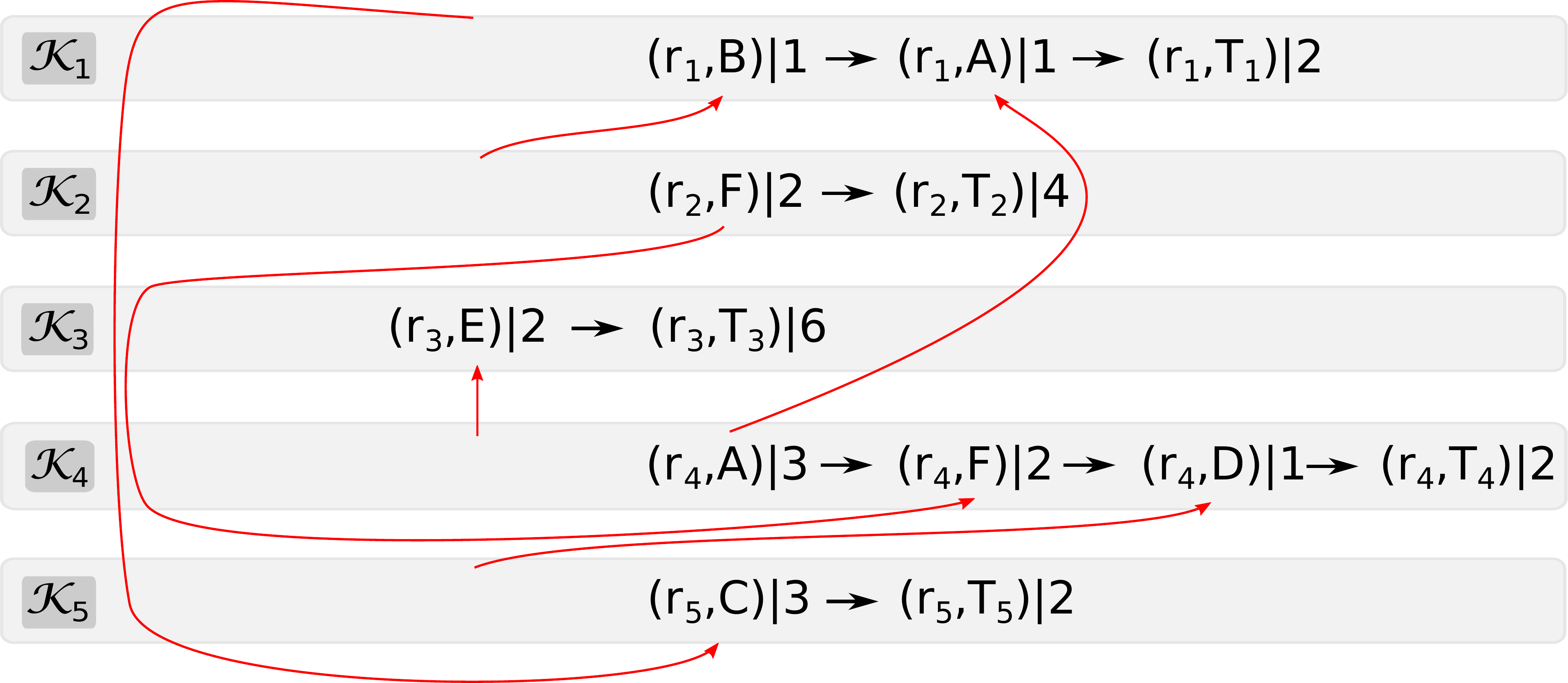}
}

  \subfloat[Schedule step 3.\label{subfig:schedule-3}]{%
  \includegraphics[width=1.6in]{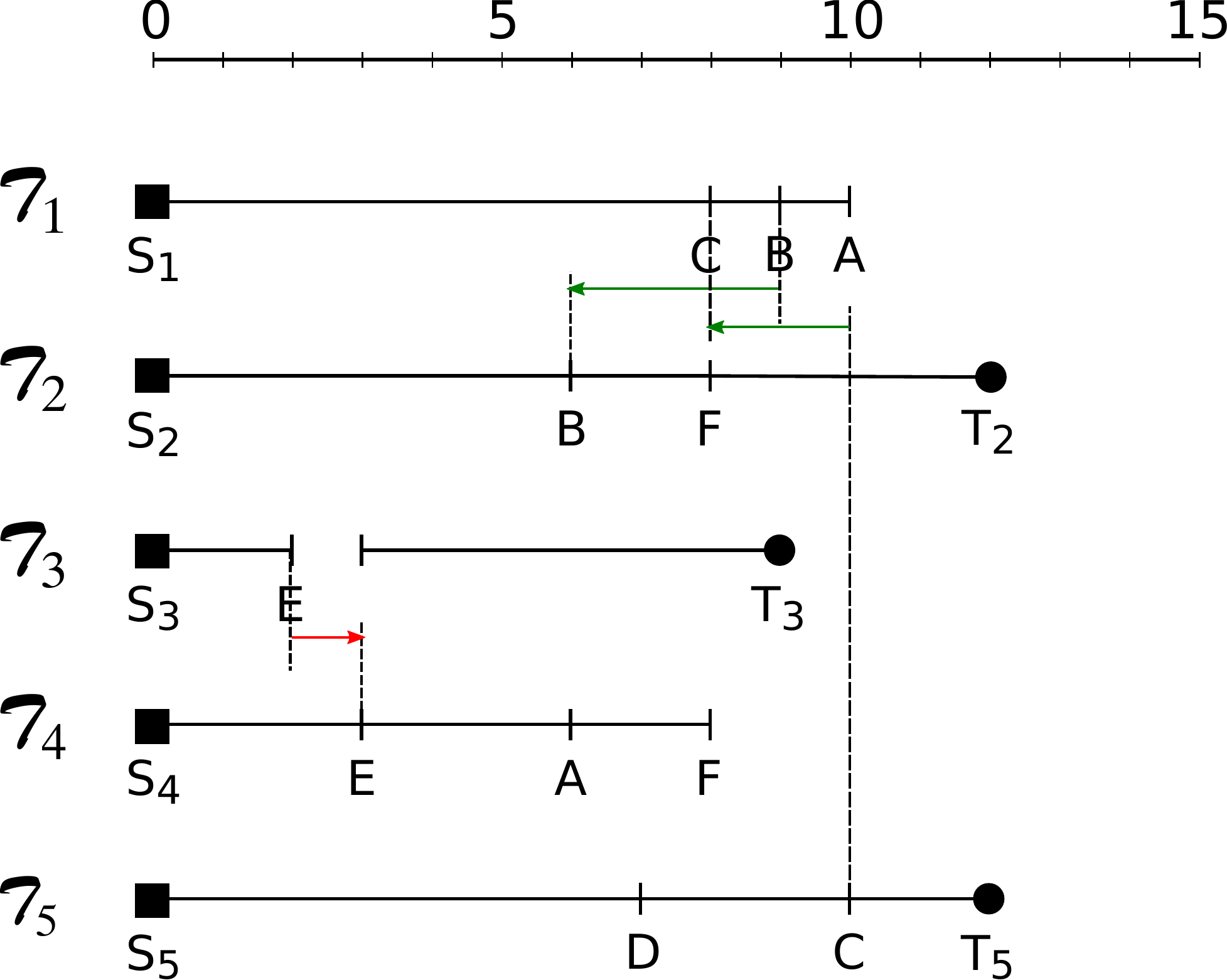}
}
\hspace{5mm}
  \subfloat[Updated NIG after step 3.  \label{subfig:DAG-3}]{%
  \includegraphics[width=2.9in]{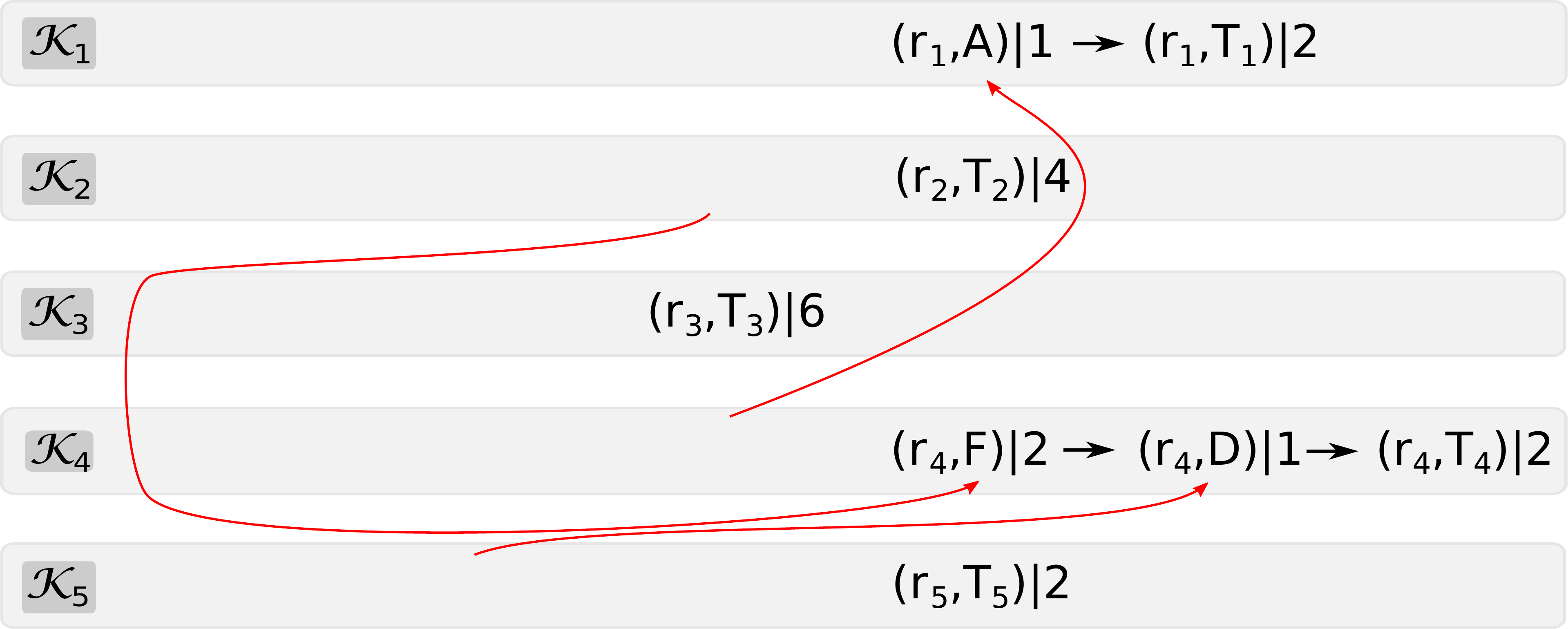}
}

  \subfloat[Schedule step 4.\label{subfig:schedule-4}]{%
  \includegraphics[width=1.6in]{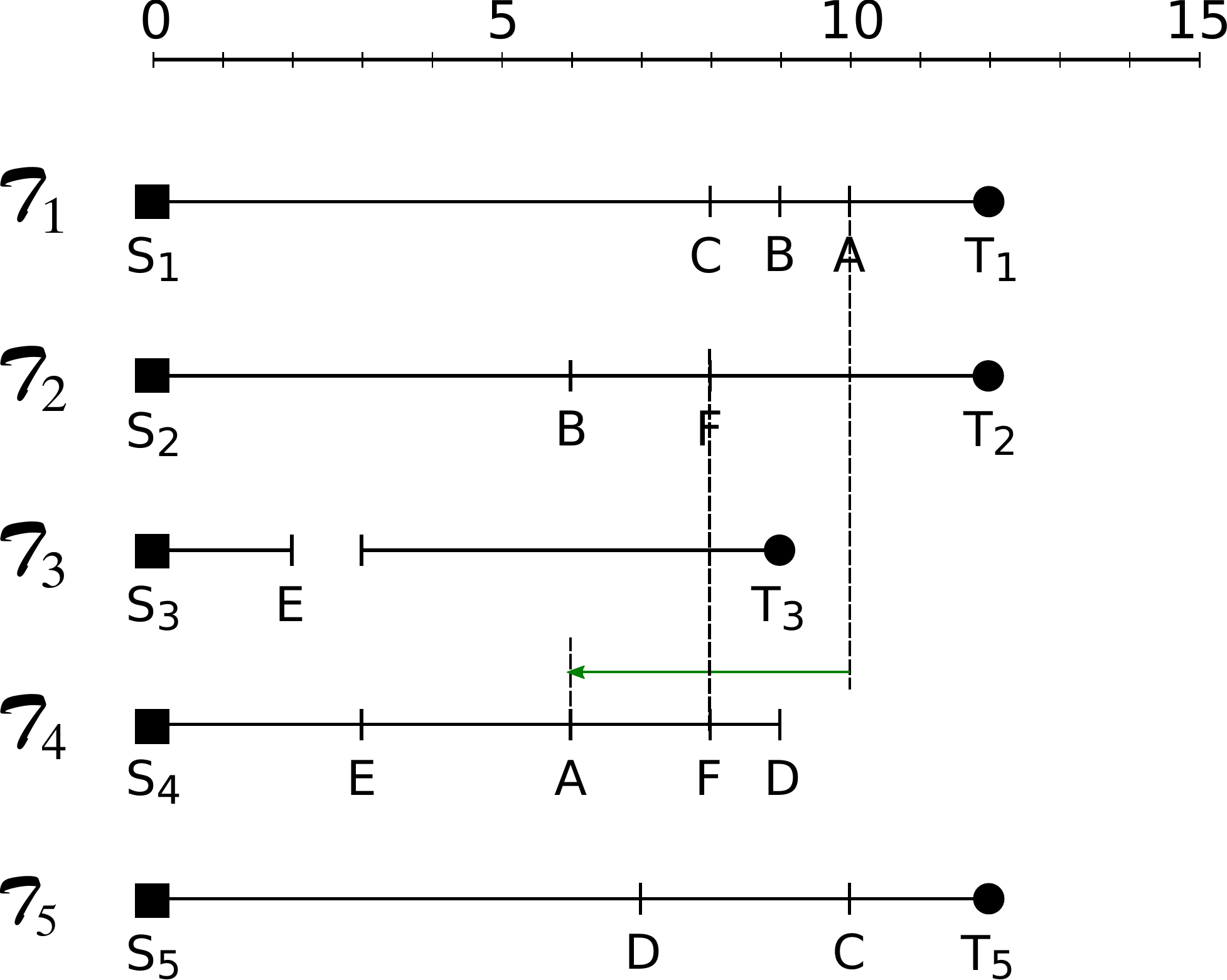}
}
\hspace{5mm}
  \subfloat[Updated NIG after step 4.  \label{subfig:DAG-4}]{%
  \includegraphics[width=2.9in]{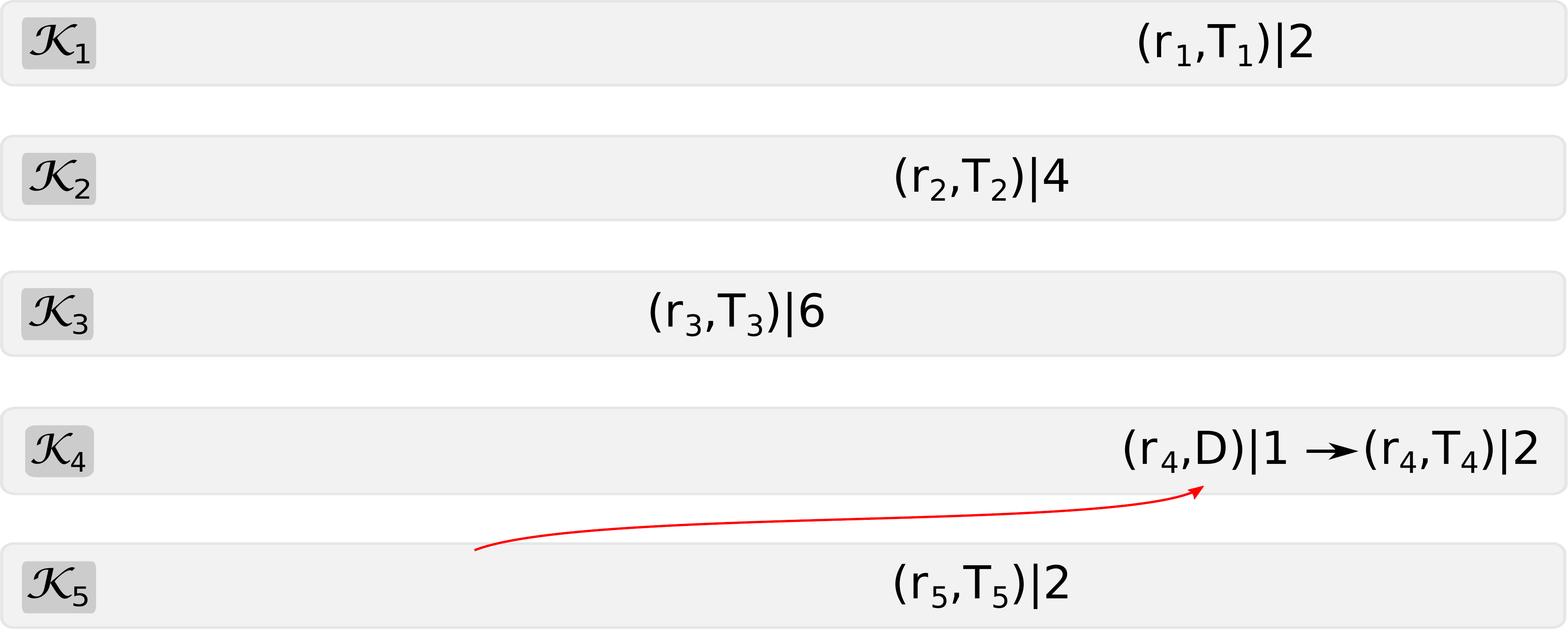}
}

  \subfloat[Schedule step 5.\label{subfig:schedule-5}]{%
  \includegraphics[width=1.6in]{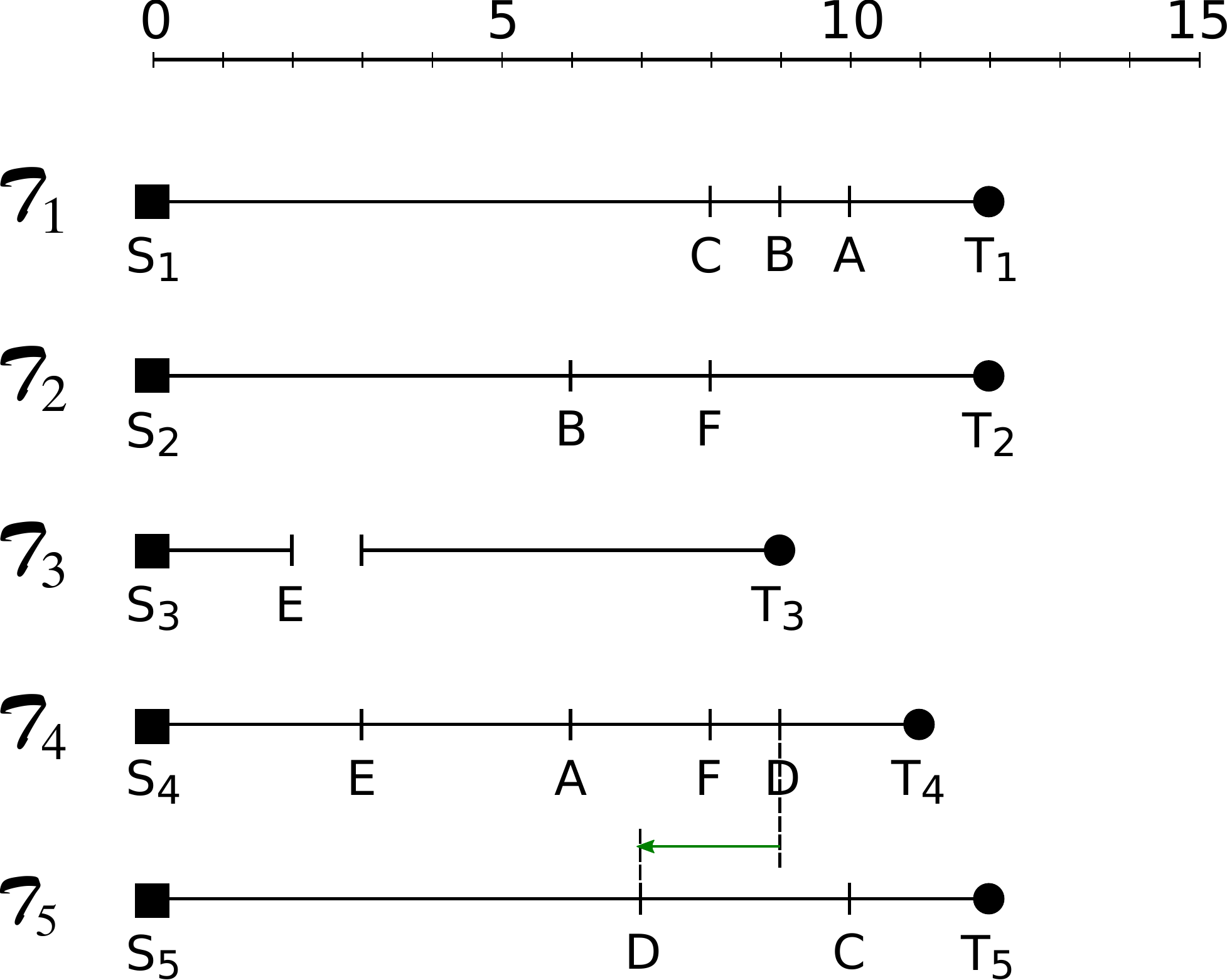}
}
\hspace{5mm}
  \subfloat[Updated NIG after step 5.  \label{subfig:DAG-5}]{%
  \includegraphics[width=2.9in]{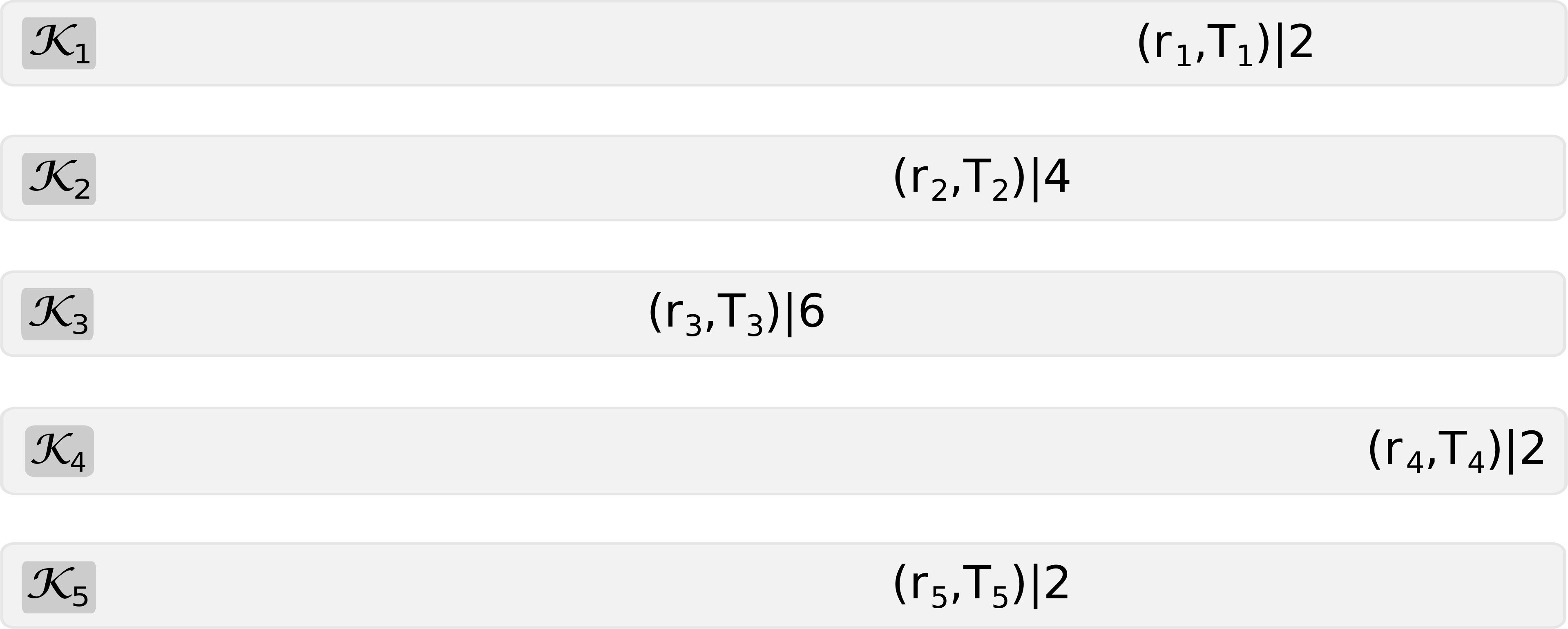}
}
\caption{Motion schedule generation for target cable configuration in Fig.~\ref{fig:straight_sequential}(a)}
\label{fig:motion-schedule}
\end{figure*}

\subsection{Hardware experiment}

We validate the proposed algorithms in a hardware experiment. We
consider the large-scale, multi-robot 3D Printing system previously
developed by our group~\cite{zhang2018large}. The tethers are
indispensable here to deliver the fresh concrete from the mixer to the
print nozzle. 
The size of each mobile robot is 960$\times$793$\times$296\,mm. 
Each cable is 10\,m in length and has a bend radius of 110\,mm.
We implemented the coordinated motion plan of
Fig.~\ref{fig:tether-intro}. 
$(S_i, T_i)$ segment for each robot was 7\,m. The maximal velocity of 
each robots was 0.6\,m/s.
To achieve the target cable configuration,
at each intersection point, the robots must follow the motion priority 
as indicated in the NIG (see Fig.~\ref{fig:network-interaction}(d)).

Fig.~\ref{fig:experiment} shows snapshots of the coordinated motions 
when the robots respected the calculated motion priority. 
The mobile robots moved in \strconc{} mode: they were commanded to 
move simultaneously along their $(S_i, T_i)$ straight-line segments.
It is worthy of mention that at 14\,s in the video of the experiment, 
$r_2$ stopped and waited until $r_1$ passed point $A$, after which 
$r_2$ passed $A$ and proceeded towards $T_2$. This movement 
reflects the motion priority of $r_1$ and $r_2$ at point $A$: 
$r_1$ must pass $A$ before $r_2$ does. 
Although the robots were not strictly points, and that 
the tethers were not perfectly flexible, the target cable 
configuration could be achieved as the robot motions respected the 
calculated priorities at the intersection points.
The result was evaluated purely based on positions of the robots, 
without taking their orientations into account.

\begin{figure*}[htp]
\centering
\includegraphics[width=.9\textwidth]{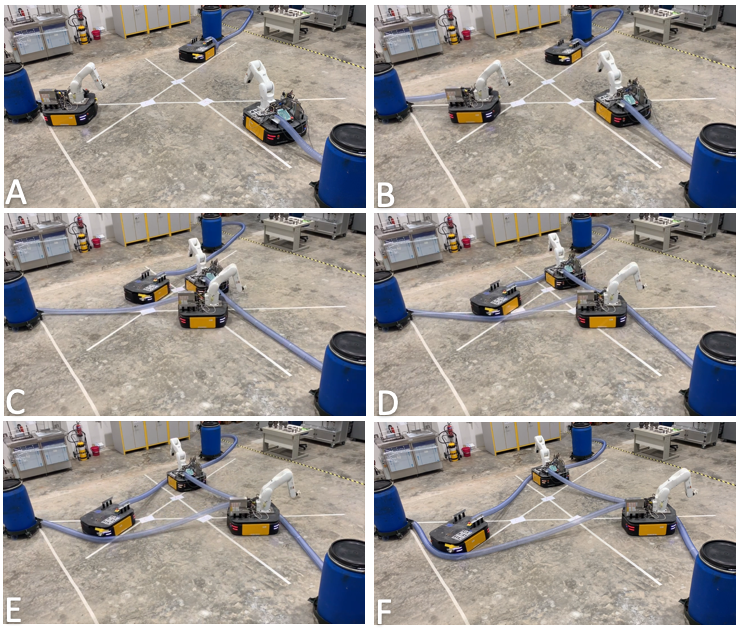}
\caption{Snapshots of the coordinated motions that achieve the desired
  target cable configuration. See the full video of the experiment at
  \url{https://youtu.be/Wdk9E0bB4yA}.  A. Tethered robots at starting
  positions; B. Navigating towards target locations in concurrent
  straight-line motion; C. Passing intersection points with calculated
  priority; D. Pushing other robot's cable during navigation;
  E. Navigation in progress; F. Achieving final (target) cable
  configuration.}\label{fig:experiment}
\end{figure*}

By contrast, Fig.~\ref{fig:experiment-wrong} shows snapshots of a
second run, where the robots did not respect the calculated priorities
at the intersection points. For better illustration of motion order at 
each intersection point, the robots were commanded to move in 
\strseq{} mode in the order of $r_1$, $r_2$ and $r_3$
(see Fig.~\ref{fig:example-intro}(b)).
Therefore, $r_1$ already reached $T_1$ before $r_3$ started to move 
(see Fig.~\ref{fig:experiment-wrong}D), 
which violated the motion priority at point $C$: $r_3$ must pass $C$ 
before $r_1$ does.
When $r_3$ navigated towards $T_3$, it had to push $C_1$ from $C$ all 
the way to $T_3$, resulting in a wrong final cable
configuration as shown in Fig.~\ref{fig:experiment-wrong}F.
Due to the limitation in tether length and the elastic property of the 
tethers, $r_1$ was dragged back when $C_1$ was pushed forward by $r_3$,
and $r_3$ failed to reach its target position.

A comparison of the total travel distance and the travel time for the above mentioned two hardware implementations is provided in Table~\ref{table:compare-mode}. Data for \bentconc{} and \bentseq{} modes was also computed for systematical comparison.
\begin{table}[htb]
\caption{Comparison of travel distance and travel time for different motion modes.}
\label{table:compare-mode}
\resizebox{\columnwidth}{!}{%
  \begin{tabular}{|l|l|l|l|l|}
  \hline
  \multicolumn{1}{|l|}{Motion mode}  
  & \begin{tabular}[c]{@{}l@{}}\texttt{Straight/} \\ \texttt{Concurrent}\end{tabular}
  & \begin{tabular}[c]{@{}l@{}}\texttt{Straight/} \\ \texttt{Sequential}\end{tabular}
  & \begin{tabular}[c]{@{}l@{}}\texttt{Bent/} \\ \texttt{Concurrent}\footnotemark\end{tabular} 
  & \begin{tabular}[c]{@{}l@{}}\texttt{Bent/} \\ \texttt{Sequential}\footnotemark[2]\end{tabular} \\ 
  \hline
  \multicolumn{1}{|l|}{Total travel distance (m)} 
  & 21 & 21 & 30 & 30 \\ 
  \hline
  \multicolumn{1}{|l|}{Travel time (s)} 
  & 14 & 40 & 20 & 60\\ 
  \hline
  \end{tabular}
}
\end{table} 
\footnotetext[2]{Note that \bentconc{} and \bentseq{} modes were not implemented in hardware experiment. The data was computed based on the robot velocity and tether length.}
\begin{figure*}[htp]
\centering
\includegraphics[width=.9\textwidth]{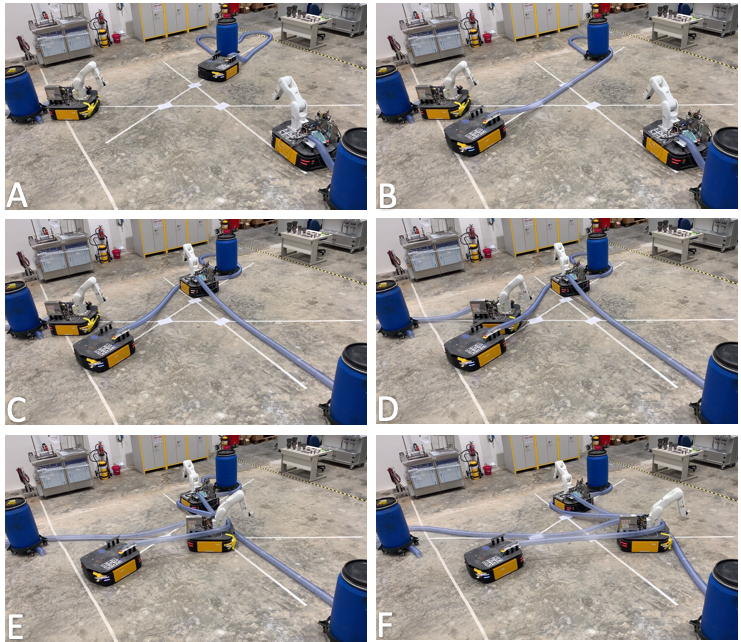}
\caption{Snapshots of the motions that fail to achieve the desired
  target cable configuration.  See second part of the video
  \url{https://youtu.be/Wdk9E0bB4yA}.  A. Tethered robots at starting
  positions; B. Robot $r_1$ passing $A$ before $r_2$ as calculated;
  C. Robot $r_2$ passing $B$ before $r_3$ as calculated; D. Robot
  $r_3$ pushing $r_1$'s cable from the wrong side as the motion order
  of $r_1$ and $r_3$ at $C$ is wrong; E. Robot $r_3$ continuing its
  navigation; F. Resulting in wrong final cable
  configuration.}\label{fig:experiment-wrong}
\end{figure*}

\section{Conclusion} \label{sec:conclusion} 

We have considered the motion planning problem for multiple tethered
planar (point) mobile robots. We have presented motion planning
algorithms for the robots to achieve a given target cable
configuration by straight-line and concurrent (\strconc{})
motions. The algorithms (i) identify whether the target cable
configuration was deadlocked; (ii) return a valid coordinated motion
plan accordingly.  The correctness of the algorithms was proved in the
process of algorithm development, and the worst time-complexity of the
full pipeline was shown to be $O(n^3)$.

The proposed algorithms have been validated in both simulations and
hardware experiments: we showed that, using our algorithms, the
tethered robots could achieve a target cable configuration that
involved a deadlock situation for \strseq{} motions, and that could
not be solved by algorithms previously proposed in the literature.

During the hardware implementation, we found that for some 
applications, it could be challenging to design a tether management 
mechanism to keep the tether taut and retractable.
In addition, robot size and geometry need to be taken into account 
during motion planning, especially for the purpose of collision 
avoidance. Our future work will consider integrating tether stiffness 
property and robot model into the proposed algorithms.

\section*{Acknowledgment} \label{sec:acknowledgment} This work was
partially supported by the Medium-Sized Centre funding scheme (awarded
by the National Research Foundation, Prime Minister's Office,
Singapore) and by Sembcorp Design \& Construction Pte Ltd.  The
authors would like to thank Hung Pham, Teguh Santoso Lembono and Lim
Jian Hui for their constructive remarks on earlier versions of the
manuscript, Lim Jian Hui and Panda Biranchi for their helps with the experiment.

\appendix
\section*{Appendix}

\section{Proofs of the pair interaction propositions}
\label{appendix:cable-polygons}

We first need a topological lemma, which is also used in Jordan's
proof of the Jordan curve theorem (see Lemma~1
of~\cite{hales2007jordan}).

\begin{lemma}
  Consider a polygon $\Pi$, a point $S$ in the exterior of $\Pi$, and
  a continuous path starting from $S$ and ending at $T$. Then $T$ is
  in the interior of $\Pi$ if and only if $\calP$ intersects the
  boundary of $\Pi$ an odd number of times.
  \label{lemma:Jordan}
\end{lemma}

We can now prove the following propositions, which are required by the
pair interaction propositions of Section~\ref{sec:PIG}.

\begin{proposition}
  Assume that $S_i, T_i \notin \Pi_j$ and $S_j, T_j\notin\Pi_i$, then
  the two segments $(S_i, T_i)$ and $(S_j, T_j)$ do not intersect.
\end{proposition}

\begin{proof}
  Assume by contradiction that $(S_i, T_i)$ and $(S_j, T_j)$ intersect
  at a point $P$. Consider the cable line $\calC_j$ of robot $r_j$. By
  lemma~\ref{lemma:Jordan}, $\calC_j$ will intersect the boundary of
  $\Pi_i$ an even number of times. Since $\calC_j$ has zero
  intersection with $\calC_i$ (cable lines are non intersecting),
  $\calC_j$ must intersect $(S_i, T_i)$ an even number of times. Since
  $(S_i, T_i)$ already intersects $(S_j, T_j)$ at one point $P$, the
  total number of intersections between $(S_i, T_i)$ and the boundary
  of $\Pi_j$ is odd. By lemma~\ref{lemma:Jordan}, either $S_i$ or
  $T_i$ must therefore be in the interior of $\Pi_j$, which raises
  a contradiction.
\end{proof}

\begin{proposition}
  Assume that $S_i \notin \Pi_j$, $S_j\notin\Pi_i$, $T_i \in \Pi_j$,
  $T_j\notin\Pi_i$, then the two segments $(S_i, T_i)$ and
  $(S_j, T_j)$ must intersect.
\end{proposition}

\begin{proof}
  Assume by contradiction that $(S_i, T_i)$ and $(S_j, T_j)$ do not
  intersect.  By lemma~\ref{lemma:Jordan}, $\calC_i$ will intersect
  boundary of $\Pi_j$ for odd number of times.  Since $\calC_i$ does
  not intersect $\calC_j$ under the problem formulation, $\calC_i$
  must intersect $(S_j, T_j)$ for odd number of times.  The assumption
  that $(S_j, T_j)$ and $(S_j, T_j)$ do not intersect then implies
  that $(S_j, T_j)$ intersects $\Pi_i$ for odd number of times.
  Therefore either $S_j$ or $T_j$ has to be in the interior of
  $\Pi_i$, which raises a contradiction.
\end{proof}

\begin{proposition}
  Assume that $S_i \notin \Pi_j$, $S_j\notin\Pi_i$, $T_i \in \Pi_j$,
  $T_j \in \Pi_i$, then the two segments $(S_i, T_i)$ and $(S_j, T_j)$
  do not intersect.
\end{proposition}

\begin{proof}
  Assume by contradiction that $(S_i, T_i)$ and $(S_j, T_j)$
  intersects at a point $P$.  By lemma~\ref{lemma:Jordan},
  $(S_i, T_i)$ will intersect boundary of $\Pi_j$ for odd number of
  times, and $\calC_j$ will intersect boundary of $\Pi_i$ for odd
  number of times.  Since $(S_i, T_i)$ and $(S_j, T_j)$ intersects at
  $P$, $(S_i, T_i)$ will intersect $\calC_j$ for even number of times.
  Therefore $\calC_j$ will intersect $\calC_i$ for odd number of
  times, which raises a contradiction as cable lines are
  non intersecting.
\end{proof}

\section{Revising and supplementing Hert and Lumelsky
  1997}\label{Hert-Lumelsky}

In this section, route intersection detection algorithms for case 1(b) and
2(b) of the paper~\cite{hert1997planar} are supplemented and revised.
Some background terms are first introduced here.

A \textit{cable route ($\calR$)} for robot $r_i$ through 
graph $G$ is defined as the set of adjacent edges leading from $S_i$ 
to some node. Without loss of generality, it is assumed that no three 
consecutive nodes from a single route are collinear.
Therefore, any two adjacent edges in a route form two distinct angles, one convex ($< \pi$) and the other concave ($> \pi$).

To find a feasible target cable configuration from $G$, the most 
fundamental task is to identify whether there is any route intersection 
in a given cable configuration.
As discussed in~\cite{hert1997planar}, there
are five different ways $\calR_A$ and $\calR_B$ can meet at a node $T_k$, considering the
relative placement of edges from $\calR_A$ and $\calR_B$:

\begin{enumerate}
\item $T_k$ is the last node of one of the routes, say $\calR_A$. $\calR_A$ 
  has one edge incident to $T_k$ and $\calR_B$ has two. The routes
  meet in one of the following two ways:
  \begin{itemize}
  \item[(a)] $\calR_A$ and $\calR_B$ do not share an edge incident to $T_k$;
  \item[(b)] $\calR_A$ and $\calR_B$ do share an edge incident to $T_k$.
  \end{itemize}

\item $T_k$ is not the last node of either of the routes. Each route has two edges incident to $T_k$. There are three cases as the following:
  \begin{itemize}
  \item[(a)] $\calR_A$ and $\calR_B$ share no edge incident to $T_k$;
  \item[(b)] $\calR_A$ and $\calR_B$ share one edge incident to $T_k$;
  \item[(c)] $\calR_A$ and $\calR_B$ share both edges incident to $T_k$.
  \end{itemize}
\end{enumerate}

Route intersection may occur in cases from 1(a) to 2(b), but can be avoided
in case 2(c).
Besides the above possible intersections, there is one more case where
two routes do not intersect at certain node, instead, they intersect
in their interiors.
Readers may refer to~\cite{hert1997planar} for more details on the cable intersection detection algorithm, but the algorithm for case 1(b) is not complete, and that for case 2(b) needs revision.
In the following section, we provide the revised and supplemented algorithms.

\subsection{Detecting Route Intersection (revised and
  supplemented)}

\subsubsection{Case 1(b)} \hspace*{\fill}

\textbf{Case 1(b)}: Suppose $\calR_A$ and $\calR_B$ have one or more edges in
common, and one of them is the last edge for $\calR_A$. $T_l$ is the
goal node of $\calR_A$.

In~\cite{hert1997planar}, the authors conclude this case as the
following:

Let $e_{a1}$ and $e_{b1}$ be the two edges from $\calR_A$ and $\calR_B$,
immediately before this common set of edges. Let $e_{b2}$ be the edge
from $\calR_B$ immediately after the common set of edges. Let $e_{fc}$
be the first edge shared by $\calR_A$ and $\calR_B$, and $e_{lc}$ the 
last edge shared by $\calR_A$ and $\calR_B$.

\begin{figure}[htp]
\centering
\includegraphics[width=2.5in]{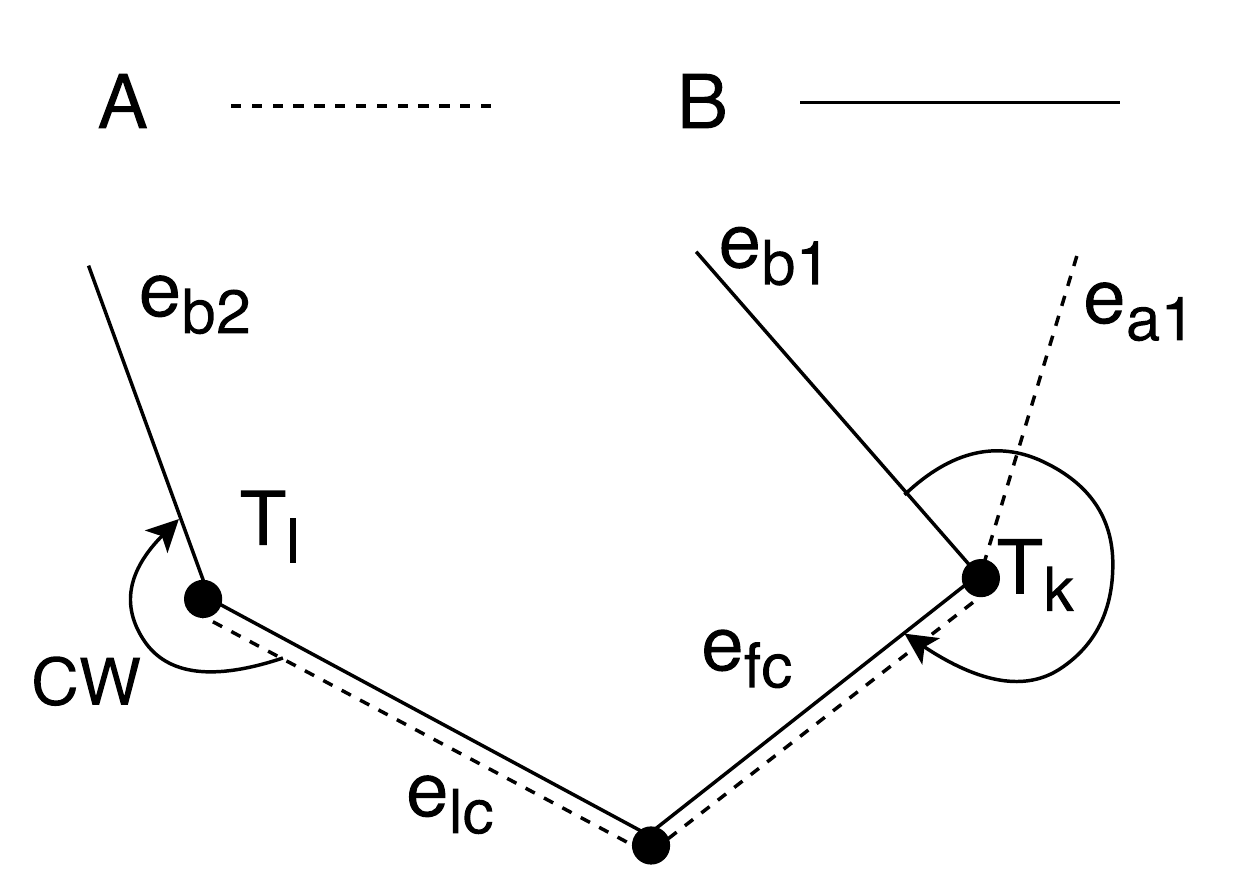}
\caption{Case 1(b). $\calR_A$ and $\calR_B$ intersect at $T_l$. (The
  figure is taken from~\cite{hert1997planar}
  Fig. 11.)  \label{fig:case1(b)}}
\end{figure}

\textbf{Lemma 2} $\calR_A$ and $\calR_B$ intersect at $T_l$ if and only if
$e_{a1}$ is contained in the clockwise (counterclockwise) angle from
edge $e_{b1}$ to $e_{fc}$ and the clockwise (counterclockwise) angle
from $e_{lc}$ to $e_{b2}$ is $> \pi$.

However, the above statement and Lemma 2 are capable of intersection 
detection for case 1(b) only if edge $e_{b1}$ exists. 
Considering the situation where $\calR_A$ and $\calR_B$ are developed from
different directions and $T_k$ is the target position of $\calR_B$, the
case has to be reformulated as in Fig.~\ref{fig:case1(b)_r} and Lemma 2s.

\begin{figure}[htp]
\centering
\includegraphics[width=2.3in]{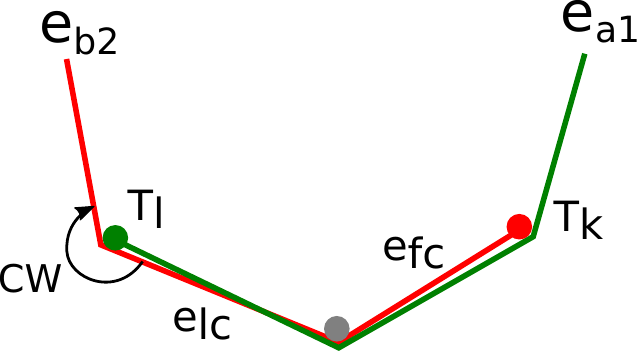}
\caption{Case 1(b) supplement. Edge $e_{b1}$ does not exist, meaning
  node $T_k$ is goal node of $\calR_B$. $\calR_A$ and $\calR_B$
  intersect at $T_l$.  \label{fig:case1(b)_r}}
\end{figure}

\textbf{Lemma 2s} If edge $e_{b1}$ does not exist, and node $T_k$ is
goal node of $\calR_B$. $\calR_A$ and $\calR_B$ intersect at $T_k$ if and only if
clockwise (counterclockwise) angle from $e_{a1}$ to $e_{fc}$ is
$> \pi$ and the clockwise (counterclockwise) angle from $e_{lc}$ to
$e_{b2}$ is $> \pi$.

\subsubsection{Case 2(b)} \hspace*{\fill}

\textbf{Case 2(b)}: Suppose $\calR_A$ and $\calR_B$ meet at node $T_k$, and
each route has two edges incident to $T_k$. They share one of the
edges. Let $\calR_A$ be the route currently in construction and $\calR_B$
be a previously constructed route.

In~\cite{hert1997planar}, this case is concluded as the following:

The edges and nodes labeling follow the convention in case 1(b), and an 
edge $e_{a2}$ of $\calR_A$ is added immediately after the common set of 
edges. An illustration is shown in Fig.~\ref{fig:case2(b)}. 
Assume that the edge $e_{b2}$ is in the counterclockwise angle from 
$e_{a2}$ to $e_{lc}$. Then Lemma 4 is used to determine intersection.

\begin{figure} [htp]
\centering
  \subfloat[\label{subfig:a}]{%
    \includegraphics[width=1.75in]{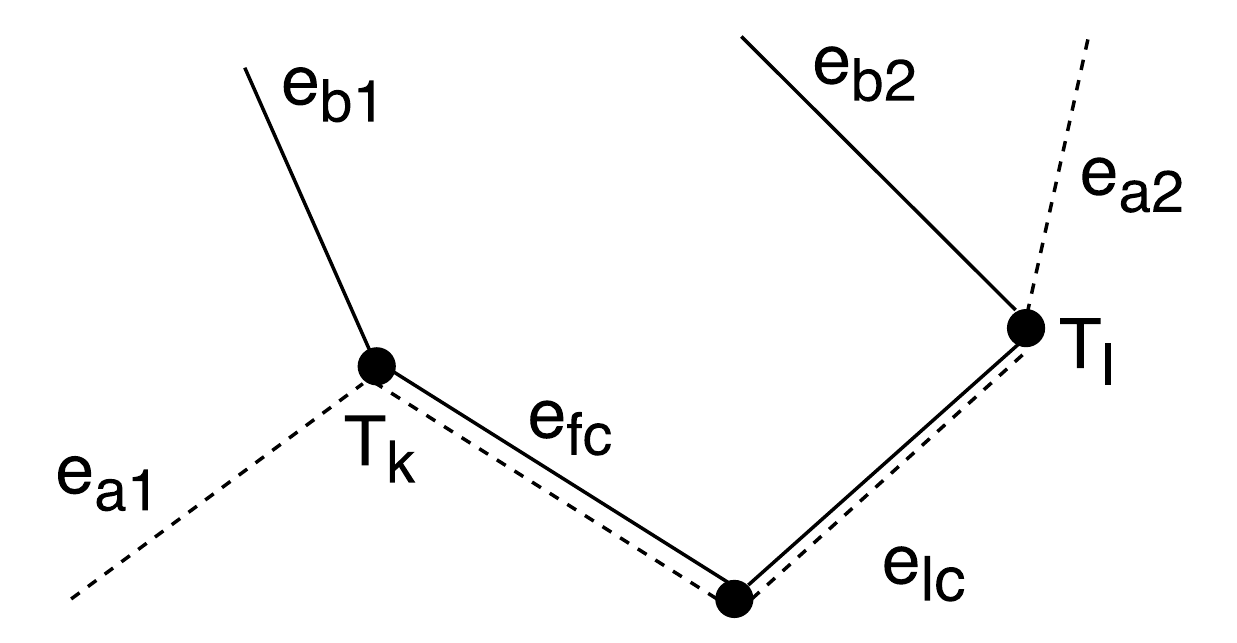}
 }
  \subfloat[\label{subfig:b}]{%
    \includegraphics[width=1.75in]{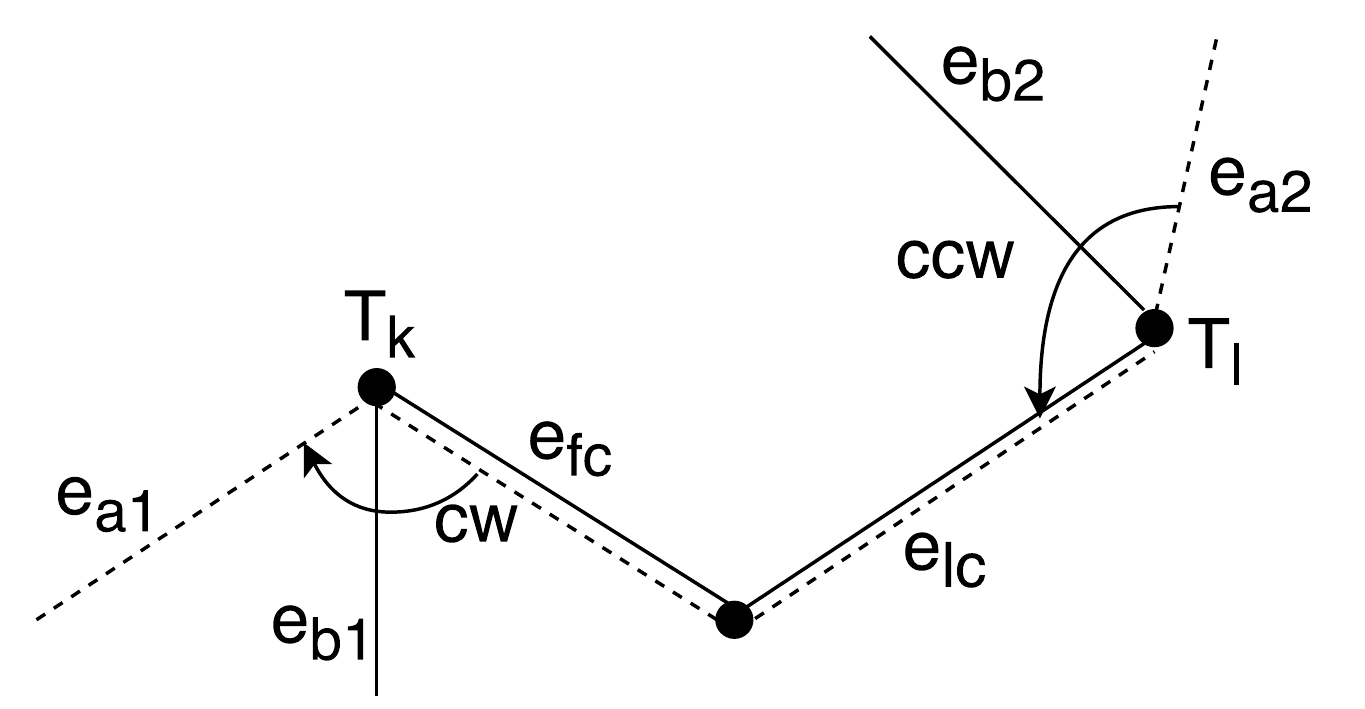}
 } 
 \caption{Case 2(b). In both (a) and (b) $\calR_A$ and $\calR_B$
   intersect at node $T_k$. (The figure is taken
   from~\cite{hert1997planar} Fig. 14.)  \label{fig:case2(b)}}
\end{figure}

\textbf{Lemma 4} $\calR_A$ and $\calR_B$ intersect at $T_k$ if and only if (a)
$e_{b1}$ is contained in the concave angle between $e_{a1}$ and
$e_{fc}$ and $e_{a1}$ is contained in the concave angle between
$e_{b1}$ and $e_{fc}$ or (b) $e_{b1}$ is contained in the clockwise
angle from $e_{fc}$ to $e_{a1}$.

However, the algorithm is redundant for option (a) in Lemma 4. 
Since the cable configuration is constructed in a depth-first manner, 
there is no need to proceed further if intersection has already been 
detected. Hence, it can be re-formulated as in Fig.~\ref{fig:case2(b)_r}.

\begin{figure}[htp]
\centering
\includegraphics[width=1.5in]{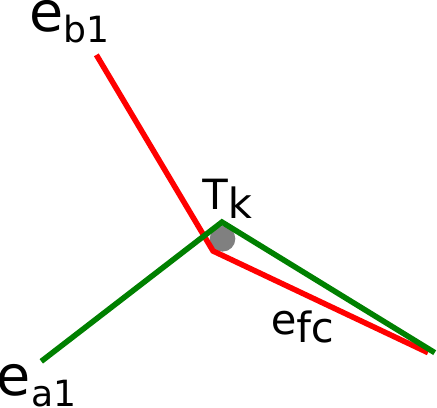}
\caption{Revised Case 2(b) for Lemma 4(a).  \label{fig:case2(b)_r}}
\end{figure}

In addition, option (b) in Lemma 4 only satisfy the situation when 
edge $e_{b2}$ exists (recall the discussion for case 1(b)). 
When $T_l$ is the goal node for $\calR_B$, meaning there is no 
edge $e_{b2}$, Lemma 4(b) has to be revised.
Below we present the revised version of Lemma 4.

\textbf{Lemma 4r} $\calR_A$ and $\calR_B$ intersect at $T_k$ if and only if (a)
$e_{b1}$ is contained in the concave angle between $e_{a1}$ and
$e_{fc}$ and $e_{a1}$ is contained in the concave angle between
$e_{b1}$ and $e_{fc}$; or (b1) $e_{b1}$ is contained in the clockwise
(counterclockwise) angle from $e_{fc}$ to $e_{a1}$ and edge $e_{b2}$
is in the counterclockwise (clockwise) angle from $e_{a2}$ to
$e_{lc}$; or (b2) $e_{b1}$ is contained in the clockwise
(counterclockwise) angle from $e_{fc}$ to $e_{a1}$ and
counterclockwise (clockwise) angle from $e_{a2}$ to $e_{lc}$ is $<\pi$
when $T_l$ is the goal node for $\calR_B$.

\section{Extension to workspace with obstacles}\label{Stage1-Obstacle}

An obstacle interferes in robot's motion only when the whole piece or
part of it is in the robot's cable polygon. 
An example is given in Fig.~\ref{fig:obstacle-interfere}.

\begin{figure} [htp]
  \centering \subfloat[The obstacle is out of $\Pi_1$ and $\Pi_2$. The
  target cable configuration can be achieved by straight-line robot
  motion as long as $r_1$ passes point $P$ before $r_2$
  does. \label{subfig:obstacle-out}]{%
    \includegraphics[width=1.6in]{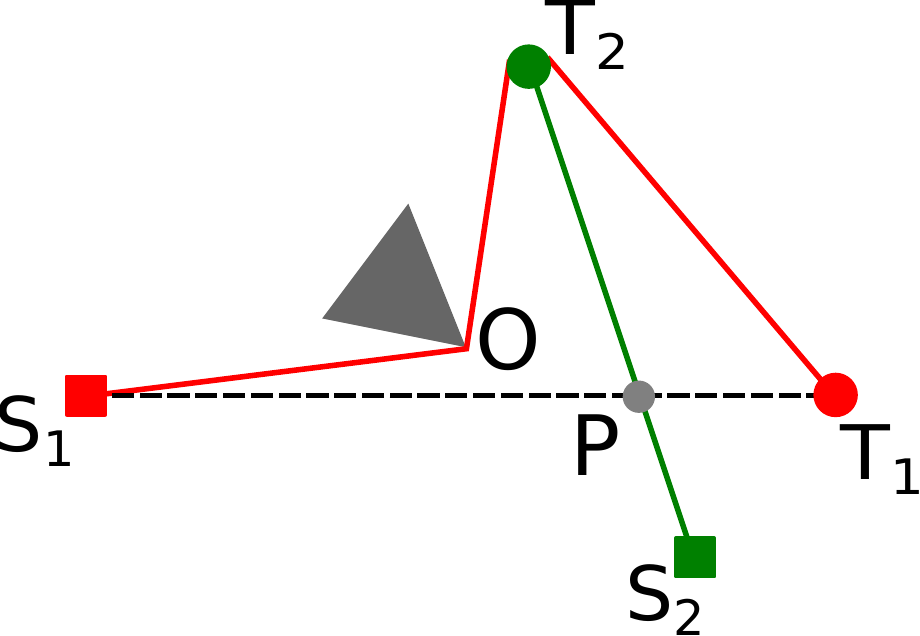}
}

\subfloat[The obstacle is in $\Pi_1$. The target cable configuration
cannot be achieved if $r_1$ moves along $(S_1,
T_1)$. \label{subfig:obstacle-in}]{%
    \includegraphics[width=1.5in]{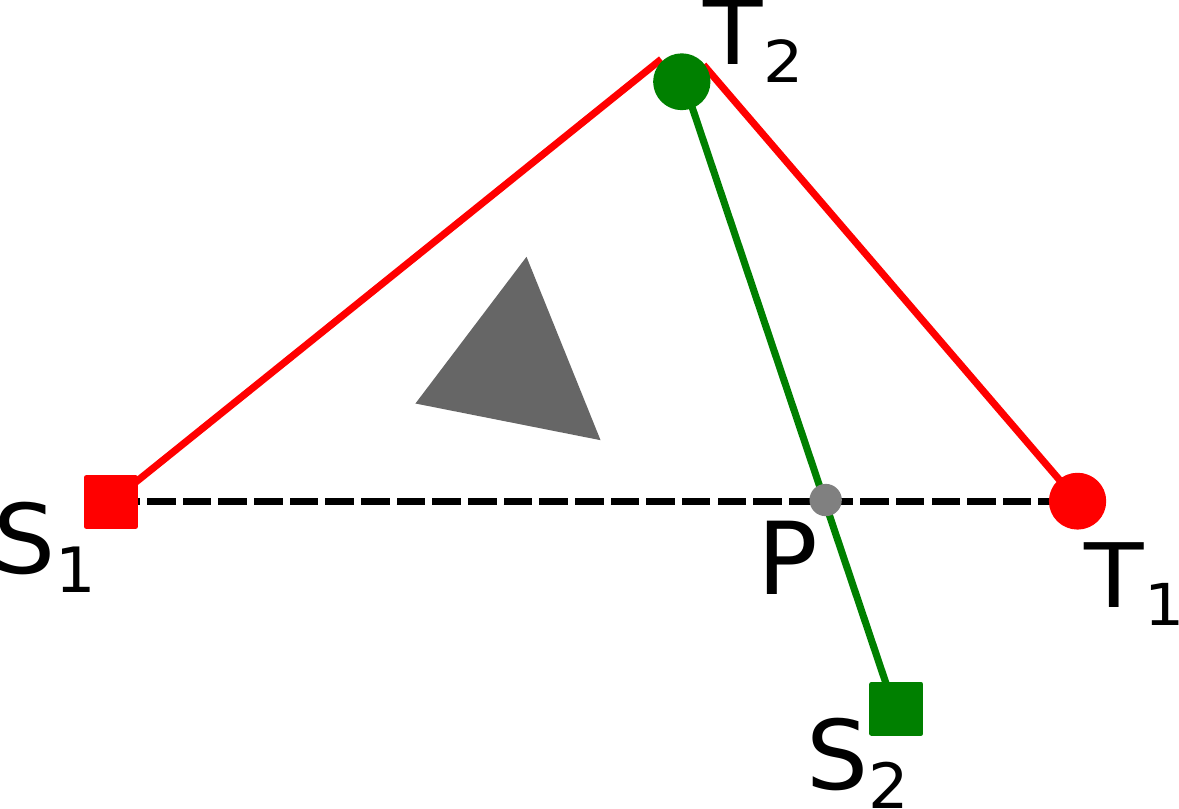}
  } \hspace{2mm} \subfloat[The shortest paths for the robots for the
  case in (b). To achieve the target cable configuration, $r_1$ has to
  move along $S_1-O'-P'-T_1$ and passes $P'$ before $r_2$
  does.\label{subfig:obstacle-in-shortest}]{%
    \includegraphics[width=1.5in]{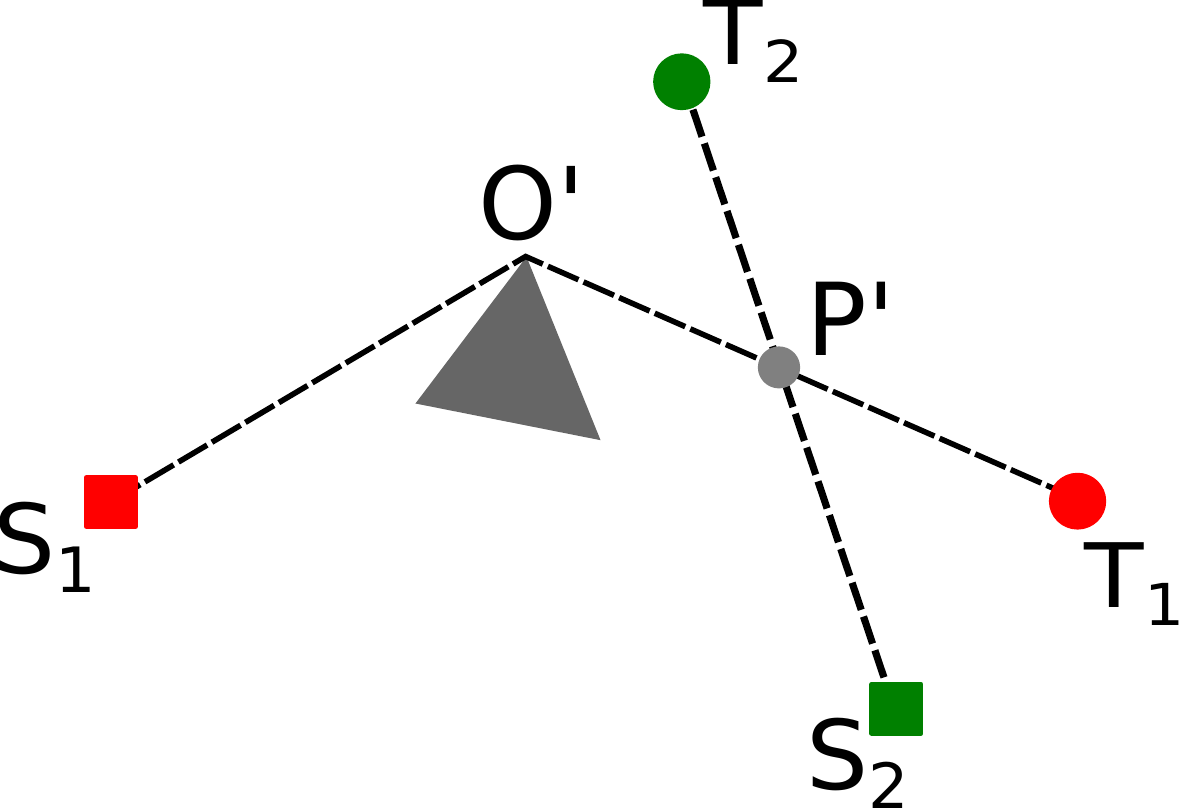}
}
\caption{Interference in robot's motion from the obstacle that is
  located in and out of robot's retraction polygon.}
 \label{fig:obstacle-interfere}
\end{figure}

In Fig.~\ref{fig:obstacle-interfere}(a), the obstacle is located in the exterior of a robot's cable polygon. Although cable line $C_1$ wraps
around the obstacle at point $O$ in the target cable configuration,
$r_1$ and $r_2$ can still move along their $(S, T)$ straight-line
path to achieve the target cable configuration as long as $r_1$
passes point $P$ before $r_2$ does, which behaves the same as the case without obstacles.
By contrast, robot's path needs to be modified when the obstacle is (partially) in the interior of the robot's cable polygon. 
As depicted in Fig.~\ref{fig:obstacle-interfere}(b, c),
to achieve the target cable configuration, the shortest path for $r_1$ 
is $S_1-O'-P'-T_1$, and $r_1$ must pass $P'$ earlier than $r_2$.

The proposed algorithms for \strconc{} motion can be adapted easily to
the situation where there are obstacles in the workspace. The only
modification is the robots' paths.  Recall that the motivation to adopt
straight-line motion is that it is the shortest path for a robot to
reach its target position.  When there are obstacles in the workspace,
the shortest path of a robot is no longer its $(S, T)$ straight-line
segment, but its retracted cable line after all other robots are
removed from the environment. 
The algorithms for \strconc{} motion still hold for workspace with obstacles with the updated shortest paths.  An example is presented in
Fig.~\ref{fig:obstacle-PIG}.

\begin{figure} [htp]
  \centering \subfloat[A given target cable configuration in the
  workspace with obstacles. \label{subfig:obstacle1}]{%
    \includegraphics[width=2.0in]{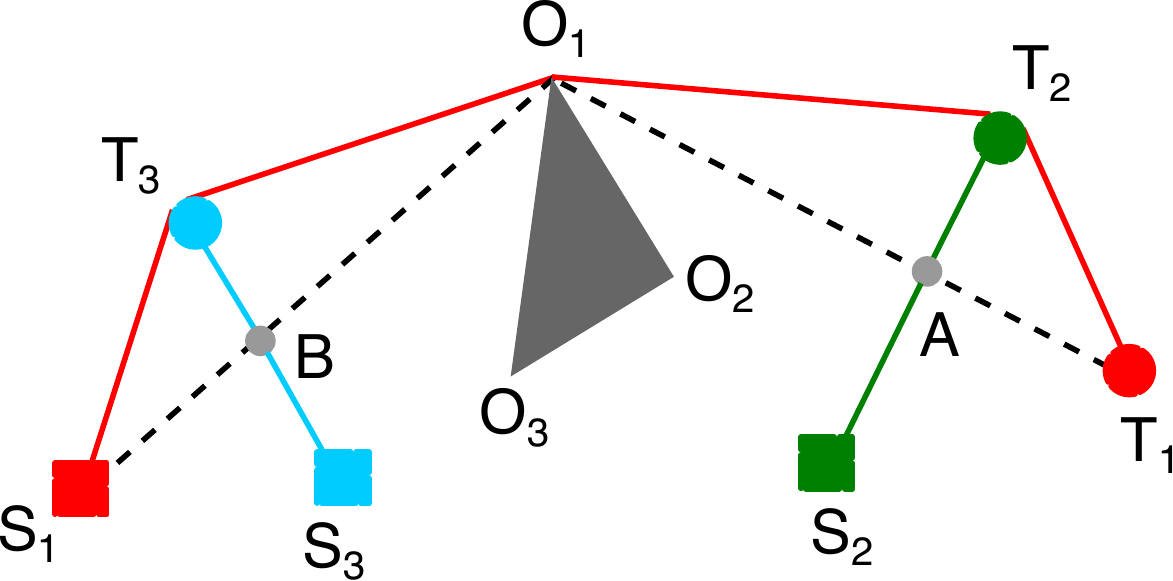}
  } \hspace{1mm} \subfloat[PIG of the target cable configuration.\label{subfig:obstacle-graph}]{%
    \includegraphics[width=1.2in]{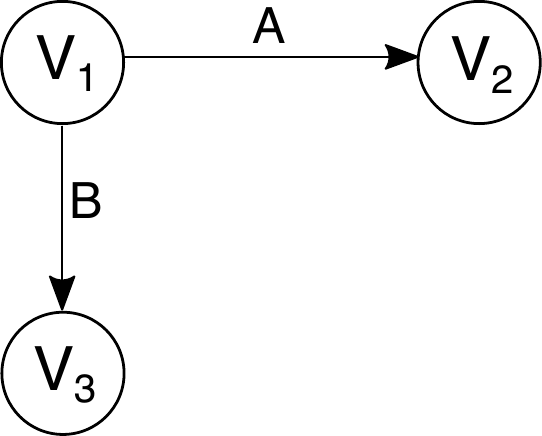}
  } \vspace{1mm} \subfloat[NIG of the target cable configuration with updated robot motion path.\label{subfig:obstacle-NIG}]{%
    \includegraphics[width=2.8in]{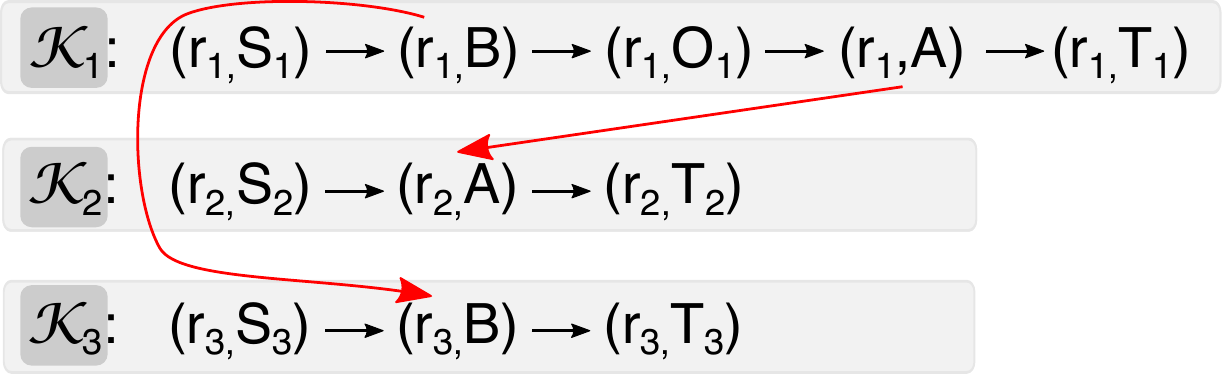}
  }
  \caption{An implementation of the proposed algorithms for the case where the workspace contains obstacles.}
 \label{fig:obstacle-PIG}
\end{figure}


\section*{References} \label{sec:references}

\end{document}